\documentclass[11pt]{article}
\usepackage[utf8]{inputenc}
\usepackage{authblk}
\usepackage[margin=1in]{geometry}
\usepackage[titletoc,title]{appendix}
\usepackage[dvipsnames,table,xcdraw]{xcolor}
\usepackage{color}
\usepackage{setspace}
\usepackage{amsmath}
\usepackage{bbm}
\usepackage{amsmath,amsfonts,amssymb,amsthm,mathtools,bm}
\usepackage{algorithm2e}
\SetKwComment{Comment}{/* }{ */}
\SetKwInput{KwInput}{Input}
\SetKwFor{When}{when}{do}{}

\usepackage{graphicx,float}
\usepackage{subcaption}  
\usepackage{booktabs}
\usepackage{booktabs,multirow}
\usepackage{tikz-cd}
\usepackage{comment}

\SetKwProg{Init}{Initialize}{}{}
\RestyleAlgo{ruled}

\usepackage[colorlinks=true, pdfborder={0 0 0}, linkcolor=red]{hyperref}

\usepackage[style=authoryear,date=year,maxcitenames=2,maxbibnames=20,giveninits,url=false,backend=biber,uniquelist=false,uniquename=false,dashed=false]{biblatex}
\bibliography{reorg}
\usepackage{wasysym,url}

\allowdisplaybreaks
\newcommand{\ti}[1]{\widetilde{#1}} 

\def\One{\mathbbm{1}}

\def\B{\mathcal B}
\def\C{\mathcal C}

\def\E{\mathcal E}
\def\F{\mathcal F}
\def\G{\mathcal G}
\def\H{\mathcal H}

\def\L{\mathcal L}
\def\M{\mathcal M}
\def\N{\mathcal N}

\def\R{\mathcal R}

\def\T{\mathcal T}
\def\V{\mathcal V}

\def\OO{\mathcal O}

\def\EE{\mathbb E}

\def\PP{\mathbb P}
\def\RR{\mathbb R}

\def\bbeta{\boldsymbol \beta}
\def\bx{\bm{x}}
\def\bX{\bm{X}}

\newtheorem{assumption}{Assumption}

\newtheorem{theorem}{Theorem}
\newtheorem{proposition}{Proposition}
\newtheorem{corollary}{Corollary}
\newtheorem{lemma}{Lemma}
\newtheorem{remark}{Remark}

\begin{document}

\title{\textbf{Finite-Sample Distribution Theory and Efficient Large-Scale Inference for Online Quantile Regression}}
\author[1]{Ziyang Wei}
\author[2]{Jiaqi Li}
\author[3]{Lan Wang}
\author[1]{Wei Biao Wu}

\affil[1]{Department of Statistics, University of Chicago}
\affil[2]{Department of Statistics, Rice University}
\affil[3]{Department of Management Science, University of Miami}

\maketitle

\begin{abstract}
This paper studies online quantile regression for large-scale and streaming data using Stochastic SubGradient Descent (SSGD) with constant learning rates. Classical offline inference for quantile regression is computationally and memory intensive. Existing works of online inference for quantile regression provide only asymptotic guarantees and typically require sub-exponential tail conditions for distribution theory. To bridge these gaps, we introduce new techniques to prove a quenched central limit theorem (CLT) and finite-sample Gaussian approximation for SSGD under a finite-moment assumption. We further show that Ruppert-Polyak averaging with a constant learning rate has a non-vanishing bias and fails to satisfy CLT centering at the population target. Hence we propose suffix averaging to address this issue and establish its finite-sample Gaussian approximation. Based on these results, we provide an efficient online inference method for quantile regression that avoids covariance estimation. Numerical experiments show that our method achieves desirable empirical coverage rates and competitive performance compared to other inference methods. We also apply our approach to U.S. wage data to demonstrate its practical effectiveness.
\end{abstract}

\vspace{0.5cm}

\section{Introduction}

The growing prevalence of large-scale streaming data has created substantial demand for statistical methods that can process observations sequentially and efficiently. Transaction systems, digital platforms, and financial markets routinely generate millions of records, often requiring estimates to be updated as new data arrive. Statistical procedures for these settings must accommodate both the volume of the data and their sequential arrival. Repeatedly refitting a model using all accumulated observations can be computationally prohibitive, while storing and revisiting the full dataset may be impractical. Online methods address these challenges by processing observations sequentially and updating estimates with limited memory. For statistical decision-making, however, scalable point estimation alone is insufficient; uncertainty quantification must also be updated online without compromising computational efficiency.

Quantile regression is particularly useful when the goal is to understand how covariates affect different parts of a response distribution. Unlike conditional mean regression, it allows covariate effects to vary across quantile levels, revealing distributional heterogeneity that an average effect can obscure (\cite{koenker_regression_1978,koenker_quantile_2001}). This perspective is relevant to applications involving wage inequality, financial risk, and asymmetric or heterogeneous outcomes: covariates may relate differently to wages at the lower and upper ends of the wage distribution (\cite{buchinsky_educational_2010}), and financial applications often focus on conditional tail quantiles rather than expected outcomes (\cite{engle_caviar_2004}). As datasets in these applications grow, an important methodological question is how to combine the distributional insights of quantile regression with computational efficiency and uncertainty quantification required for large-scale streaming analysis.

This paper develops an online estimation and inference framework for quantile regression based on stochastic subgradient descent (SSGD) (\cite{ermoliev_stochastic_2008,nemirovski_robust_2009,avron_efficient_2012}) with a learning rate held constant across iterations. SSGD is an adaption of stochastic gradient descent (SGD) (\cite{robbins_stochastic_1951}) to non-differentiable loss functions, where each incoming observation generates a recursive update, allowing estimation in a single pass through the data. Our inference procedure uses parallel online recursions and avoids computationally intensive resampling and explicit covariance-matrix estimation. We show its computational advantage in Figure~\ref{fig:intro}(b). The scale motivating our analysis is reflected in numerical experiments with sample sizes of $n=10^6$ and $n=10^7$. At these scales, computational scalability becomes especially important when uncertainty quantification is required alongside point estimation.

Our choice of a constant learning rate is motivated by its practical appeal for large-scale data (\cite{paquette_sgd_2021,li_statistical_2025}). Here, ``constant'' means constant across iterations within a run; the learning rate itself may depend on the sample size in our theoretical analysis. A constant learning rate requires selecting a single learning rate for each run, without additionally specifying an iteration-dependent decay schedule. Unlike a decaying learning rate, it does not automatically shrink the magnitude of an update as the iteration index increases, which can help sustain progress when substantial estimation error remains. With a suitable choice of learning rate, this can accelerate the reduction of initialization error over the available data stream, as illustrated in Figure~\ref{fig:intro}(a). These practical advantages motivate the question of how to obtain valid statistical inference while retaining a constant learning rate.

For large datasets treated as a fixed sample, computational challenges for quantile regression have been addressed through specialized optimization (\cite{portnoy_gaussian_1997}), randomized approximation and subsampling (\cite{yang_quantile_2013,wang_optimal_2021}), distributed or memory-constrained computation (\cite{volgushev_distributed_2019,chen_quantile_2019,chen_quantile_2020}), and smoothing of the quantile-regression objective (\cite{horowitz_bootstrap_1998,fernandes_smoothing_2021,tan_high-dimensional_2022,he_smoothed_2023}). The literature on online quantile regression is more limited. \textcite{shen_online_2025} establish non-asymptotic estimation-error and regret bounds for SSGD, whereas \textcite{lee_fast_2025} develop asymptotically valid inference based on Ruppert-Polyak (RP) averaged SSGD (\cite{ruppert_efficient_1988,polyak_acceleration_1992}) with a decaying learning rate. Neither provides the finite-sample distributional guarantees needed for inference with constant-learning-rate SSGD. Such guarantees remain important even with millions of observations, because approximation accuracy depends on the learning rate, initialization, averaging window, and covariate distribution, as well as the number of observations available to each parallel recursion. In particular, a small bias can remain substantial relative to the standard error (\cite{dieuleveut_bridging_2020,yu_analysis_2021}), so accurate point estimation alone does not ensure well-calibrated confidence intervals. Finite-sample theory quantifies these effects and identifies conditions under which the approximation is sufficiently accurate.

Establishing such distributional guarantees requires more than an estimation-error bound, since controlling the distance between an estimator and its target does not determine the accuracy of a Gaussian approximation or the coverage of a confidence interval. This task is further complicated by the need for weak tail assumptions: existing finite-sample analyses often assume bounded covariates (\cite{fernandes_smoothing_2021}) or sub-Gaussian or sub-exponential tails (\cite{tan_high-dimensional_2022,he_smoothed_2023,zhang_piecewise_2025}). We instead establish Gaussian approximations under a finite $p$-th moment condition on the covariates for some $p\geq4$, together with the other regularity conditions, accommodating heavier-tailed covariates and quantifying distributional approximation errors beyond what estimation-error bounds or asymptotic normality alone provide.

A further difficulty is that holding the learning rate constant changes the distributional behavior of stochastic approximation. Under suitable stability conditions, iterates with a fixed positive learning rate fluctuate around a stationary distribution whose mean may differ from the population target (\cite{pflug_stochastic_1986,dieuleveut_bridging_2020,yu_analysis_2021}). Classical Ruppert-Polyak averaging can achieve target-centered asymptotic normality under suitable decaying learning rates (\cite{ruppert_efficient_1988,polyak_acceleration_1992,moulines_non-asymptotic_2011}), but with a constant learning rate, full-trajectory averaging need not adequately control both initialization effects and stationary bias: we demonstrate this difficulty through a quantile-regression counterexample in which the $\sqrt{n}$-scaled bias of the full-trajectory average remains uniformly bounded away from zero. This finding motivates suffix averaging, which discards an initial portion of the trajectory and averages the subsequent iterates (\cite{rakhlin_making_2012,shamir_stochastic_2013,jain_parallelizing_2018}): discarding early iterates reduces the effect of initialization, while an appropriate sample-size-dependent choice of the constant learning rate controls the remaining bias. See Figure~\ref{fig:intro}(c) as an example. Under joint conditions on the learning rate and averaging window, we establish a quenched CLT and a finite-sample Gaussian approximation for the suffix-averaged estimator, providing the foundation for our parallel online inference procedure and its finite-sample coverage guarantees.

\begin{figure}[htbp!]
    \centering
    \includegraphics[width=0.95\linewidth]{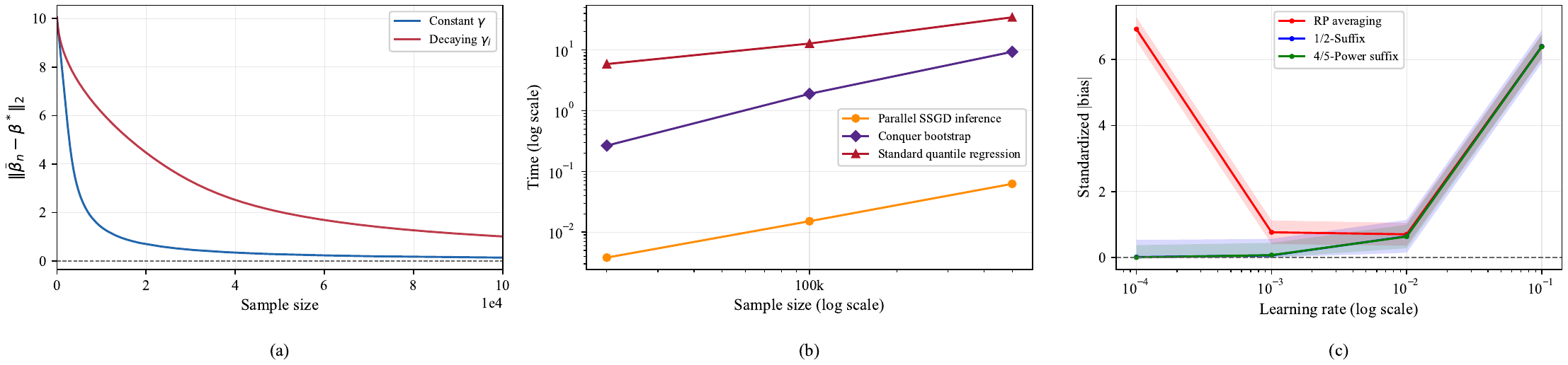}
    \caption{Convergence, inference speed, and bias of RP averaging. (a) compares convergence under constant and decaying learning rates with dimension $d=100$. (b) shows the computational time of three inference methods with dimension $d=100$. Parallel SSGD is the one proposed in this paper, which achieves a speedup of more than $100\times$ over the two competing methods. (c) shows that the $\sqrt{n}$-scaled bias of full RP averaging remains non-negligible across different constant learning rates, whereas suffix averaging substantially reduces this scaled bias. The $1/2$-suffix estimator averages the second half of the iterates, while the $4/5$-power-suffix estimator approximately averages the final $n- n^{4/5}$ iterates.}
    \label{fig:intro}
\end{figure}

\textit{Contributions.} 
Our contributions are threefold.
\textbf{(1)} We establish finite-sample distribution theory for constant-learning-rate SSGD in quantile regression under a finite $p$-th moment condition on the covariates. In particular, we derive consistency (Section~\ref{sec_consistency}) and Gaussian approximation (Section~\ref{subsec_last_iterarte}) for SSGD iterates, relaxing the sub-exponential tail conditions commonly imposed in previous works. \textbf{(2)} We characterize the bias problem associated with the classical Ruppert-Polyak averaged SSGD estimator (Section~\ref{subsec_bias}). 
We address this issue using suffix averaging and establish its Gaussian approximation (Section~\ref{subsec_suffix}) as a theoretical foundation for inference procedures. \textbf{(3)} Building on these results, we justify a parallel online inference procedure for recursive quantile regression with finite-sample coverage guarantees (Section~\ref{sec_method}). Numerical experiments, including datasets with one million and ten million observations, evaluate the procedure's statistical and computational performance, and an application to U.S.\ wage data illustrates its empirical use (Sections~\ref{sec_simu}~\&~\ref{sec_data}). Together, these developments connect the computational demands of large-scale streaming data with rigorous uncertainty quantification for quantile regression.

\textit{Notation.}
For a vector $\bm{v} = (v_{1}, \ldots, v_{d})^{\top} \in\RR^{d}$ and $p \ge 1$, we denote $|\bm{v} |_{p}=(\sum_{i = 1}^{d} |v_{i}|^{p})^{1/p}$ and $|\bm{v} |=|\bm{v} |_2$. For any $s > 0$ and a random vector $\bx$, we say $\bx \in\L^s$ if $\| \bx \|_{s} := (\mathbb{E} |\bx|^{s})^{1/s} < \infty$. For a real symmetric matrix $S$, let $\Lambda_{\min}(S)$ and $\Lambda_{\max}(S) = |S|$  denote its smallest and largest eigenvalue. For two sequences of positive numbers $(a_n)$ and $(b_n)$, we say $a_n=\OO(b_n)$ or $a_n\lesssim b_n$ (resp. $a_n\asymp b_n$) if there exists $C>0$ such that $a_n/b_n\le C$ (resp. $1/C\le a_n/b_n\le C$) for all large $n$, and write $a_n=o(b_n)$ or $a_n\ll b_n$ if $a_n/b_n\rightarrow0$ as $n\rightarrow\infty$. 
Let $(X_n)$ and $(Y_n)$ be two sequences of random variables. Write $X_n = \OO_{\PP}(Y_n)$ if for $\forall \epsilon>0$, there exists $c>0$ such that $\PP(|X_n/Y_n|\le c)>1-\epsilon$. We use $\mathbf{0}_d$ and $\mathbf{1}_{d}$ to denote the all-zero and all-one vectors with dimension $d$, and $\One_{\G}$ to denote the indicator of an event $\G$.

\section{Recursive Quantile Regression}\label{sec_consistency}

This section introduces the recursive quantile regression estimator and establishes its convergence properties. We first review the classical formulation for a fixed data set and then present a stochastic subgradient update for streaming observations with a constant learning rate. We next derive a last-iterate error bound under finite-moment conditions on the covariates, providing the foundation for the distribution theory in subsequent sections.

\subsection{Quantile Regression for a Fixed Data Set}
Consider a univariate response variable $y \in \mathbb{R}$ and a $d$-dimensional covariate vector $\bm{x} = (x_{1}, \ldots, x_{d})^{\top} \in \mathbb{R}^{d}$. 
For a fixed quantile level $\tau \in (0, 1)$, consider the linear quantile regression model 
$F_{y\mid\bm{x}}^{-1}(\tau) = \bm{x}^{\top}\bbeta^*(\tau)$,
where $F_{y\mid\bm{x}}(\cdot)$ denotes the conditional distribution function of $y$ given $\bm{x}$, and $\bbeta^*(\tau)=(\beta_1^*(\tau),\ldots,\beta_d^*(\tau))^{\top}\in \mathbb{R}^{d}$ is the population quantile regression parameter at level $\tau$. For simplicity of notation, we suppress the dependence on $\tau$ and write $\bbeta^{*} = \bbeta^{*} (\tau)$. Equivalently, the model can be expressed as $y = \bm{x}^{\top} \bbeta^{*} + \epsilon$, $\mathbb{P} (\epsilon \leq 0 | \bm{x}) = \tau$.
Given a random sample $\{(y_{i}, \bm{x}_{i})\}_{i = 1}^{n}$ from the joint distribution of $(y, \bm{x})$, the traditional offline quantile regression estimator for $\bbeta^{*}$ is obtained by minimizing the empirical loss function~(\cite{koenker_regression_1978}), 
\begin{align}
    \label{eq_loss}
\min_{\bbeta\in\RR^d} \frac{1}{n}\sum_{i=1}^n\rho_{\tau}(y_i-\bx_i^{\top}\bbeta),
\end{align}
where $\rho_{\tau}(u) = u(\tau-\One_{\{u<0\}})$, $u \in \mathbb{R}$, is the check function. This empirical loss function is not strongly convex and is non-differentiable at points where $y_i = \bx_i^{\top}\bbeta$.

\subsection{Quantile Regression for Streaming Data}
To address the challenges of computational and memory costs in large-scale data, we adopt the Stochastic SubGradient Descent (SSGD) algorithm (\cite{ermoliev_stochastic_2008}) for online quantile regression. Specifically, define the quantile loss function as $g(\bbeta,\bx,y) = \rho_{\tau}(y-\bx^{\top}\bbeta)$ with subgradient $\nabla g(\bbeta, \bm{x},y)  = -\bm{x}(\tau - \One_{\{y-\bm{x}^{\top}\bbeta<0\}})$.

Consider i.i.d.\ observation pairs $(y_1,\bm{x}_1),(y_2,\bm{x}_2),\ldots$ from some unknown distribution $\Pi$. Given a constant learning rate $\gamma>0$ and an initial vector $\bbeta_0 \in \mathbb{R}^{d}$, for $n\ge1$, the $n$-th iteration of the online SSGD update takes the form
\begin{align}\label{onlineqr}
    \bbeta_n(\gamma) = \bbeta_{n-1}(\gamma) -\gamma\nabla g(\bbeta_{n-1}, \bm{x}_n,y_n) =   \bbeta_{n-1}(\gamma) +\gamma\bm{x}_n(\tau - \One_{\{y_n-\bm{x}_n^{\top}\bbeta_{n-1}(\gamma)<0\}}), 
\end{align}
where the subgradient is evaluated at the previous estimate $\bbeta_{n-1}(\gamma)$ using a new observation $(y_n,\bm{x}_n)$. Throughout the paper, we interchangeably write $\bbeta_n(\gamma)=\bbeta_n$ if no confusion is caused.

\subsection{Last-Iterate Convergence under Finite Moments}

Before developing the distributional theory of the SSGD iterates, we first establish their stochastic localization around the target parameter $\bbeta^*$. This preliminary step is essential because a Gaussian approximation cannot be obtained solely from a central limit theorem for the stochastic score. One must first show that the iterates enter and remain in a sufficiently small neighborhood of $\bbeta^*$, within which the population score can be accurately approximated and the resulting remainder terms can be controlled. Establishing such localization is particularly nontrivial for quantile regression because the sample loss is nonsmooth and the covariates are allowed to be heavy-tailed under only a finite $p$th-moment condition.

In particular, we provide a consistency theory for the online SSGD estimator $\bbeta_n$ from recursion~\eqref{onlineqr}. First, we outline the conditions on the distribution of covariates $\bm{x}$ and the ones on the conditional density function of the noise. Denote the conditional density function of $\epsilon$ given $\bx$ by $f_{\epsilon\mid\bx}(\cdot)$. Recall that $\Pi$ denotes the joint distribution of $(y,\bx)$, and let $P_{\bx}$ denote the corresponding marginal distribution of $\bx$.

\begin{assumption}[Conditional density]\label{as1}
There exist constants $r> 0$, $L > 0$, $\mu_{\tau}>0$, and $\M_f > 0$ such that for $P_{\bx}$-almost every $\bx$, the conditional density $f_{\epsilon\mid\bx}(t)$ satisfies the following conditions: (i) $|f_{\epsilon\mid\bx}(t) - f_{\epsilon\mid\bx}(0)| \leq L|t|$ for all $t \in [-r,r]$; (ii) $f_{\epsilon\mid\bx}(t) \leq \M_f$; (iii) $f_{\epsilon\mid\bx}(0) \ge \mu_{\tau}$.
\end{assumption}

\begin{assumption}[Finite moment]\label{as2}
For some $p \ge 4$, there exists a constant $M_p > 0$ such that $ M_p:= \|\bx\|_p  < \infty$. Moreover, $\Lambda_{min} (\EE [\bx\bx^{\top}])>0$.
\end{assumption}

Assumption~\ref{as1} imposes standard regularity conditions on the conditional density $f_{\epsilon\mid\bx}(t)$, including local Lipschitz continuity and boundedness, which are widely used in quantile regression. See for example, \textcite{belloni_1-penalized_2011} assume that the conditional density is differentiable, uniformly bounded, and bounded away from zero at the conditional quantile; \textcite{he_smoothed_2023} impose a Lipschitz condition on the conditional density together with a uniform lower bound at zero; local lower and upper density bounds are also assumed in \textcite{wang_quantile_2012} and \textcite{shen_online_2025}.
Assumption~\ref{as2} requires a finite $p$-th moment of the covariates $\bm{x}$ for some $p\ge4$, together with a positive-definite assumption on the matrix $\EE[\bm{x}\bm{x}^{\top}]$. Similar nondegeneracy conditions are commonly used in quantile regression (\cite{belloni_1-penalized_2011,he_smoothed_2023}). Our Assumption~\ref{as2} is mild compared to recent works on online quantile regression, which assume Gaussian (\cite{shen_online_2025}) or sub-exponential covariates (\cite{zhang_piecewise_2025}). In this paper, we establish non-asymptotic distribution theory and finite-sample inference guarantees under only a finite-moment condition.

Under these assumptions, the following theorem quantifies how the last-iterate error of online SSGD estimator $\bbeta_n(\gamma)$ depends on the sample size $n$ and learning rate $\gamma$.
\begin{theorem}[Convergence of SSGD]\label{thm_moment_sgd}
Suppose that Assumptions \ref{as1}-\ref{as2} hold. Fix $\vartheta \in (1/p,1)$ and $c_{\diamond}>0$. For any $\gamma$ satisfying $0<\gamma \leq 1/(2\lambda)$, $\gamma n^{\vartheta} \le c_{\diamond}$, and every $\delta \in (0,1)$, we have
\begin{align}
\label{Op_consistency}
 \PP\left( |\bbeta_{n}-\bbeta^* |^2 > \frac{C_1}{\sqrt{\delta}}\big((1-\lambda\gamma)^n+\gamma \big) \right)  \leq \delta + M_p^pn^{1-p\vartheta},   
\end{align}
where $C_1$ is a constant independent of $(n,\gamma,\delta)$, and 
\begin{equation}\label{lambda}
    \lambda = \frac{\mu_\tau}{16\varrho}
        \Lambda_{\min}(\EE[\bx\bx^\top]), \quad 
        \varrho =
    \max\left\{1,\,\left(\frac{2M_p^p}{\Lambda_{\min}(\EE[\bx\bx^\top])}\right)^{1/(p-2)}\right\}.
\end{equation}
\end{theorem}
By Theorem~\ref{thm_moment_sgd}, we immediately obtain $|\bbeta_{n}-\bbeta^* |^2  = \OO_{\PP} \big((1-\lambda\gamma)^n+\gamma \big)$, which separates the last-iterate error into an exponentially decaying transient term and a learning-rate-dependent stochastic error. The quantity $\lambda$ defined in~\eqref{lambda} can be viewed as an effective local contraction parameter, and the finite-moment factor $\varrho$ in~\eqref{lambda} reflects the additional cost of controlling potentially heavy-tailed covariates. In particular, a larger $M_p$ leads to a larger $\varrho$, and hence a smaller contraction rate $\lambda$. Accordingly, $(1-\lambda\gamma)^n$ in~\eqref{Op_consistency} shows that the effect of the initialization $\bbeta_0$ decays exponentially fast, and the term $\gamma$ reflects the stochastic fluctuations around the mean of the limiting distribution. Specifically, with a constant learning rate $\gamma$, the SSGD sequence $\bbeta_n(\gamma)$ fluctuates around the stationary region with the magnitude $\sqrt{\gamma}$ as shown by the CLT in \textcite{wei_gaussian_2025}, and therefore the term $\gamma$ in~\eqref{Op_consistency} can not be further improved. By contrast, existing moment bounds for $\bbeta_n$ have typically been established under stronger distributional assumptions on the covariates, such as Gaussianity (\cite{shen_online_2025}) or sub-exponential tails (\cite{zhang_piecewise_2025}). 

For a given positive learning rate $\gamma$, once the transient term is dominated by $\gamma$, the bound~\eqref{Op_consistency} places the last-iterate SSGD $\bbeta_n(\gamma)$ within an $\mathcal{O}_{\PP}(\sqrt{\gamma})$ neighborhood of $\bbeta^*$, provided the conditions in Theorem~\ref{thm_moment_sgd} hold. This neighborhood shrinks as $\gamma\to0$. In particular, choosing $\gamma$ according to the sample size so that $\gamma\to0$ and $n\gamma\to\infty$ yields $|\bbeta_n-\bbeta^*|^2=o_{\PP}(1)$, while keeping the learning rate constant within each run.
Theorem~\ref{thm_moment_sgd} provides the stochastic localization needed for the subsequent analysis, but it does not characterize the distribution of the normalized error. We next strengthen this rate result by deriving a Gaussian approximation 
and explicitly quantifying its finite-sample approximation error.

\section{Finite-Sample Distribution Theory}\label{sec_conc}
Having established the consistency of constant-learning-rate SSGD under finite-moment conditions, we now turn to its non-asymptotic distribution theory, which provides the theoretical foundation for statistical inference. We proceed in three steps. First, we establish a finite-sample Gaussian approximation for the last SSGD iterate of quantile regression. We then investigate averaging schemes and show that classical Ruppert-Polyak averaging in~\eqref{eq_RP}, which averages the entire SSGD trajectory $\{\bbeta_n(\gamma)\}_{n\ge0}$, suffers from a non-vanishing bias under a constant learning rate. This bias can therefore preclude the usual centered CLT about $\bbeta^*$. Third, motivated by this limitation, we consider the suffix averaging estimator in~\eqref{suffix} that discards an initial portion of the trajectory. For this estimator, we establish both a quenched CLT and a finite-sample Gaussian approximation. These results are subsequently used in Section~\ref{sec_method} to derive finite-sample guarantees for online inference.

Throughout the remainder of the paper, we let the learning rate $\gamma$ satisfy $\gamma \gtrsim 1/n$, which is a standard condition in the literature of stochastic approximation (\cite{polyak_acceleration_1992, kushner_stochastic_2003}). This essentially excludes learning rates so small that the recursion makes asymptotically negligible progress over the observed sample.

\subsection{Gaussian Approximation of SSGD}\label{subsec_last_iterarte}
We first study the finite-sample distribution of the last SSGD iterate $\bbeta_n(\gamma)$. Existing distribution theory for constant-learning-rate SGD-type algorithms commonly relies on strong convexity and smoothness of the objective function, together with Lipschitz-type regularity of the stochastic gradient (\cite{dieuleveut_bridging_2020,wei_gaussian_2025}). These techniques cannot be directly applied to quantile regression, because the stochastic subgradient associated with the empirical loss function in~\eqref{eq_loss} is discontinuous in the parameter $\bbeta$ and does not satisfy the stochastic Lipschitz condition (\cite{li_statistical_2025}).
Meanwhile, existing studies of online quantile regression primarily provide estimation error or concentration bounds (\cite{shen_online_2025,zhang_piecewise_2025}), whereas the available inference theory is asymptotic and focuses mainly on
averaged SSGD iterates (\cite{lee_fast_2025}). 
To the best of our knowledge, an explicit finite-sample Gaussian approximation for the last iterate under finite-moment covariates remains unavailable. The following theorem fills this gap by coupling the normalized last iterate with a Gaussian vector and explicitly bounding the coupling error.

\begin{theorem}[Gaussian approximation of SSGD]\label{clt_last}
    Suppose that Assumptions \ref{as1}-\ref{as2} hold. Let $\gamma$ satisfy the same conditions as in Theorem \ref{thm_moment_sgd}, and in addition, let $\gamma < 1/(\M_f|\EE[\bx\bx^\top]|)$. On a sufficiently rich probability space, there exists some Gaussian random vector $\zeta_n\sim \N(\mathbf 0_d,\Gamma_n(\gamma))$ such that for every $\delta \in (0,1)$, 
    $$ \PP\left[\left |\frac{\bbeta_n-\bbeta^*}{\sqrt{\gamma}} - \zeta_n\right | > \frac{C_2\E_n^{\mathrm{SSGD}}}{\sqrt{\delta}} \right]  \leq \delta + M_p^pn^{1-p\vartheta},$$
    where $\E_n^{\mathrm{SSGD}} =  \gamma^{1/4}+(1-\lambda\gamma)^{(n-1)/4}+n\sqrt{\gamma}(1-\lambda\gamma)^{n-1}+\sqrt{\gamma}n^{1/p}\sqrt{\log n}$, and $C_2$ is a constant independent of $(n,\gamma,\delta)$. Equivalently, with probability at least $1-\delta-M_p^p n^{1-p\vartheta}$, $\left|(\bbeta_n-\bbeta^*)/\sqrt{\gamma}-\zeta_n\right|
    \leq C_2\E_n^{\mathrm{SSGD}}/\sqrt{\delta}$.
    The contraction parameter $\lambda$ is defined in \eqref{lambda}, and
\begin{align}\label{Gamma}
\Gamma_n(\gamma) = \sum_{j=0}^{n-1} \gamma ({\bf I}_d-\gamma A)^j S ({\bf I}_d-\gamma A)^j,
\end{align}
where the matrices $A$ and $S$ are respectively defined as
\begin{align}\label{Hessian}
   A= \EE [f_{\epsilon\mid\bx}(0) \bx\bx^{\top}], \quad S=\EE[\nabla g(\bbeta^*,\bx_n,y_n)\nabla g(\bbeta^*,\bx_n,y_n)^{\top}] = \tau(1-\tau) \EE[\bx\bx^{\top}].
\end{align}
\end{theorem}

In Theorem~\ref{clt_last}, the SSGD iterate and its Gaussian counterpart are constructed jointly on a sufficiently rich probability space while preserving the marginal distribution of the original iterate (\cite{karmakar_optimal_2020,mies_sequential_2023}). 
Each term in the coupling error rate $\E_n^{\mathrm{SSGD}}$ has a distinct source. The term $\gamma^{1/4}$ arises from controlling the state-dependent martingale remainder, while $(1-\lambda\gamma)^{(n-1)/4}$ is its transient component before the SSGD iterates converge to the stationary region, and $n\sqrt{\gamma}(1-\lambda\gamma)^{n-1}$ controls the transient nonlinear residual. Finally, $\sqrt{\gamma}n^{1/p}\sqrt{\log n}$ is the Gaussian approximation error for the leading weighted sum under the finite $p$-th moment condition. Its dependence on $n^{1/p}$ agrees, up to a logarithmic factor, with the optimal finite-moment Gaussian-approximation rate in \textcite{karmakar_optimal_2020}. The following Corollary~\ref{cor_clt_last_iterate} is a direct consequence of Theorem~\ref{clt_last}.

\begin{corollary}[Quenched CLT of SSGD]\label{cor_clt_last_iterate}
Let $\Gamma$ denote the unique solution of the Lyapunov equation $A\Gamma + \Gamma A = S$, where $A$ and $S$ are defined in~\eqref{Hessian}. Then,
$\Gamma_n(\gamma)\rightarrow\Gamma$ as $n \gamma \rightarrow \infty$ and $\gamma \rightarrow 0$, where $\Gamma_n(\gamma)$ is defined in~\eqref{Gamma}. Moreover, if $n \rightarrow \infty$ and $\gamma \rightarrow 0$ such that $n \gamma /\log n \rightarrow \infty$ and $n^{2/p}\gamma \log n \rightarrow 0$, then $
    (\bbeta_{n} - \bbeta^{*})/\sqrt{\gamma} \stackrel{D}{\to}  \N(\mathbf 0_d, \Gamma)$.
\end{corollary}
Here, the limiting covariance matrix $\Gamma$ can be equivalently expressed as $\Gamma = \int_{\mathbb{R}^+} e^{-At}Se^{-At}dt$ (\cite{pflug_stochastic_1986}). Different from a classical CLT, Corollary~\ref{cor_clt_last_iterate} provides a quenched CLT which holds for every fixed initialization $\bbeta_0\in\RR^d$. Therefore, it does not require the SSGD recursion~\eqref{onlineqr} to be initialized in stationarity. Moreover, it is obtained under the local Lipschitz density condition in Assumption~\ref{as1}, without requiring stronger smoothness conditions in the existing literature (\cite{dieuleveut_bridging_2020,li_statistical_2025}). Notably, this result applies to every quantile level $\tau\in(0,1)$. This contrasts with the stationary-distribution analysis of \textcite{wei_central_2026}, which relies on a lattice-state Markov-chain argument and is therefore restricted to rational quantile levels.
Our finite-time coupling bypasses this arithmetic restriction and can be of independent interest.

\subsection{Non-Vanishing Bias of Ruppert-Polyak Averaging}\label{subsec_bias}

A classical approach to reducing the variance of stochastic approximation while preserving its online nature is Ruppert-Polyak (RP) averaging (\cite{ruppert_efficient_1988,polyak_acceleration_1992}). It has become a standard tool in stochastic optimization and statistical inference, with extensive asymptotic and finite-sample theory (\cite{moulines_non-asymptotic_2011,gadat_optimal_2023,lee_fast_2025}). For the SSGD trajectory $\{\bbeta_n\}_{n\ge0}$ in~\eqref{onlineqr}, the RP estimator can be updated recursively as follows,
\begin{align}
    \label{eq_RP}
    \bbeta_n^{\mathrm{RP}} = \frac{1}{n}\sum_{i=1}^n\bbeta_i = \frac{1}{n}\big[(n-1)\bbeta_{n-1}^{\mathrm{RP}} + \bbeta_n\big], \quad n\ge1.
\end{align}
With a decaying learning rate $\gamma_i\rightarrow0$ as $i\rightarrow\infty$, RP averaging is well known to attain the optimal
$\sqrt n$ estimation rate and, under suitable regularity and
learning-rate conditions, a centered asymptotic normal distribution
about the population target $\bbeta^*$
(\cite{ruppert_efficient_1988,polyak_acceleration_1992,moulines_non-asymptotic_2011}). 
Its behavior can be drastically different,
however, when a single learning rate is held constant throughout the trajectory.
In this setting, the SSGD or SGD iterates form a time-homogeneous Markov chain that does not converge to the population target $\bbeta^*$, but instead fluctuates according to a stationary distribution (provided one exists) whose mean may differ from $\bbeta^*$.

The following remark draws on the classical setting of strongly convex and smooth SGD to illustrate the competing learning-rate requirements for controlling the transient effect of initialization and the stationary bias when seeking a $\sqrt n$-CLT for full-trajectory RP averaging centered at $\bbeta^*$.

\begin{remark}[Transient and stationary bias]
Consider constant-learning-rate SGD for a strongly convex and sufficiently smooth objective, under suitable assumptions on the gradient noise and the learning rate (\cite{dieuleveut_bridging_2020,yu_analysis_2021}). Given a fixed $\gamma>0$, let $\pi_\gamma$ denote the unique stationary distribution of the SGD sequence $\{\bbeta_n^{\mathrm{SGD}}(\gamma)\}_{n\ge0}$ and let $\bbeta_\gamma^\pi=\int \bbeta\pi_\gamma(d\bbeta)=\lim_{n\rightarrow\infty}\EE[\bbeta_n^{\mathrm{SGD}}(\gamma)]$ denote its stationary mean. Under the geometric ergodicity established in this setting, the corresponding Markov-chain CLT gives
\begin{align*}
    \sqrt{n}\big\{\frac{1}{n}\sum_{i=1}^n\bbeta_i^{\mathrm{SGD}}(\gamma) -\bbeta_\gamma^\pi\big\}
\Rightarrow
\N(\mathbf{0}_d,\Sigma_\gamma),
\end{align*}
where $\Sigma_{\gamma}$ is the corresponding long-run covariance matrix (\cite{glynn_liapounov_1996,dieuleveut_bridging_2020,yu_analysis_2021}). 
Thus, the fixed-$\gamma$ Gaussian limit is centered at the stationary mean $\bbeta_\gamma^\pi$, rather than necessarily at the population target $\bbeta^*$. Indeed, 
\begin{align*}
    \sqrt{n}\EE\big[\frac{1}{n}\sum_{i=1}^n\bbeta_i^{\mathrm{SGD}}(\gamma)-\bbeta^*\big]
=
\sqrt{n}\EE\big[\frac{1}{n}\sum_{i=1}^n\bbeta_i^{\mathrm{SGD}}(\gamma)-\bbeta_\gamma^\pi\big]
+
\sqrt{n}\big(\bbeta_\gamma^\pi-\bbeta^*\big).
\end{align*}
The first term is the scaled transient bias due to initialization, whereas the second is the scaled stationary bias. Reducing $\gamma$ can decrease the stationary bias while slowing the decay of the initialization effect. A sufficient condition for the total $\sqrt{n}$-scaled bias to vanish is that both the scaled transient and stationary biases converge to zero simultaneously.
\end{remark}

The preceding remark highlights a potential conflict between transient and stationary biases, but does not establish their incompatibility: a leading bias coefficient may vanish, or the two biases may cancel. The following scalar quantile-estimation example demonstrates that the total $\sqrt{n}$-scaled bias cannot vanish by establishing a strictly positive lower bound, uniformly over the learning rates under consideration.

\begin{proposition}[Bias of standardized full-trajectory RP averaging]\label{bias}
Let $y_{1}, y_{2}, \ldots \in \mathbb{R}$ be i.i.d.~random variables with cumulative distribution function $\mathbb{P}(y_{1} \leq z) = z^{2} \mathbbm{1}\{0 \leq z \leq 1\} + \mathbbm{1}\{z > 1\}$, $z \in \mathbb{R}$. The population median is $\beta^*=\sqrt{2}/2$. 
Consider the scalar online SSGD estimator in~\eqref{onlineqr} with $\tau=1/2$, $\beta_0=0$, and deterministic covariates $x_i\equiv1$ for all $i$. There exists a universal constant $\B>0.1$ such that
\begin{align*}
    \inf_{n \geq 1,\,\gamma\in(0,0.1]} |\sqrt n\,\EE(\beta_n^{\mathrm{RP}}-\beta^*)| \geq \B.
\end{align*}
\end{proposition}
Consequently, uniformly over $\gamma\in(0,0.1]$, including choices of $\gamma\rightarrow0$ as $n\rightarrow\infty$, the absolute bias of $\beta_n^{\mathrm{RP}}(\gamma)$ is bounded below by $\B/\sqrt n$. Thus, the bias of full-trajectory RP averaging is not $o(n^{-1/2})$, regardless of how the constant-within-run learning rate is chosen. This conclusion does not contradict an $\mathcal O(n^{-1/2})$ estimation-error rate or an $\mathcal O(n^{-1})$ mean-squared-error rate. Rather, it shows that the bias can remain a first-order component of the estimation error and thus cannot be omitted from first-order Gaussian approximation centered at $\beta^*$.

\subsection{Gaussian Approximation of Suffix-Averaging}\label{subsec_suffix}

The preceding results identify two sources of bias under a learning rate held constant within each run: the contribution of early iterates generated during the transient phase and the stationary bias induced by the constant learning rate. These two components cannot generally be controlled simultaneously when the entire trajectory is averaged. This motivates \emph{suffix averaging}, which discards an initial portion of the SSGD trajectory $\{\bbeta_n(\gamma)\}_{n\ge0}$ and averages only the subsequent iterates. Suffix averaging was originally introduced to improve the convergence of nonsmooth stochastic optimization (\cite{rakhlin_making_2012,shamir_stochastic_2013}), and related tail-averaging schemes have been studied for more general weighted SGD procedures (\cite{jain_parallelizing_2018}).

Let $T_n:=T(n)$ denote the index of the first SSGD iterate included in the average. By Theorem~\ref{thm_moment_sgd}, the influence of the initialization decays geometrically at rate $(1-\lambda\gamma)^k$. Choosing $T_n$ sufficiently large therefore allows the iterates included in the average to substantially forget the initialization, while the remaining $n-T_n+1$ iterates allow the average to reduce stochastic fluctuations. We therefore define the suffix averaging estimator as 
\begin{align}\label{suffix}
\bbeta_n^{\mathrm{SF}} = \frac{1}{n-T_n+1}\sum_{k=T_n}^n \bbeta_k.
\end{align}
Examples of choices of $T_n$ include $T_n = n-\lceil \kappa n\rceil+1$ (\cite{rakhlin_making_2012}) and $T_n \asymp n^{\kappa}$ for some $0<\kappa \leq 1$. 

The next theorem provides a finite-sample Gaussian approximation for the suffix-averaged SSGD estimator. Unlike a conventional CLT, it quantifies the finite-sample error between $\sqrt n(\bbeta_n^{\mathrm{SF}}-\bbeta^*)$ and a centered Gaussian vector, in terms of $n$, $\gamma$ and $T_n$.

\begin{theorem}[Gaussian approximation for suffix-averaged SSGD]\label{mainGA}
Consider the suffix-averaged SSGD defined in \eqref{suffix} and let $n-T_n \asymp n $. Under the same conditions as in Theorem \ref{clt_last}, on a potentially enriched probability space, there exists some Gaussian random vector $\zeta_n\sim \N(\mathbf 0_d,V_n(\gamma))$ such that for every $\delta \in (0,1)$, 
 $$ \PP\left[\left| \sqrt{n}\big(\bbeta_n^{\mathrm{SF}}-\bbeta^*\big)  - \zeta_n\right|  > \frac{C_3\E_n^{\mathrm{SF}}}{\sqrt{\delta}} \right]  \leq \delta + M_p^pn^{1-p\vartheta},$$
 where $\E_n^{\mathrm{SF}} = \sqrt{n}\gamma + (1-\lambda\gamma)^{T_n-1}/\sqrt{n}\gamma + (T_n-1)(1-\lambda \gamma)^{T_n-1}/\sqrt{n} + (1-\lambda\gamma)^{T_n/4-1/4}/\sqrt{n\gamma}+ \gamma^{1/4} + n^{1/p-1/2}\sqrt{\log n},$
and $C_3$ is a constant independent of $(n,\gamma,\delta)$. Equivalently, with probability at least $1-\delta-M_p^p n^{1-p\vartheta}$, $\left|
        \sqrt n\big( \bbeta_n^{\mathrm{SF}}-\bbeta^*\big)-\zeta_n
    \right| \leq C_3\E_n^{\mathrm{SF}}/\sqrt{\delta}$.
Define $\Xi_i = ({\bf I}_d-\gamma A)^i$ and $\chi(Q_1,Q_2)=A^{-1}(Q_1-Q_2)S(Q_1-Q_2)A^{-1}$ for matrices $Q_1$ and $Q_2$. The finite-sample covariance $V_n(\gamma) = n[ \sum_{k=1}^{T_n-1} \chi( \Xi_{T_n-k},\Xi_{n-k+1})+\sum_{k=T_n}^{n}  \chi( {\bf I}_d , \Xi_{n-k+1}) ]/(n-T_n+1)^2$.
\end{theorem}

\begin{remark}[Choice of suffix initialization and learning rate]\label{rmk_burnin_choice}
To illustrate how the approximation error depends on the suffix starting point and learning rate, consider $T_n =  \max\{\lfloor \T n^{\kappa} \rfloor, 1\}$ and $\gamma = c_0n^{-\nu}\log n$, where $c_0>0$, $\T \in (0,1)$, and $1/2<\nu\le\kappa\le1$. The restriction $\T<1$ ensures that the suffix length satisfies $n-T_n\asymp n$ even when $\kappa=1$. Then, 
$|\sqrt{n}(\bbeta_n^{\mathrm{SF}}-\bbeta^*) - \zeta_n |  = \OO_{\PP}(\phi(n))$, where
\begin{align}
\label{main_Gaussian}
\phi(n):= \ & n^{1/2-\nu} \log n + n^{\kappa-1/2} \exp\left(-\lambda c_0 \mathcal{T} n^{\kappa-\nu} \log n\right) \nonumber \\
& + \frac{n^{\nu/2-1/2}  \exp\left(-\lambda c_0 \mathcal{T} n^{\kappa-\nu} \log n /4\right)}{\sqrt{\log n}}  + n^{-\nu/4} (\log n)^{1/4} + n^{1/p-1/2} (\log n)^{1/2}.
\end{align}
Here $\nu\le\kappa$ ensures that the effective burn-in prefix length $\gamma T_n$ is at least of logarithmic order, so the transient terms decay. The first and fourth terms respectively reflect the nonlinear and nonsmooth SSGD approximation errors, and the last term is the classical Gaussian approximation error under a finite $p$-th moment assumption.
Within this class of choices, the upper bound $\phi(n)$ is minimized by taking $\kappa = \nu = 1$ and $c_0\lambda\mathcal{T} \ge 1$. Under this choice,
\begin{align}
\label{optimalrate}
|\sqrt{n}(\bbeta_n^{\mathrm{SF}}-\bbeta^*) - \zeta_n |  = \OO_{\PP} \Big(\Big(\frac{\log n}{n} \Big)^{1/4}+ n^{1/p-1/2 }(\log n)^{1/2}\Big).
\end{align}
When $p=4$, the second term in~\eqref{optimalrate} from Gaussian approximation dominates the bound, whereas for $p>4$, it decays polynomially faster and the $(\log n/n)^{1/4}$ term dominates.
\end{remark}

Theorem~\ref{mainGA} provides a finite-sample Gaussian approximation for the suffix-averaged SSGD estimator at the $n^{-1/2}$ scale. 
The result is extended to a general stochastic approximation setting in Appendix~\ref{sec_proof_suffix_clt} (see Lemma~\ref{GA_gSGD}). For comparison, the last iterate of SSGD  fluctuates at scale $\sqrt{\gamma}$ (Theorem~\ref{clt_last}). 
Whenever $n\gamma\to\infty$, $n^{-1/2}/\sqrt\gamma=1/\sqrt{n\gamma}\rightarrow0$. Thus, after discarding a sufficiently long initial period, suffix averaging produces a smaller stochastic scale than the last iterate.
Moreover, as discussed in Remark~\ref{rmk_burnin_choice}, the Gaussian approximation error is of order $n^{-1/4}$ up to logarithmic factors, which is sharp and consistent with the classical $n^{-3/4}$ order (after rescaling by $\sqrt{n}$) in Bahadur-Kiefer-type representations for quantile and quantile-regression estimators (\cite{he_general_1996,belloni_conditional_2019}). 

The covariance $V_n(\gamma)$ in Theorem~\ref{mainGA} further clarifies the asymptotic efficiency of suffix averaging. Consider the regime $n\gamma \rightarrow \infty$ and $n/(n-T_n+1)\rightarrow w$ for some constant $w \ge  1$. In Appendix \ref{sec_proof_suffix_clt} (see Proposition \ref{truncated_V}), we establish that $V_n(\gamma) \rightarrow V := wA^{-1}SA^{-1}$.
In particular, when the excluded initial portion in suffix-averaged SSGD is asymptotically negligible, i.e., $T_n=o(n)$, then we have $w=1$. Hence, suffix-averaged SSGD is asymptotically efficient relative to the classical offline quantile-regression benchmark, in the sense that its limiting covariance is $A^{-1}SA^{-1}$, the asymptotic sandwich covariance for offline quantile regression (\cite{koenker_regression_1978}). For the choice of $T_n\asymp n^\kappa$, this corresponds to $\kappa<1$.

By Theorem~\ref{mainGA}, an appropriate choice of suffix can therefore eliminate the initialization effect, recover the $\sqrt n$ inferential scale, and when $T_n=o(n)$, attain the classical quantile-regression covariance. When the covariates are sub-exponentially distributed, we show in Remark~\ref{rmk_subexp} in the appendix that the last-iterate and suffix-average Gaussian couplings can further be strengthened to global $\mathcal L^2$ bounds with the same rates. To the best of our knowledge, such finite-sample Gaussian approximation has not previously been available for online quantile regression with a constant learning rate under only a finite-moment condition.

\section{Scalable and Efficient Online Inference}
\label{sec_method}
Inference for quantile regression commonly requires either estimating the sandwich covariance matrix, including the conditional density at the target quantile, or applying a resampling method such as the bootstrap (\cite{feng_wild_2011,he_smoothed_2023}). In online settings, these approaches may require additional nuisance estimation, repeated resampling, or maintaining auxiliary stochastic sequences (\cite{fang_online_2018}). More recently, \textcite{lee_fast_2025} developed a random-scaling procedure for online quantile regression that avoids direct covariance estimation and yields asymptotically valid confidence intervals (\cite{zhu_high_2024}). Its limiting pivotal distribution is nonstandard, however, so its critical values must be obtained numerically, which can be inconvenient when inference is required at multiple or high confidence levels. Motivated by the gap in the literature, we develop an efficient online inference framework for large-scale quantile regression using suffix-averaged SSGD and derive an explicit finite-sample bound on the coverage error.
\subsection{Recursive Computing of Confidence Intervals}
The proposed parallel inference procedure converts independent suffix-averaged SSGD replicates into a Studentized statistic. It requires neither direct estimation of the quantile-regression sandwich covariance matrix nor bootstrap resampling. Instead, the sampling variability is estimated from the variation across a fixed number of independent SSGD replicates. The procedure is single-pass and can be implemented recursively on parallel data streams to construct valid confidence intervals.

Suppose, for simplicity, that the total sample size satisfies $N=Kn$, where $K\geq2$ is fixed. Partition the observations into $K$ disjoint streams, each containing $n$ observations. On every stream, run the same suffix-averaged SSGD algorithm using the same learning rate and suffix starting point $T_n$, 
producing the suffix average $\bbeta^{\mathrm{SF}}_{n,(k)}$. Since the $K$ streams are independent and the algorithm is identical across runs, the outputs $\{\bbeta^{\mathrm{SF}}_{n,(k)}\}_{k=1}^K$ are i.i.d.\ and asymptotically normal. For any coordinate $j$, or more generally, any linear functional $\bm{v}^{\top}\bbeta^*$ with some unit vector $\bm{v}\in\RR^d$, define the parallel average and the sample variance of the estimators by
\begin{equation}
    \label{eq:parallel_avg}
    \bar{\bbeta}_{n,K} = \frac{1}{K}\sum_{k=1}^K \bbeta^{\mathrm{SF}}_{n,(k)}, \qquad \hat{\sigma}^2_{\bm{v}} = \frac{1}{K-1}\sum_{k=1}^K \big(\bm{v}^{\top}\bbeta^{\mathrm{SF}}_{n,(k)} -\bm{v}^{\top} \bar{\bbeta}_{n,K}\big)^2.
\end{equation}
The procedure can be implemented directly as data arrive. In particular, the observations may be assigned cyclically to $K$ machines $\mathcal M_1,\ldots,\mathcal M_K$: the observation with global index $\ell=k+K(i-1)$ is the $i$-th local observation processed by machine $\mathcal M_k$. For a prespecified reporting horizon $n$, each machine maintains its current SSGD iterate and a running suffix sum, $\bm s_{i,(k)}=\bm s_{i-1,(k)}+\mathbbm{1}\{i\geq T_n\}\bbeta_{i,(k)}$, with $\bm s_{0,(k)}=\mathbf 0_d$. At the end of the local stream, $\bbeta_{n,(k)}^{\mathrm{SF}}=\bm s_{n,(k)}/(n-T_n+1)$. Thus, the suffix average can be computed without storing the trajectory.

We studentize the scaled estimation error of $\bar{\bbeta}_{n,K}$ using
\begin{equation}
    \label{eq:t_stat}
    \hat{t}_{\bm{v}} = \frac{\sqrt{K}\big(\bm{v}^{\top}\bar{\bbeta}_{n,K} - \bm{v}^{\top}\bbeta^*\big)}{\hat{\sigma}_{\bm{v}} }.
\end{equation}
Under the conditions of Theorem \ref{mainGA}, with $K$ fixed and the Gaussian approximation error tending to zero, $\hat{t}_{\bm{v}}$ converges to the $t_{K-1}$ distribution. Consequently, an asymptotic $(1-\alpha)\times 100\%$ confidence interval for $\bm{v}^{\top}\bbeta^*$ is 
\begin{equation}
    \label{eq:parallel_ci}
    \widehat{\text{CI}}_{\bm{v}} = \left[\bm{v}^{\top}\bar{\bbeta}_{n,K} \;\pm\; t_{1-\alpha/2,\,K-1}\,\frac{\hat{\sigma}_{\bm{v}}}{\sqrt{K}}\right],
\end{equation}
where $t_{1-\alpha/2,\,K-1}$ denotes the $(1-\alpha/2)$ quantile of the $t_{K-1}$ distribution. Notice that the online updating in \eqref{eq:parallel_avg} requires only $\OO(d)$ computation and storage per iteration. The total storage across all $K$ machines is $\mathcal O(Kd)$, while only the final suffix averages need to be communicated when a confidence interval is constructed. The complete recursive procedure is summarized in Algorithm~\ref{alg:parallel_qr} in Appendix~\ref{sec_supp_sim}.

\subsection{Non-Asymptotic Student-\texorpdfstring{$t$}{t} Approximation}
Theorem~\ref{mainGA} provides a finite-sample Gaussian approximation for each suffix-averaged SSGD estimator. For $K$ independent data streams used in parallel inference, this allows the joint distribution of the corresponding SSGD replicates to be compared with that of $K$ independent Gaussian replicates. Studentizing these Gaussian replicates produces an exact Student-$t$ pivot. Theorem~\ref{coverage} below quantifies the deviation of the proposed $t$-statistic from its ideal Student-$t$ distribution and provides finite-sample bounds on both the distributional error of $\hat t_{\boldsymbol v}$ and the coverage error of the confidence interval.

\begin{theorem}[Non-asymptotic convergence of coverage rate]\label{coverage}
For each batchwise SSGD sequence with $n =N/K$ iterations, let $T_n =  \max\{\lfloor \T n^{\kappa} \rfloor, 1\}$ and $\gamma = c_0n^{-\nu}\log n$ with $\T \in (0,1)$ and $1/2 < \nu \leq \kappa \leq 1$. Under the same conditions as in Theorem \ref{mainGA}, for any fixed nonzero vector $\bm{v} \in \RR^d$ and $K \ge 2$, we have
\begin{align*}
    \sup_{s \in \RR}| \PP(\hat{t}_{\bm{v}} \ge s) - \PP(T_{K-1}\ge s)|
    \leq C_4\{\phi(n)\}^{1/4}+K M_p^p n^{1-p/2},
\end{align*}
where $T_{K - 1}$ denotes a random variable following $t_{K - 1}$ distribution, $\phi(\cdot)$ is defined in \eqref{main_Gaussian}, and $C_4$ is a constant independent of $(n,N,\gamma)$.
\end{theorem}


Theorem \ref{coverage} also implies a finite-sample guarantee for high confidence levels as stated below.

\begin{corollary}\label{Corollary_cov}
Under the same conditions as in Theorem~\ref{coverage}, for any fixed confidence level $1-\alpha\in(0,1)$, $\left|\PP\left(
\bm v^\top\bbeta^*\in\widehat{\mathrm{CI}}_{\bm v}
\right)-(1-\alpha)\right|
\leq 2C_4\{\phi(n)\}^{1/4}+2K M_p^p n^{1-p/2}
\lesssim
\left\{\phi\left(n\right)\right\}^{1/4}$.
\end{corollary}

\section{Numerical Experiments}\label{sec_simu}
In this section, we conduct simulation studies scaling up to $n =10^7$ total samples with thousands of dimensions to empirically demonstrate the validity and efficiency of our proposed parallel inference methodology, and draw comparisons to the existing relevant inference approaches.

We generate data as $y_i = \bm{x}_i^{\top}\bbeta^* + \epsilon_i$ for $i = 1, \ldots, n$, where the true coefficient $\bbeta^* = \mathbf{1}_{d+1}$. The covariate $\bm{x}_i$ is a $(d+1)$-dimensional vector with the first intercept-element equal to $1$ and remaining elements i.i.d.\ standard normal. The noise $\{\epsilon_i\}$ is also i.i.d.\ standard normal mutually independent of $\{\bx_i\}$. Under this setting, for the quantile level $\tau=1/2$, we have 
$A = {\bf I}_{d+1}/\sqrt{2\pi}, \quad S = \tau(1-\tau) {\bf I}_{d+1}$, where $A$ and $S$ are defined in \eqref{Hessian}. We focus on the inference for the first non-intercept slope coefficient (the second element of $\bbeta^*$) without loss of generality. Unless otherwise stated, we set the suffix proportion and exponent $\T=0.1$, $\kappa=1$, the target quantile level $\tau = 0.5$, $d = 20$, and the initialization at $\mathbf{0}_{d+1}$. The number of batches for parallel inference is $K=6$ according to the default choice in \textcite{zhu_high_2024}.

\subsection{Gaussian Approximation}

Figure~\ref{fig:hist_ssubgd} reports Monte Carlo results based on $10000$ independent replications, with sample size $n=10^6$ and learning rate $\gamma = 0.01$. The experiment illustrates the accuracy of the Gaussian approximations for the last-iterate and suffix-averaged SSGD estimators at a large but finite sample size. In particular, for coordinate $j$, we examine the standardized statistics $(\beta_{n,j}-\beta_j^*)/\sqrt{\gamma\,\Gamma_{jj}}$ and $\sqrt{n}\bigl(\beta_{n,j}^{\mathrm{SF}}-\beta_j^*\bigr)/\sqrt{V_{jj}}$, where the theoretical asymptotic covariance matrices are $\Gamma=\sqrt{2\pi}\,{\bf I}_{d+1}/8$ for the scaled last iterate and $V=w\pi{\bf I}_{d+1}/2$ for the scaled suffix average. Under the choices $\mathcal T=0.1$ and $\kappa=1$, we have $w=10/9$. The histograms and Q-Q plots show close agreement with $\N(0,1)$ after being standardized by theoretical standard deviations and centered at the true parameter $1$, consistent with the Gaussian approximations theory in Theorems~\ref{clt_last} and~\ref{mainGA}.

\begin{figure}[htbp!]
    \centering
    \includegraphics[width=0.98\linewidth]{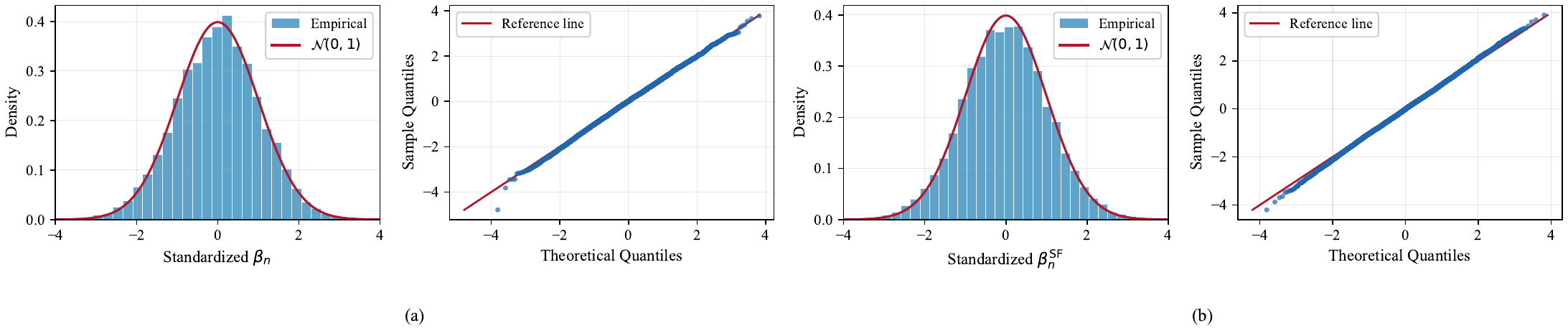}
    \caption{Asymptotic normality: the density plot and QQ plot of standardized estimation errors. (a) last-iterate SSGD. (b) suffix averaged SSGD.}
    \label{fig:hist_ssubgd}
\end{figure}

\subsection{Time Efficiency}
We next compare the computational costs of the competing estimation and inference procedures. For each scenario, the reported run times are averaged over five repetitions. In particular, we compare the SSGD with the full-batch \textsc{conquer} estimator and standard quantile regression implemented by \texttt{statsmodels}. For statistical inference, we compare the parallel inference procedure with the \textsc{conquer} multiplier bootstrap using 200 bootstrap draws, the standard quantile-regression sandwich interval, and the random-scaling method.

As shown in Figure~\ref{fig:estimation_speed}, the computational gap between the online and full-batch procedures widens rapidly as the sample size increases. The SSGD algorithm is more than 100 times faster than the full-batch alternatives. The advantage carries over to inference: the parallel SSGD confidence interval is over $100$  times faster than standard quantile-regression inference and $40$ to $50$ times faster than the \textsc{conquer} bootstrap. The right panel focuses on the two fully online procedures and compares parallel inference with random scaling under the same total sample budget. Our proposed approach exhibits superior efficiency for $(n,d) = (3\times 10^6,500)$ and $(10^7,1500)$. These results highlight the computational benefit of combining one-pass SSGD estimation with the parallel construction of confidence intervals. 

\begin{figure}[htbp!]
    \centering
    \includegraphics[width=0.95\linewidth]{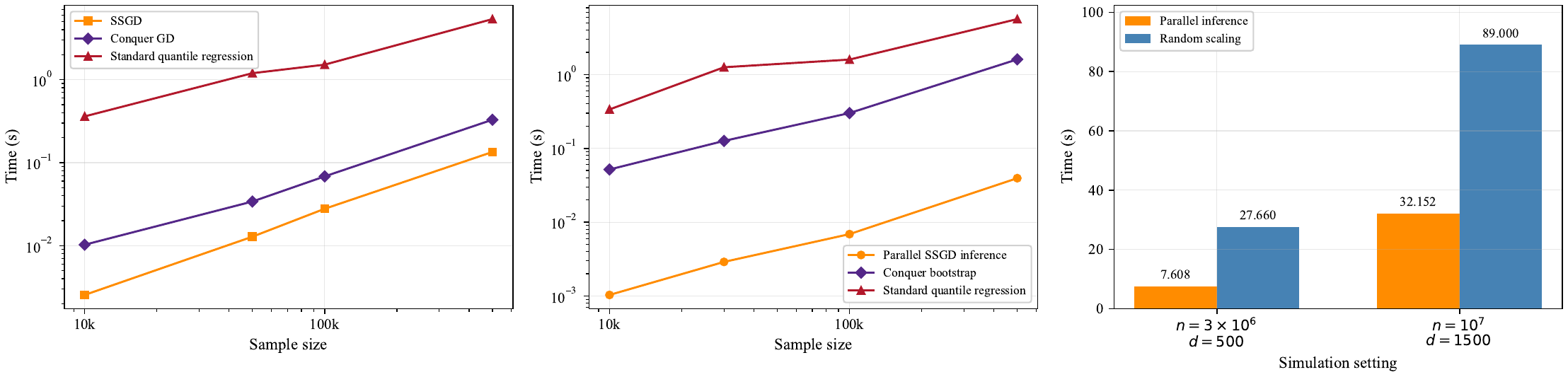}
    \caption{ Estimation and inference wall-clock time vs. sample size and dimension. The time for the first two panels is plotted on a logarithmic scale on the y-axis. The SSGD algorithm and parallel inference method are orders of magnitude faster.}
    \label{fig:estimation_speed}
\end{figure}

\subsection{Statistical Inference}

Finally, we evaluate the empirical performance of the parallel inference method and compare it with the random-scaling procedure in \textcite{lee_fast_2025}, a directly comparable fully online inference procedure for quantile regression. We consider two large-scale designs, $(n,d)=(10^7,1500)$ and $(3\times10^6,500)$ with the learning rates set to $0.001$ and $0.002$, respectively. The random scaling allows only for a decaying learning rate, which we set as $\gamma_k = k^{-0.501}$, same as in \textcite{lee_fast_2025}. Both methods are initialized randomly without a warm start. Figure~\ref{fig_inference} reports empirical coverage probabilities and confidence-interval lengths as the iterations progress. We use $1000$ Monte Carlo replications for the $95\%$ nominal level and increase this number to $10000$ for the $99\%$ level, where the smaller non-coverage probability requires substantially more replications for an accurate Monte Carlo assessment.

Across both designs and confidence levels, the coverage of the parallel procedure quickly stabilizes around the nominal level, while its confidence intervals become progressively narrower. In contrast, random scaling slowly converges and consistently overcovers, with empirical coverage remaining close to $100\%$. Its intervals are also substantially wider and exhibit considerably greater variability across replications. Thus, its high coverage reflects conservative uncertainty quantification rather than accurate calibration. Related conservatism of self-normalized procedures is also observed in other settings (\cite{shao_self-normalized_2010}). 

Additional simulation results are available in Table~\ref{tb_inf_a005_d500}-\ref{tb_inf_a001_d1500} in Appendix~\ref{sec_supp_sim}. Overall, these results provide strong empirical evidence that the proposed parallel inference approach delivers well-calibrated and precise uncertainty quantification in quantile regression with massive datasets, especially for high confidence levels. 

\begin{figure}[hbtp!]
    \centering
    \begin{subfigure}{0.45\textwidth}
        \centering
        \includegraphics[width=\linewidth]{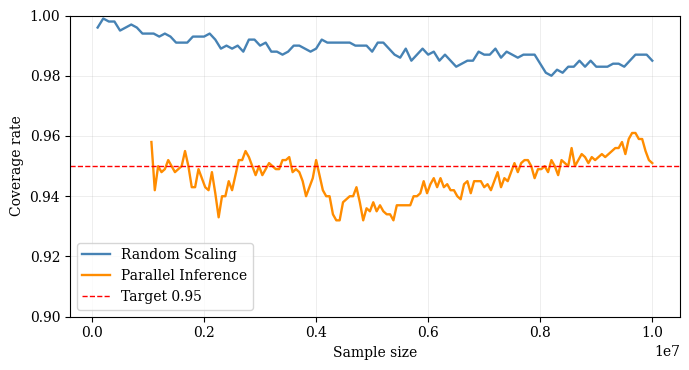}
    \end{subfigure}
    \hfill
    \begin{subfigure}{0.45\textwidth}
        \centering
        \includegraphics[width=\linewidth]{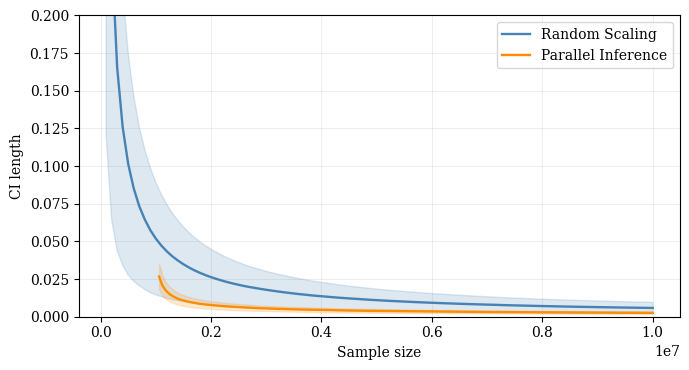}
    \end{subfigure}
    \vspace{1em}
    \begin{subfigure}{0.45\textwidth}
        \centering
        \includegraphics[width=\linewidth]{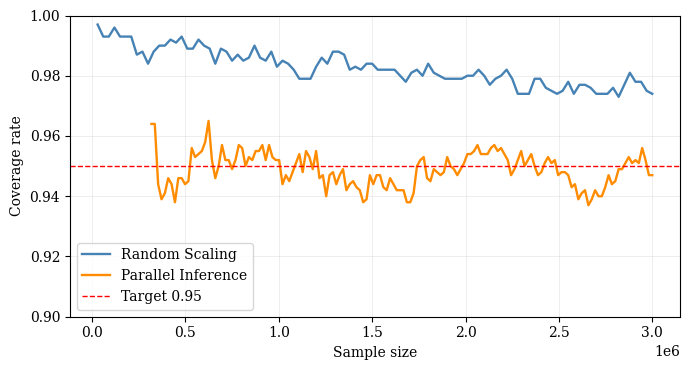}
    \end{subfigure}
    \hfill
    \begin{subfigure}{0.45\textwidth}
        \centering
        \includegraphics[width=\linewidth]{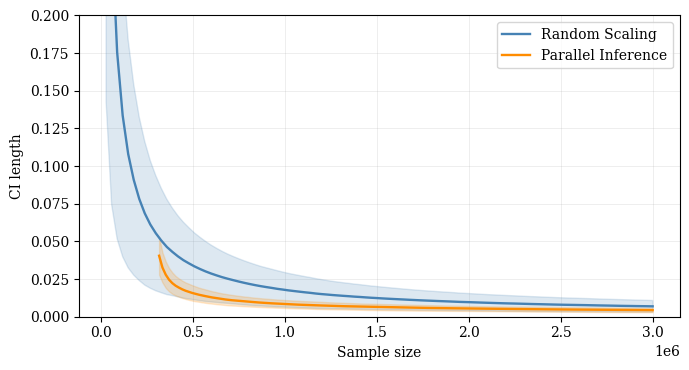}
    \end{subfigure}
    \vspace{1em}
    \begin{subfigure}{0.45\textwidth}
        \centering
        \includegraphics[width=\linewidth]{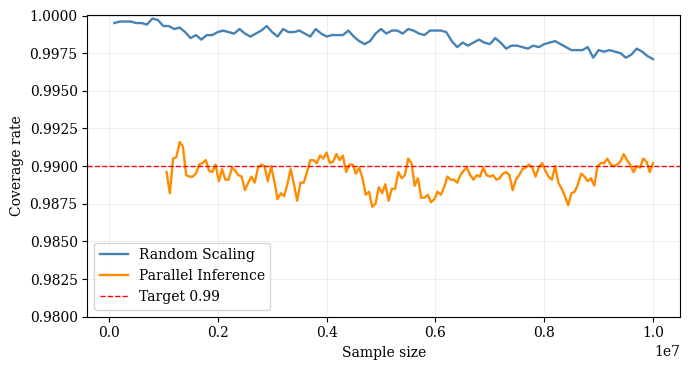}
    \end{subfigure}
    \hfill
    \begin{subfigure}{0.45\textwidth}
        \centering
        \includegraphics[width=\linewidth]{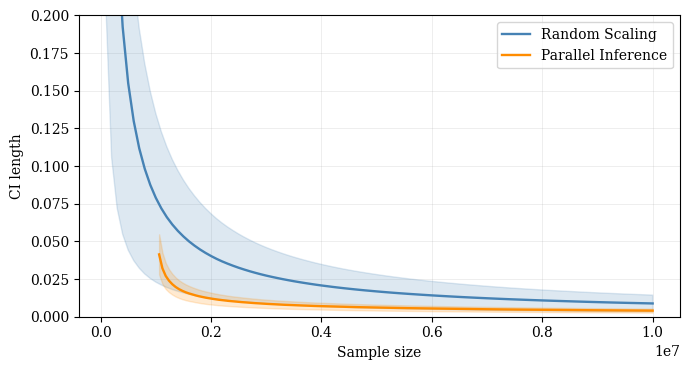}
    \end{subfigure}
    \vspace{1em}
    \begin{subfigure}{0.45\textwidth}
        \centering
        \includegraphics[width=\linewidth]{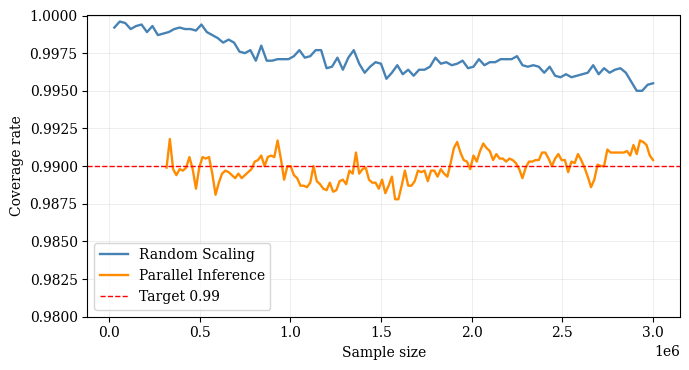}
    \end{subfigure}
    \hfill
    \begin{subfigure}{0.45\textwidth}
        \centering
        \includegraphics[width=\linewidth]{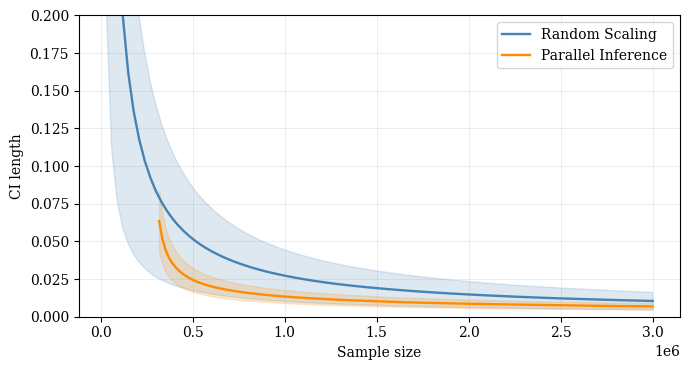}
    \end{subfigure}
    \caption{Parallel inference and random scaling -- empirical coverage and length of confidence intervals. The first two rows: $(n,d, 1-\alpha) \in \{(10^7, 1500, 95\%),(3 \times10^6, 500, 95\%)\}$ and 1000 replications. The last two rows: $(n,d, 1-\alpha) \in \{(10^7, 1500, 99\%),(3 \times10^6, 500, 99\%)\}$ and 10000 replications. The shadows represent the error bars with one standard deviation.}
    \label{fig_inference}
\end{figure}

\section{Real Data Application}\label{sec_data}

In this section, we demonstrate the practical scalability of the parallel SSGD procedure for quantile regression by examining the wage differentials and their associations with worker characteristics and labor-market institutions, at different points of the wage distribution. We use the 2017--2023 Merged Outgoing Rotation Group (MORG) Earnings Data from the Current Population Survey (CPS), distributed by the National Bureau of Economic Research (NBER). The CPS has long been an important data source for research on the U.S.\ wage distribution (\cite{buchinsky_changes_1994,buchinsky_quantile_1995}). 

We employ a Mincer-style specification with \(d=24\) regressors, including an intercept, standardized age and its square, a male indicator, four education dummies (high-school graduate, some college or an associate degree, bachelor's degree, and advanced degree), a Hispanic indicator, three race indicators, marital status, union membership, a public-sector indicator, three Census-region dummies, and six year indicators for 2018–2023, with 2017 as the reference year. For each worker, we construct the nominal hourly wage by dividing reported weekly earnings by usual weekly hours. To account for inflation, we convert nominal wages into constant 2023 dollars using the Consumer Price Index for All Urban Consumers (CPI-U), a standard measure of changes in consumer prices published by the U.S. Bureau of Labor Statistics. We then take the natural logarithm of the inflation-adjusted hourly wage to obtain the outcome variable. The sample is restricted to wage-and-salary workers aged 25--64 with valid earnings and hours, and the total sample size is $n = 715103$.

Following notation in Section~\ref{sec_consistency}, we model the conditional quantile of log hourly wages as $F^{-1}_{y\mid \bx_i}(\tau) = \bx_i^\top \bbeta^*(\tau) = \sum_{r=1}^{d} x_{ir}\beta_r^*(\tau)$, $\tau\in(0,1)$, where $x_{i1}\equiv1$ and $d=24$. Such Mincer-type wage equations and conditional quantile specifications are widely used in empirical wage-distribution analysis (\cite{mincer_schooling_1974,buchinsky_changes_1994,lemieux_mincer_2006}).

We summarize the empirical results in Figures~\ref{fig:realdata_coef} and~\ref{fig:realdata_joint}. For each selected regressor, Figure~\ref{fig:realdata_coef} plots the estimated coefficient and corresponding confidence interval at the nine prespecified quantile levels, ranging from $0.1$ to $0.9$. To control the family-wise error rate across these quantile levels, we apply a Bonferroni correction and construct each interval at the adjusted confidence level $1-\alpha/9$. Thus, when $\alpha=0.05$, the adjusted confidence level is approximately $99.44 \%$. The finite-sample Gaussian approximation developed in Theorem~\ref{coverage} remains applicable at the Bonferroni-adjusted critical values because its distributional bound is uniform over the threshold. More precisely, if the absolute coverage error of each adjusted pointwise interval is bounded by $r_n$, a union-bound argument gives $\PP\left\{\beta_j^*(\tau)\in\widehat{\mathrm{CI}}_j(\tau)
    \text{ for every }\tau\in\{0.1,0.2,\ldots,0.9\}
    \right\}
    \geq
    1-\alpha-9r_n.$
Hence, the proposed theory yields an explicit finite-sample bound on the simultaneous coverage error for this prespecified grid. Comparable explicit finite-sample coverage-error bounds are not currently available for the asymptotic random-scaling procedure in
Section~\ref{sec_simu}.

Figure~\ref{fig:realdata_joint} complements the coordinate-wise
analysis with two-dimensional Hotelling-$T^2$-type confidence regions
for the coefficients on standardized age and squared standardized age. At each quantile level, the region provides joint inference for these
two coefficients, and the nine regions are Bonferroni-adjusted across
the prespecified quantile levels. As shown in Figure~\ref{fig:realdata_joint}, the joint regions support a positive linear age coefficient and a negative quadratic age coefficient, providing evidence of a concave conditional age--wage relationship. This pattern is consistent with the concavity commonly associated with Mincer-type life-cycle earnings profiles. 
\begin{figure}[htbp!]
\centering
\begin{minipage}[t]{0.47\textwidth}
    \centering
    \includegraphics[width=\linewidth]{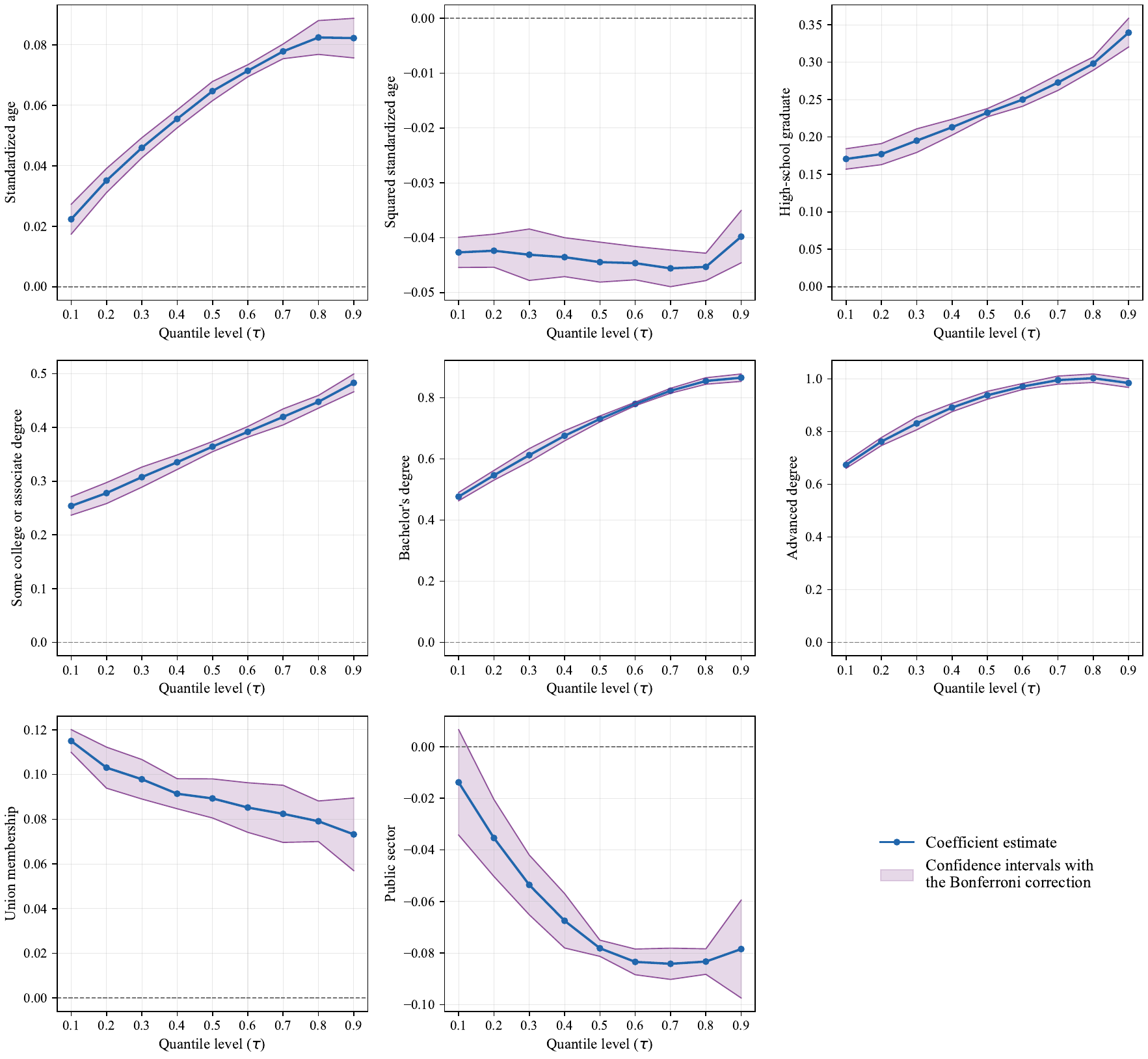}
    \caption{Per-coordinate simultaneous 95\% confidence intervals with the Bonferroni correction for selected CPS-MORG coefficients at $\tau \in \{0.1,0.2,\ldots,0.9\}$.}
    \label{fig:realdata_coef}
\end{minipage}\hfill
\begin{minipage}[t]{0.47\textwidth}
    \centering
    \includegraphics[width=\linewidth]{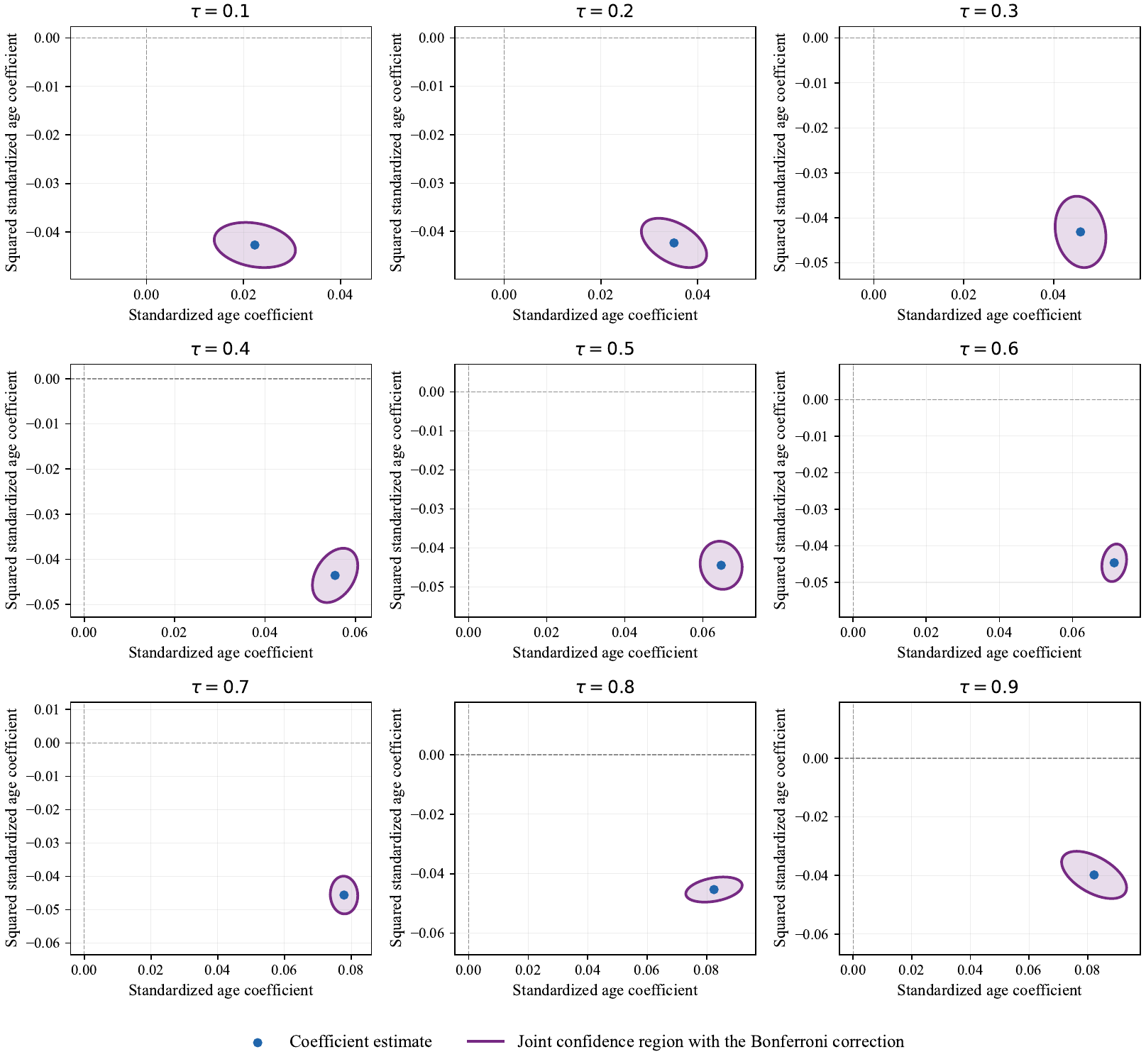}
    \caption{Hotelling’s T-squared type joint confidence regions of the standardized age and age squared coefficients, with the Bonferroni correction.}
    \label{fig:realdata_joint}
\end{minipage}
\end{figure}
The estimated coefficient paths also exhibit substantial heterogeneity
across conditional quantiles, illustrating the usefulness of quantile
methods for studying wage relationships that vary across the conditional distribution
(\cite{buchinsky_quantile_1995, angrist_quantile_2006}). The positive union coefficient is consistent with the empirical literature documenting a union wage premium (\cite{freeman_longitudinal_1984}). Its declining point-estimate profile indicates that the conditional union wage premium is concentrated toward the lower tail of the wage distribution and diminishes at higher conditional wage ranks, which is confirmed by broader evidence on union wage compression (\cite{chamberlain_quantile_1994,firpo_unconditional_2009}). The public-sector coefficient, which summarizes conditional public--private wage differences, exhibits a similar downward and compressing pattern. Related studies likewise find relatively modest public-sector wage gaps at the bottom and substantially negative ones at the top (\cite{poterba_distribution_1994, melly_public-private_2005, bargain_is_2026}). We note that these features demonstrated by the study are conditional quantile associations rather than causal effects. 

To assess the overall lack-of-fit of online quantile regression, we apply the simulation-based graphical diagnostic proposed by \textcite{wei_conditional_2006}. See Figure \ref{fig:realdata_qqplot} in Appendix~\ref{sec_supp_sim} for  details.

\printbibliography

\newpage

\begin{appendix}

\section{Regularity Lemma of Quantile Regression}

Throughout the paper, let $\mathcal{F}_k = \sigma (\{(y_{i},\bm{x}_i)\}_{i=1}^k)$ denote the natural filtration generated by the data. Define $m(\bbeta) =  \EE [\nabla g(\bbeta, \bx, y)]$ as the population averaged gradient. Notice that the loss function $g(\bbeta, \bx, y)$ is non-differentiable only at $y = \bx^{\top}\bbeta$ and its subgradient is bounded by the integrable $|\bx|$. Under Assumption \ref{as1}, the bounded $f_{\epsilon|\bx}(t)$ implies no point masses in the conditional distribution. Therefore, for any fixed $\bbeta$, $\PP(y = \bx^{\top}\bbeta) = 0$, and $m(\bbeta)$ is well-defined. 

Define $F_{\epsilon\mid\bx}$ as the conditional cumulative distribution function of $\epsilon$ given $\bx$. We introduce some basic regularity lemmas of quantile regression.

\begin{lemma}\label{Holder}
Under Assumptions \ref{as1} and \ref{as2}, the following statements hold for all $\bbeta, \bbeta' \in \mathbb{R}^d$:
\begin{enumerate}
    \item $m(\bbeta) = \EE \left\{ [F_{\epsilon\mid\bx}(\bx^{\top} (\bbeta-\bbeta^*))-\tau]\bx \right\}$, and consequently, $m(\bbeta^*)=0$.
    \item There exists a constant $\ti{C}_1 > 0$ such that:
    \begin{equation*}
        \|\nabla g(\bbeta, \bx, y) - \nabla g(\bbeta', \bx, y) \|_2 \leq \ti{C}_1 |\bbeta-\bbeta'|^{1/2},
    \end{equation*}
    and $$ | m(\bbeta)- m(\bbeta')| \leq \ti{C}_1|\bbeta-\bbeta'|^{1/2}.$$
\end{enumerate}
\end{lemma}

\begin{lemma}\label{Taylorexpansion}
Under Assumptions \ref{as1} and \ref{as2}, there exists a constant $\ti{C}_2$ such that for any $\bbeta$,
$$ | m(\bbeta) - A(\bbeta-\bbeta^*) | \leq \ti{C}_2 | \bbeta-\bbeta^* |^2. $$
\end{lemma}

\subsection{Truncated SSGD}
For $n\ge1$ and some constant $\vartheta > 1/p$, define the truncation level and event
\begin{align}
\G_n := \{\max_{1\le t\le n}|\bx_t| \le n^{\vartheta}\}.
\end{align}
Also define a truncation operator $\V_{n}(\cdot)$ as
\begin{align}
    \V_{n}(\bx) = \bx \min\{1,\frac{n^{\vartheta}}{|\bx|}\},
\end{align}
and the truncated analogue of $\bbeta_k$ from the recursion with $\ti{\bbeta}_{n,0} = \bbeta_0$:
\begin{align}\label{truncate}
    \ti{\bbeta}_{n,k} = \ti{\bbeta}_{n,k-1} + \gamma \V_n(\bm{x}_k)(\tau - \One_{\{y_k-\bm{x}_k^{\top}\ti{\bbeta}_{n,k-1}<0\}}), \quad  1\leq k \leq n, \quad \ti{\bbeta}_{n,0} = \bbeta_0,
\end{align}
where $\V_n(\bm{x}_k) = \bm{x}_k $ when $|\bm{x}_k| \leq n^{\vartheta}$, and $\V_n(\bm{x}_k) = n^{\vartheta}\bm{x}_k/|\bm x_k| $ when $|\bm{x}_k| > n^{\vartheta}$.
Denote the truncated version of the gradient and its expectation as
 $$\nabla \ti{g}_n(\bbeta,\bx,y)
 :=
 \V_n(\bx)
 \left(
 \One_{\{y-\bx^\top\bbeta<0\}}-\tau
 \right), \quad \ti{m}_n(\bbeta) = \EE \nabla \ti{g}_n(\bbeta,\bx,y).$$
It is clear that $\ti{m}_n(\bbeta^*) = 0$ for all $n \ge 1$. Denote $\mathcal{B}_r(\bm{v})$ as $\{ \bm{u}\in\RR^{d}: |  \bm{u}- \bm{v}| \leq r \}$ for some $r>0$ and $\bm{v} \in \RR^d$.
\begin{lemma}\label{verify_conditions}
Under Assumptions \ref{as1}-\ref{as2}, there exist positive constants $\ti{L}$, $\ti{\mu}$, $\ti{r}$, and $\ti{h}$ independent of $n$, $k$ and $\gamma$, such that
    \begin{itemize}
        \item $\forall \bbeta, \bbeta' \in \RR^d$, $|\ti{m}_n (\bbeta) - \ti{m}_n (\bbeta')| \leq \ti{L} |\bbeta-\bbeta' |$.
        \item $\forall \bbeta, \bbeta' \in \mathcal{B}_{\ti{r}}(\bbeta^*)$, $(\bbeta - \bbeta')^{\top} (\ti{m}_n (\bbeta) - \ti{m}_n (\bbeta')) \ge \ti{\mu} |\bbeta-\bbeta'  |^2.$
        \item $\forall \bbeta\notin \mathcal{B}_{\ti{r}}(\bbeta^*)$, $|\ti{m}_n(\bbeta)| \leq M_p$, and $ (\bbeta - \bbeta^*)^{\top} \ti{m}_n (\bbeta)  \ge \ti{h} |\bbeta-\bbeta^* |.$
    \end{itemize}
Here and in the sequel, we choose these constants as
\begin{equation}\label{eq:truncated-uniform-constants}
    \begin{split}
        \widetilde L
        &:=
        \M_fM_p^2,\\
        \widetilde r
        &:=
        \frac{1}{\varrho}
        \min\left\{
            r,\frac{\mu_\tau}{2L}
        \right\},\\
        \widetilde\mu
        &:=
        \frac{\mu_\tau}{4\varrho}
        \Lambda_{\min}(\EE[\bx\bx^\top]),\\
        \widetilde h
        &:=
        \widetilde\mu\,\widetilde r,
    \end{split}
\end{equation}
where $\varrho$ is defined in \eqref{lambda}.
\end{lemma}
Define 
\begin{equation}\label{truncatedA}
A_n
=
\EE\left[
    f_{\epsilon\mid\bx}(0)\V_n(\bx)\bx^\top
\right],
\end{equation}
and
\begin{equation}\label{truncatedS}
S_n =
\EE\left[
    \V_n(\bx)\V_n(\bx)^\top
    \left(
        \tau-\One_{\{y-\bx^\top\bbeta^*<0\}}
    \right)^2
\right]
\end{equation}
as the truncated analogues of $A$ and $S$. 
\begin{lemma}\label{uniform_matrix_bound}
Under Assumptions \ref{as1}-\ref{as2}, there exist constants $0<\lambda_* \leq \lambda^* <\infty$ and $0 < \underline{s}\leq \overline{s} < \infty$ such that for all $n\ge1$, 
\begin{equation}\label{eq:An-uniform-all-n}
\lambda_*{\bf I}_d
\preceq
A_n
\preceq
\lambda^*{\bf I}_d,
\qquad  \underline{s}{\bf I}_d
\preceq
S_n
\preceq
\overline{s}{\bf I}_d.
\end{equation} 
Here the eigenvalue bounds can be chosen as
\begin{align*}\label{eq:explicit-An-constants}
&\lambda_*
=
\frac{\mu_\tau}{2\varrho}
\Lambda_{\min}(\EE[\bx\bx^\top]),
\quad
\lambda^* =
\M_f\big|\EE[\bx\bx^\top]\big|, \\
&\underline{s} =
\frac{\tau(1-\tau)}{2\varrho^2}
\Lambda_{\min}(\EE[\bx\bx^\top]),
\quad
\overline{s} =
\tau(1-\tau)
\big|\EE[\bx\bx^\top]\big|,
\end{align*}
where $\varrho$ is defined in \eqref{lambda}.
\end{lemma}
\begin{remark}
Since $\varrho \ge 1$, it is also clear that $\lambda_*{\bf I}_d
\preceq A \preceq
\lambda^*{\bf I}_d$ and $\underline{s}{\bf I}_d \preceq S
\preceq
\overline{s}{\bf I}_d$. We will continue to use these choices of $\lambda_*$, $\lambda^*$, $\underline{s}$, and $\overline{s}$ in this paper.
\end{remark}
Finally, in this section, we introduce an analogous regularity lemma for the truncated version of the gradient and its expectation.
\begin{lemma}\label{Holder_truncated}
Under Assumptions \ref{as1} and \ref{as2}, the following statements
hold for all \(\bbeta,\bbeta'\in\RR^d\) and all \(n\geq1\):
\begin{itemize}
    \item $\ti{m}_n(\bbeta)
    =
    \EE\left\{
        \left[
            F_{\epsilon\mid\bx}\left(\bx^{\top}(\bbeta-\bbeta^*)\right)-\tau
        \right]\V_n(\bx)
    \right\}$,
    and consequently \(\ti{m}_n(\bbeta^*)=0\).
    \item There exists a constant \(\ti{C}_3>0\), independent of \(n\), such that
    \[
    \left\|
        \nabla \ti{g}_n(\bbeta,\bx,y)
        -
        \nabla \ti{g}_n(\bbeta',\bx,y)
    \right\|_2
    \leq
    \ti{C}_3|\bbeta-\bbeta'|^{1/2},
    \]
    and $|\ti{m}_n(\bbeta)-\ti{m}_n(\bbeta')| \leq \ti{C}_3|\bbeta-\bbeta'|^{1/2}$.
    \item There exists a constant \(\ti{C}_4>0\), independent of \(n\), such that 
\[
\left|
    \ti{m}_n(\bbeta)-A_n(\bbeta-\bbeta^*)
\right|
\leq
\ti{C}_4|\bbeta-\bbeta^*|^2.
\]
\end{itemize}
\end{lemma}

\section{Proof of Consistency}
We prove the following lemma that establishes the moment convergence for the truncated sequence.
\begin{lemma}\label{moment_tilde}
     Under Assumptions \ref{as1}-\ref{as2}, consider the triangular sequence $\{ \ti{\bbeta}_{n,k}\}_{0 \leq k \leq n}$ defined in \eqref{truncate}. Fix any constant
     \begin{equation}\label{eq:lambda-choice}
         \lambda \in 
        \Bigg( 0,\frac14
        \min\left\{
            \frac{\ti{h}}{\ti{r}},
            2\ti{\mu}
        \right\}\Bigg].
    \end{equation} 
    For any $\gamma$ satisfying $0<\gamma \leq 1/(2\lambda)$, $\gamma n^{\vartheta} \le c_{\diamond}$ with some constants $\vartheta \in (1/p,1)$ and $c_{\diamond}>0$, we have
     $$\| \ti{\bbeta}_{n,k}-\bbeta^* \|_4^2 \lesssim (1-\lambda\gamma)^k+\gamma.$$
\end{lemma}

\begin{proof}
Define $\gamma_0 = \min\{
            1,\, 2\lambda/\C_1,\,
            \lambda/\C_2,\,
            1/(3\lambda)\}$, where
        $$ \C_1:=
        \frac{e^{2+c_{\diamond}/\ti{r}}  M_2^2}{2\ti{r}^2}, \quad
        \C_2:=
        \frac{6e^{1+c_{\diamond}/\ti{r}}}{\ti{r}^2}
        \left(
            2\max(1,\ti{r}^{\,2})M_2^2+eM_4^4
        \right).$$
Recall that $M_2 = \| \bx\|_2$ and $M_4 = \| \bx\|_4$. Under Assumption~\ref{as2}, we have
    \begin{align}
        |\nabla \ti{g}_n(\bbeta,\bx,y)|
        &\leq n^{\vartheta},                                                    \label{eq:truncated-score-bound}\\
        \sup_{\bbeta\in\RR^d}
        \EE|\nabla \ti{g}_n(\bbeta,\bx,y)|^q
        &\leq  M_q^q,
        \qquad q\in\{2,4\}.                                          \label{eq:truncated-score-moments}
    \end{align}
Define two Lyapunov functions by
    \begin{align}
        G(\bbeta)
        &:=
        \begin{cases}
            \displaystyle
            \frac{e|\bbeta-\bbeta^*|^2}{2\ti{r}^{\,2}},
            &|\bbeta-\bbeta^*|\leq\ti{r},\\[1.2ex]
            \displaystyle
            \exp\left(
                \frac{|\bbeta-\bbeta^*|}{\ti{r}}
            \right)-\frac e2,
            &|\bbeta-\bbeta^*|>\ti{r},
        \end{cases}                                                   \label{eq:H0-definition}\\
        H(\bbeta)
        &:=
        |\bbeta-\bbeta^*|^2G(\bbeta).
        \label{eq:H2-definition}
    \end{align}
    We first consider the case when $\gamma \leq \gamma_0$. By Lemma~1(1) of \textcite{zhang_piecewise_2025},
    \begin{equation}\label{eq:H2-fourth-lower}
        H(\bbeta)
        \geq
        \frac{e}{2\ti{r}^{\,2}}
        |\bbeta-\bbeta^*|^4.
    \end{equation}
        Moreover, Lemma~1(2) of
    \textcite{zhang_piecewise_2025} shows that $G$ and $H$ are twice continuously differentiable. We also have
    \begin{align}
        |\nabla^2G(\bbeta)|
        &\leq
        \frac{e}{\ti{r}^{\,2}}
        \exp\left(
            \frac{|\bbeta-\bbeta^*|}{\ti{r}}
        \right),                                                     
        \label{eq:H0-Hessian}\\
        |\nabla^2H(\bbeta)|
        &\leq
        \frac{6e}{\ti{r}^{\,2}}|\bbeta-\bbeta^*|^2
        \exp\left(
            \frac{|\bbeta-\bbeta^*|}{\ti{r}}
        \right).
        \label{eq:H2-Hessian}
    \end{align}
    Indeed, when $|\bbeta-\bbeta^*|\leq\ti{r}$, it is clear that $|\nabla^2G(\bbeta)|
        =e/\ti{r}^{\,2}$, whereas for $|\bbeta-\bbeta^*|>\ti{r}$, 
        \begin{align*}
        \nabla^2G(\bbeta)&=
        \frac{\exp\left(|\bbeta-\bbeta^*|/\ti r\right)
        (\bbeta-\bbeta^*) (\bbeta-\bbeta^*)^\top}{\ti r^{\,2} |\bbeta-\bbeta^* |^2} \\
        &+
        \frac{1}{\ti r |\bbeta-\bbeta^*|}
        \exp\left(\frac{|\bbeta-\bbeta^*|}{\ti r}\right)
        \left({\bf I}_d-(\bbeta-\bbeta^*) \frac{(\bbeta-\bbeta^*)^\top}{|\bbeta-\bbeta^* |^2}\right).
    \end{align*}
        The radial and tangential eigenvalues of
    $\nabla^2G(\bbeta)$ are respectively
    \[
        \frac{1}{\ti{r}^{\,2}}
        \exp\left(\frac{|\bbeta-\bbeta^*|}{\ti{r}}\right)
        \quad\text{and}\quad
        \frac{1}{\ti{r}|\bbeta-\bbeta^*|}
        \exp\left(\frac{|\bbeta-\bbeta^*|}{\ti{r}}\right).
    \]
    Since $|\bbeta-\bbeta^*|>\ti{r}$, both eigenvalues are bounded by
    $\ti{r}^{-2}\exp(|\bbeta-\bbeta^*|/\ti{r})$, proving
    \eqref{eq:H0-Hessian}.

    Moreover,
    \[
        \nabla^2H(\bbeta)
        =
        2G(\bbeta){\bf I}_d
        +
        2(\bbeta-\bbeta^*)\nabla G(\bbeta)^\top
        +
        2\nabla G(\bbeta)(\bbeta-\bbeta^*)^\top
        +
        |\bbeta-\bbeta^*|^2\nabla^2G(\bbeta).
    \]
    When $|\bbeta-\bbeta^*|\leq\ti{r}$, direct differentiation gives $|\nabla^2H(\bbeta)|
        \leq 6e|\bbeta-\bbeta^*|^2/\ti{r}^{\,2}$. When $|\bbeta-\bbeta^*|>\ti{r}$, using $|\nabla^2G(\bbeta)| \leq \ti{r}^{-2}\exp(|\bbeta-\bbeta^*|/\ti{r})$ we just derived, and
    \[
        G(\bbeta)\leq\exp\left(\frac{|\bbeta-\bbeta^*|}{\ti{r}}\right),
        \qquad
        |\nabla G(\bbeta)|
        =
        \frac{1}{\ti{r}}
        \exp\left(\frac{|\bbeta-\bbeta^*|}{\ti{r}}\right),
    \]
    we obtain
    \begin{align*}
        |\nabla^2H(\bbeta)|
        &\leq
        \left(
            2+\frac{4|\bbeta-\bbeta^*|}{\ti{r}}+\frac{|\bbeta-\bbeta^*|^2}{\ti{r}^{\,2}}
        \right)
        \exp\left(\frac{|\bbeta-\bbeta^*|}{\ti{r}}\right)\\
        &\leq
        \frac{7}{\ti{r}^{\,2}}|\bbeta-\bbeta^*|^2
        \exp\left(\frac{|\bbeta-\bbeta^*|}{\ti{r}}\right)\\
        &\leq
        \frac{6e}{\ti{r}^{\,2}}|\bbeta-\bbeta^*|^2
        \exp\left(\frac{|\bbeta-\bbeta^*|}{\ti{r}}\right).
    \end{align*}
    This proves \eqref{eq:H2-Hessian}. Since $G$ and $H$ coincide respectively with $V_{1,0}$ and $V_{1,2}$ in Lemma~1 of \textcite{zhang_piecewise_2025}, with $\Delta=\ti r$, applying Lemma~1(4) of \textcite{zhang_piecewise_2025} and Lemma \ref{verify_conditions}, we obtain
    \begin{align}
\nabla G(\bbeta)^{\top}\ti{m}_n(\bbeta)
        &\geq
        \min\left\{
            \frac{\ti{h}}{\ti{r}},
            2\ti{\mu}
        \right\}G(\bbeta)
        \geq4\lambda G(\bbeta),                                    \label{eq:H0-drift}\\
            \nabla H(\bbeta)^{\top}\ti{m}_n(\bbeta)
        &\geq
        \min\left\{
            \frac{\ti{h}}{\ti{r}},
            4\ti{\mu}
        \right\}H(\bbeta)
        \geq4\lambda H(\bbeta).
        \label{eq:H2-drift}
    \end{align}
    From the definitions of $G$ and $H$, we have
    \begin{align}
        \exp\left(
            \frac{|\bbeta-\bbeta^*|}{\ti{r}}
        \right)
        &\leq
        e(G(\bbeta)+1),
        \label{eq:exponential-H0}\\
        |\bbeta-\bbeta^*|^2
        \exp\left(
            \frac{|\bbeta-\bbeta^*|}{\ti{r}}
        \right)
        &\leq
        2\max(1,\ti{r}^{\,2})(H(\bbeta)+G(\bbeta)).
        \label{eq:quadratic-exponential-H2}
    \end{align}
    To verify \eqref{eq:exponential-H0}, if
    $|\bbeta-\bbeta^*|\leq\ti{r}$, then
    \[
        \exp\left(
            \frac{|\bbeta-\bbeta^*|}{\ti{r}}
        \right)
        \leq e
        \leq e(G(\bbeta)+1).
    \]
    If $|\bbeta-\bbeta^*|>\ti{r}$, then
    \[
        \exp\left(
            \frac{|\bbeta-\bbeta^*|}{\ti{r}}
        \right)
        =
        G(\bbeta)+\frac e2
        \leq
        e(G(\bbeta)+1).
    \]
    Similarly for \eqref{eq:quadratic-exponential-H2}, if $|\bbeta-\bbeta^*|\leq\ti{r}$, then
    \begin{align*}
        |\bbeta-\bbeta^*|^2
        \exp\left(
            \frac{|\bbeta-\bbeta^*|}{\ti{r}}
        \right)
        &\leq
        e|\bbeta-\bbeta^*|^2\\
        &=
        2\ti{r}^{\,2}G(\bbeta)\\
        &\leq
        2\max(1,\ti{r}^{\,2})(H(\bbeta)+G(\bbeta)).
    \end{align*}
    If $|\bbeta-\bbeta^*|>\ti{r}$, then
    \[
        G(\bbeta)
        =
        \exp\left(
            \frac{|\bbeta-\bbeta^*|}{\ti{r}}
        \right)-\frac e2
        \geq
        \frac12
        \exp\left(
            \frac{|\bbeta-\bbeta^*|}{\ti{r}}
        \right),
    \]
    and hence
    \[
        |\bbeta-\bbeta^*|^2
        \exp\left(
            \frac{|\bbeta-\bbeta^*|}{\ti{r}}
        \right)
        \leq
        2H(\bbeta)
        \leq
        2\max(1,\ti{r}^{\,2})(H(\bbeta)+G(\bbeta)).
    \]
    For every $s\in[0,1]$, \eqref{eq:truncated-score-bound} and the range of $\gamma$ give
    \begin{align}
        \exp\left(
            \frac{
                |\bbeta-\bbeta^*
                -s\gamma\nabla \ti{g}_n(\bbeta,\bx,y)|
            }{\ti{r}}
        \right)
        &\leq
        \exp\left(
            \frac{|\bbeta-\bbeta^*|}{\ti{r}}
        \right)
        \exp\left(
            \frac{\gamma n^{\vartheta}}{\ti{r}}
        \right) \nonumber\\
        &\leq
        e^{c_{\diamond}/\ti{r}}
        \exp\left(
            \frac{|\bbeta-\bbeta^*|}{\ti{r}}
        \right).
        \label{eq:uniform-exponential-shift}
    \end{align}

        Now we derive a one-step contraction for $G$. Conditional on
    $\mathcal F_{k-1}$, the integral form of Taylor's theorem gives
    \begin{align*}
        &\EE\left[
            G(\ti{\bbeta}_{n,k})
            \mid\mathcal F_{k-1}
        \right]\\
        &\quad=
        G(\ti{\bbeta}_{n,k-1})
        -
        \gamma
        \nabla G(\ti{\bbeta}_{n,k-1})^\top
        \ti{m}_n(\ti{\bbeta}_{n,k-1})\\
        &\qquad+
        \gamma^2
        \EE\left[
            \int_0^1(1-s)
            \nabla \ti{g}_n(\ti{\bbeta}_{n,k-1},\bx_k,y_k)^\top
            \nabla^2G\left(
                \ti{\bbeta}_{n,k-1}
                -
                s\gamma
                \nabla \ti{g}_n(\ti{\bbeta}_{n,k-1},\bx_k,y_k)
            \right)\right.\\
        &\hspace{7cm}\left.
            {}\times
            \nabla \ti{g}_n(\ti{\bbeta}_{n,k-1},\bx_k,y_k)
            \,ds
            \,\middle|\,\mathcal F_{k-1}
        \right].
    \end{align*}
    Since $\ti{\bbeta}_{n,k-1}$ is
    $\mathcal F_{k-1}$-measurable and $(\bx_k,y_k)$ is independent of
    $\mathcal F_{k-1}$, \eqref{eq:truncated-score-moments} gives
    \[
        \EE\left[
            |\nabla \ti{g}_n(
                \ti{\bbeta}_{n,k-1},\bx_k,y_k)|^q
            \,\middle|\,\mathcal F_{k-1}
        \right]
        \leq M_q^q,
        \qquad q\in\{2,4\}.
    \]
    Therefore, by \eqref{eq:H0-Hessian},
    \eqref{eq:uniform-exponential-shift}, and
    \eqref{eq:exponential-H0}, the absolute value of the Taylor
    remainder is bounded by
    \begin{align*}
        &\gamma^2
        \EE\left[
            \int_0^1(1-s)
            \left|
                \nabla \ti{g}_n(
                    \ti{\bbeta}_{n,k-1},\bx_k,y_k)
            \right|^2
            \left|
                \nabla^2G\left(
                    \ti{\bbeta}_{n,k-1}
                    -
                    s\gamma
                    \nabla \ti{g}_n(
                        \ti{\bbeta}_{n,k-1},\bx_k,y_k)
                \right)
            \right|
            ds
            \,\middle|\,\mathcal F_{k-1}
        \right]\\
        &\quad\leq
        \frac{e}{\ti{r}^{\,2}}e^{c_{\diamond}/\ti{r}}\gamma^2
        \exp\left(
            \frac{
                |\ti{\bbeta}_{n,k-1}-\bbeta^*|
            }{\ti{r}}
        \right)
        M_2^2
        \int_0^1(1-s)\,ds\\
        &\quad=
        \frac12 \frac{e}{\ti{r}^{\,2}} e^{c_{\diamond}/\ti{r}}M_2^2\gamma^2
        \exp\left(
            \frac{
                |\ti{\bbeta}_{n,k-1}-\bbeta^*|
            }{\ti{r}}
        \right)\\
        &\quad\leq
        \C_1\gamma^2
        (G(\ti{\bbeta}_{n,k-1})+1).
    \end{align*}
    Combining this with \eqref{eq:H0-drift}, we obtain
    \begin{align*}
        \EE\left[
            G(\ti{\bbeta}_{n,k})
            \mid\mathcal F_{k-1}
        \right]
        &\leq
        (1-4\lambda\gamma+\C_1\gamma^2)
        G(\ti{\bbeta}_{n,k-1})
        +
        \C_1\gamma^2\\
        &\leq
        (1-2\lambda\gamma)
        G(\ti{\bbeta}_{n,k-1})
        +
        \C_1\gamma^2,
    \end{align*}
    where the last inequality follows from
    $\gamma\leq2\lambda/\C_1$. Thus
    \begin{equation}\label{eq:H0-one-step}
        \EE\left[
            G(\ti{\bbeta}_{n,k})
            \mid\mathcal F_{k-1}
        \right]
        \leq
        (1-2\lambda\gamma)
        G(\ti{\bbeta}_{n,k-1})
        +
        \C_1\gamma^2.
    \end{equation}
    Iterating \eqref{eq:H0-one-step} gives
    \begin{align}
        \EE G(\ti{\bbeta}_{n,k})
        &\leq
        (1-2\lambda\gamma)^kG(\bbeta_0)
        +
        \C_1\gamma^2
        \sum_{j=0}^{k-1}(1-2\lambda\gamma)^j \nonumber\\
        &\leq
        (1-2\lambda\gamma)^kG(\bbeta_0)
        +
        \frac{\C_1}{2\lambda}\gamma.
        \label{eq:H0-iteration-final}
    \end{align}
    We next consider $H$. The same Taylor's theorem gives
    \begin{align*}
        &\EE\left[
            H(\ti{\bbeta}_{n,k})
            \mid\mathcal F_{k-1}
        \right]\\
        &\quad=
        H(\ti{\bbeta}_{n,k-1})
        -
        \gamma
        \nabla H(\ti{\bbeta}_{n,k-1})^\top
        \ti{m}_n(\ti{\bbeta}_{n,k-1})\\
        &\qquad+
        \gamma^2
        \EE\left[
            \int_0^1(1-s)
            \nabla \ti{g}_n(\ti{\bbeta}_{n,k-1},\bx_k,y_k)^\top
            \nabla^2H\left(
                \ti{\bbeta}_{n,k-1}
                -
                s\gamma
                \nabla \ti{g}_n(\ti{\bbeta}_{n,k-1},\bx_k,y_k)
            \right)\right.\\
        &\hspace{7cm}\left.
            {}\times
            \nabla \ti{g}_n(\ti{\bbeta}_{n,k-1},\bx_k,y_k)
            \,ds
            \,\middle|\,\mathcal F_{k-1}
        \right].
    \end{align*}
    For every $s\in[0,1]$,
    \begin{align*}
        &\left|
            \ti{\bbeta}_{n,k-1}-\bbeta^*
            -
            s\gamma
            \nabla \ti{g}_n(
                \ti{\bbeta}_{n,k-1},\bx_k,y_k)
        \right|^2\\
        &\qquad\leq
        2|\ti{\bbeta}_{n,k-1}-\bbeta^*|^2
        +
        2\gamma^2
        |\nabla \ti{g}_n(
            \ti{\bbeta}_{n,k-1},\bx_k,y_k)|^2.
    \end{align*}
    Consequently, by \eqref{eq:H2-Hessian} and
    \eqref{eq:uniform-exponential-shift}, the absolute value of the
    Taylor remainder for $H$ is bounded by
    \begin{align*}
        &\frac{6e}{\ti{r}^{\,2}}e^{c_{\diamond}/\ti{r}}\gamma^2
        \exp\left(
            \frac{
                |\ti{\bbeta}_{n,k-1}-\bbeta^*|
            }{\ti{r}}
        \right)
        \int_0^1(1-s)
        \EE\left[
            \left(
                2|\ti{\bbeta}_{n,k-1}-\bbeta^*|^2
                +
                2\gamma^2
                |\nabla \ti{g}_n(
                    \ti{\bbeta}_{n,k-1},\bx_k,y_k)|^2
            \right)\right.\\
        &\hspace{8cm}\left.
            {}\times
            |\nabla \ti{g}_n(
                \ti{\bbeta}_{n,k-1},\bx_k,y_k)|^2
            \,\middle|\,\mathcal F_{k-1}
        \right]ds\\
        &\leq
        \frac{6e}{\ti{r}^{\,2}}e^{c_{\diamond}/\ti{r}}\gamma^2
        \left[
            M_2^2
            |\ti{\bbeta}_{n,k-1}-\bbeta^*|^2
            \exp\left(
                \frac{
                    |\ti{\bbeta}_{n,k-1}-\bbeta^*|
                }{\ti{r}}
            \right) +
            \gamma^2M_4^4
            \exp\left(
                \frac{
                    |\ti{\bbeta}_{n,k-1}-\bbeta^*|
                }{\ti{r}}
            \right)
        \right] \\
        &\leq \frac{6e}{\ti{r}^{\,2}}e^{c_{\diamond}/\ti{r}}\cdot 2\max(1,\ti{r}^{\,2})M_2^2\gamma^2
        (H(\ti{\bbeta}_{n,k-1})
          +G(\ti{\bbeta}_{n,k-1}))+
        \frac{6e}{\ti{r}^{\,2}}e^{c_{\diamond}/\ti{r}}eM_4^4\gamma^4
        (G(\ti{\bbeta}_{n,k-1})+1).
    \end{align*}
    Here the last inequality is from \eqref{eq:exponential-H0} and
    \eqref{eq:quadratic-exponential-H2}.
    Since $\gamma\leq\gamma_0\leq1$, this upper bound is at most
    \[
        \C_2\gamma^2
        (H(\ti{\bbeta}_{n,k-1})
          +G(\ti{\bbeta}_{n,k-1}))
        +
        \C_2\gamma^4.
    \]
    Combining this with \eqref{eq:H2-drift} gives
    \begin{align}\label{eq:H2-one-step}
        \EE\left[
            H(\ti{\bbeta}_{n,k})
            \mid\mathcal F_{k-1}
        \right]
        &\leq
        (1-4\lambda\gamma+\C_2\gamma^2)
        H(\ti{\bbeta}_{n,k-1})+
        \C_2\gamma^2G(\ti{\bbeta}_{n,k-1})
        +
        \C_2\gamma^4 \nonumber \\
        &\leq
        (1-3\lambda\gamma)
        H(\ti{\bbeta}_{n,k-1})+
        \C_2\gamma^2G(\ti{\bbeta}_{n,k-1})
        +
        \C_2\gamma^4,
    \end{align}
    where the last inequality follows from
    $\gamma\leq\lambda/\C_2$. Since $\gamma\leq1/(3\lambda)$, both
    $1-3\lambda\gamma$ and $1-2\lambda\gamma$ belong to $[0,1)$. Iterating \eqref{eq:H2-one-step} and using
    \eqref{eq:H0-iteration-final} gives
    \begin{align}
        \EE H(\ti{\bbeta}_{n,k})
        \leq{}&
        (1-3\lambda\gamma)^kH(\bbeta_0)+
        \C_2\gamma^2G(\bbeta_0)
        \sum_{j=0}^{k-1}
            (1-3\lambda\gamma)^{k-1-j}
            (1-2\lambda\gamma)^j \nonumber \\
        &+
        \frac{\C_2\C_1}{2\lambda}\gamma^3
        \sum_{j=0}^{k-1}
            (1-3\lambda\gamma)^{k-1-j}+
        \C_2\gamma^4
        \sum_{j=0}^{k-1}
            (1-3\lambda\gamma)^{k-1-j}.
        \label{eq:H2-before-sums}
    \end{align}
    Moreover, we have
    \begin{align}
        \sum_{j=0}^{k-1}
            (1-3\lambda\gamma)^{k-1-j}
            (1-2\lambda\gamma)^j
        =
        \frac{
            (1-2\lambda\gamma)^k
            -(1-3\lambda\gamma)^k
        }{\lambda\gamma} \leq
        \frac{(1-2\lambda\gamma)^k}
             {\lambda\gamma},
        \label{eq:two-geometric-sum}
    \end{align}
    and
    \begin{equation}\label{eq:one-geometric-sum}
        \sum_{j=0}^{k-1}
            (1-3\lambda\gamma)^{k-1-j}
        \leq
        \frac{1}{3\lambda\gamma}.
    \end{equation}
    It follows that
    \begin{align}
        \EE H(\ti{\bbeta}_{n,k})
        \leq 
        (1-3\lambda\gamma)^kH(\bbeta_0)
        +
        \frac{\C_2}{\lambda}\gamma
        (1-2\lambda\gamma)^kG(\bbeta_0)
        +
        \frac{\C_2\C_1}{6\lambda^2}\gamma^2
        +
        \frac{\C_2}{3\lambda}\gamma^3.
        \label{eq:H2-refined-bound}
    \end{align}
    Because $\gamma\leq1$ and $0\leq1-3\lambda\gamma
        \leq1-2\lambda\gamma<1$, we conclude that
    \begin{equation}\label{eq:H2-uniform-bound}
        \EE H(\ti{\bbeta}_{n,k})
        \lesssim 
        (1-2\lambda\gamma)^k+\gamma^2,
        \qquad 0\leq k\leq n.
    \end{equation}
    Finally, \eqref{eq:H2-fourth-lower} and
    \eqref{eq:H2-uniform-bound} imply
    \[
        \EE|\ti{\bbeta}_{n,k}-\bbeta^*|^4
        \lesssim 
        (1-2\lambda\gamma)^k+\gamma^2.
    \]
    Since $\gamma\leq1/(3\lambda)$, $\sqrt{1-2\lambda\gamma}
        \leq
        1-\lambda\gamma$ and therefore,
    \begin{align*}
        \|\ti{\bbeta}_{n,k}-\bbeta^*\|_4^2
        &=
        \left(
            \EE|\ti{\bbeta}_{n,k}-\bbeta^*|^4
        \right)^{1/2}\\
        &\lesssim 
        (1-2\lambda\gamma)^{k/2}
        +
        \gamma\\
        &\leq
        (1-\lambda\gamma)^k+\gamma,
        \qquad 0\leq k\leq n,
    \end{align*}
    and the implicit constant is independent of $n$, $k$, $\gamma$, and $\vartheta$. The case when $\gamma_0<\gamma\leq 1/(2\lambda)$ is straightforward, since now $n^{1/p}\le n^\vartheta \le c_{\diamond}/\gamma
<c_{\diamond}/\gamma_0$, i.e., $n \leq (c_{\diamond}/\gamma_0)^p$. Moreover, for the truncated sequence, $|\ti{\bbeta}_{n,k} -\ti{\bbeta}_{n,k-1}| \leq \gamma n^{\vartheta} \leq c_{\diamond}$. By the triangle inequality,
$$\|\ti{\bbeta}_{n,k}-\bbeta^*\|_4
\le |\bbeta_0 -\bbeta^*|+kc_{\diamond} \le |\bbeta_0 -\bbeta^*|+c^{p+1}_{\diamond}/\gamma_0^p =: \ti{C} .$$
Since $\gamma >\gamma_0$, we have $\|\ti{\bbeta}_{n,k}-\bbeta^*\|_4^2 \leq \ti{C}^2 \gamma/\gamma_0 \leq \ti{C}^2[(1-\lambda\gamma)^k +\gamma]/\gamma_0$.
\end{proof}

\begin{proof}[Proof of Theorem \ref{thm_moment_sgd}]
    On the event $\G_n$, $\bbeta_n = \ti{\bbeta}_{n,n}$. The probability of $\G^c_n$ is at most $\PP(\G^c_n) \leq n\PP(|\bx_i| \ge n^{\vartheta}) \leq n \EE |\bx_i|^p/n^{\vartheta p} \leq M_p^p /n^{\vartheta p-1}$. By Lemma \ref{moment_tilde}, when $0<\gamma \leq 1/(2\lambda)$ and $\gamma n^{\vartheta}\leq c_{\diamond}$, there exists some constant $C_1$ independent of $n$ and $\gamma$ such that
    $$\| \ti{\bbeta}_{n,n}-\bbeta^* \|_4^2 \leq C_1((1-\lambda\gamma)^n+\gamma).$$
    By Markov's inequality, for every $\delta \in (0,1)$
    \begin{align} \label{proofOp} 
    \PP( | \bbeta_{n}-\bbeta^* |^2 > \frac{C_1}{\sqrt{\delta}} ((1-\lambda\gamma)^n+\gamma)) &\leq \PP( | \bbeta_n-\bbeta^* |^2 >\frac{C_1}{\sqrt{\delta}}  ((1-\lambda\gamma)^n+\gamma), \G_n) \nonumber \\
    &\quad +\PP( | \bbeta_n-\bbeta^* |^2 >\frac{C_1}{\sqrt{\delta}} ((1-\lambda\gamma)^n+\gamma), \G^c_n)\nonumber\\
    &\leq \PP( | \ti{ \bbeta}_{n,n}-\bbeta^* |^2 >\frac{C_1}{\sqrt{\delta}}  ((1-\lambda\gamma)^n+\gamma), \G_n) +M_p^pn^{1-p\vartheta}\nonumber\\
    &\leq \frac{\delta \| \ti{\bbeta}_{n,n}-\bbeta^* \|_4^4}{ C_1^2((1-\lambda\gamma)^n+\gamma)^2 }+ M_p^pn^{1-p\vartheta}\nonumber\\
    &\leq \delta+ M_p^pn^{1-p\vartheta} .
    \end{align}
    Since $\vartheta >1/p$, \eqref{proofOp} also implies that
     $$| \bbeta_{n}-\bbeta^* |^2  = \OO_{\PP} \Big((1-\lambda\gamma)^n+\gamma \Big).$$
\end{proof}

\section{Proof of CLT of the Last Iterate}\label{subsec_proof_last_iterate_clt}
\subsection{Properties of covariance of the last iterate}
\begin{proposition}\label{prop_Gamma}
Define
\begin{align}\label{tildeGamma}
\ti \Gamma_n(\gamma) = \sum_{j=0}^{n-1} \gamma ({\bf I}_d-\gamma A_n)^j S_n ({\bf I}_d-\gamma A_n)^j.
\end{align}
Under Assumptions \ref{as1}-\ref{as2}, the following statements hold:
\begin{itemize}
    \item $|A_n-A|+|S_n-S| \lesssim n^{\vartheta(2-p)}$ for all $n\ge1$.
    \item $|\ti{\Gamma}_n(\gamma)-\Gamma_n(\gamma)| \lesssim \min\{n\gamma,1 \}n^{\vartheta(2-p)}$ for all $n\ge1$ and $0 < \gamma < 1/\lambda^*$.
\end{itemize}
\end{proposition}

Here and in the sequel, let $Z_{n,k} = -\nabla \ti g_n(\bbeta^*, \bx_k, y_k)$. Denote the estimation error by $\Delta_{n,k} = \ti\bbeta_{n,k}-\bbeta^*$. We have 
\begin{equation}\label{error}
\Delta_{n,k}=\Delta_{n,k-1}-\gamma A_n\Delta_{n,k-1}+\gamma R_{n,k}+\gamma D_{n,k}+\gamma Z_{n,k}, 
\end{equation}
where we define the Taylor expansion remainder and the martingale difference noise term, respectively, by
\begin{align*}
    R_{n,k} & = A_n\Delta_{n,k-1}-\ti m_n(\ti\bbeta_{n,k-1}), \\
    D_{n,k} & =\ti m_n(\ti\bbeta_{n,k-1})-\nabla \ti g_n(\ti\bbeta_{n,k-1}, \bx_k,y_k)+ \nabla \ti g_n(\bbeta^*, \bx_k,y_k).
\end{align*}
We also denote $\|\cdot\| = \|\cdot \|_2$ when no confusion is caused.
\subsection{Proof of Theorem \ref{clt_last}}
\begin{proof}
Recursively updating equation \eqref{error}, we obtain
$$\Delta_{n,n} = ({\bf I}_d - \gamma A_n)^n \Delta_0 + \sum_{k=1}^n \gamma ({\bf I}_d-\gamma A_n)^{n-k} (Z_{n,k} + R_{n,k} + D_{n,k}).$$
It is clear that $\{D_{n,k}\}$ is a martingale difference sequence w.r.t. the filtration $\{ \mathcal{F}_k\}_{k=1}^n$. Since $\lambda < \lambda_*$ defined in Lemma \ref{uniform_matrix_bound}, we have $|{\bf I}_d-\gamma A_n| \leq 1-\lambda \gamma$. Therefore,
\begin{align}\label{linear_approx1}
\|\sum_{k=1}^n \gamma ({\bf I}_d-\gamma A_n)^{n-k}D_{n,k}\|^2
&\leq \gamma^2 \sum_{k=1}^n |({\bf I}_d-\gamma A_n)|^{2n-2k} \|D_{n,k}\|^2 \nonumber  \\
&\lesssim \gamma^2  \sum_{k=1}^n |({\bf I}_d-\gamma A_n)|^{2n-2k} \|\ti{\bbeta}_{n,k-1}-\bbeta^*\| \nonumber \\
&\lesssim \gamma^2 \sum_{k=1}^n (1-\lambda \gamma)^{2n-2k} (1-\lambda \gamma)^{(k-1)/2}  +  \gamma^{2.5}\sum_{k=1}^n (1-\lambda \gamma)^{2n-2k}\nonumber  \\
&\lesssim \gamma^2 \frac{(1-\gamma \lambda)^{(n-1)/2}-(1-\gamma \lambda)^{2n-1/2}}{\gamma \lambda} + \gamma^{2.5} \frac{1-(1-\gamma \lambda)^{2n}}{2\gamma \lambda-\gamma^2 \lambda^2} \nonumber \\
&\leq \gamma \frac{(1-\lambda \gamma)^{(n-1)/2}}{\lambda}+\frac{\gamma^{1.5}}{2\lambda-\gamma \lambda^2},
\end{align}
and
\begin{align}\label{linear_approx2}
\|\sum_{k=1}^n \gamma ({\bf I}_d-\gamma A_n)^{n-k}R_{n,k}\| & \lesssim \gamma \sum_{k=1}^n |{\bf I}_d-\gamma A_n|^{n-k}   \|\ti{\bbeta}_{n,k-1}-\bbeta^*\|_4^2 \nonumber \\
   & \lesssim n\gamma(1-\gamma \lambda)^{n-1}+\gamma^{2}\frac{1-(1-\gamma \lambda)^n}{\gamma \lambda} \nonumber \\
   & \lesssim \gamma+n\gamma (1-\gamma \lambda)^{n-1}. 
\end{align}
Here we use $\|D_{n,k}\|^2\lesssim \|\Delta_{n,k-1}\|_1 \lesssim (1-\lambda\gamma)^{(k-1)/2}+\sqrt{\gamma} $ and $ |R_{n,k}| \lesssim   | \Delta_{n,k-1} |^2$, which follows from Lemma \ref{Holder_truncated} and \ref{moment_tilde}. We also have $ \|(  {\bf I}_d-\gamma A_n)^n \Delta_0 \| \leq (1-\gamma\lambda)^n \| \Delta_0\|$. Combining this with \eqref{linear_approx1}-\eqref{linear_approx2}, we have
\begin{equation}\label{linear_approx}
 \|\frac{\Delta_{n,n}}{\sqrt{\gamma}} -\sum_{k=1}^n \sqrt{\gamma }({\bf I}_d-\gamma A_n)^{n-k}Z_{n,k} \| \lesssim \gamma^{1/4}+(1-\lambda\gamma)^{(n-1)/4}+n\sqrt{\gamma}(1-\gamma\lambda)^{n-1}+\frac{(1-\gamma\lambda)^n}{\sqrt{\gamma}}.    
\end{equation}
By Theorem 2.1 in \textcite{mies_sequential_2023}, on a potentially enriched probability space, there exists some centered Gaussian vector $\ti \zeta_n\sim \N(\mathbf 0_d,\ti\Gamma_n(\gamma))$ such that 
\begin{equation}\label{MS_last}
  \| \sum_{k=1}^n \sqrt{\gamma }({\bf I}_d-\gamma A_n)^{n-k}Z_{n,k}  - \ti \zeta_n\|^2 \lesssim (\sqrt{\gamma}n^{1/p}\sqrt{\log n})^2,
\end{equation} 
where $\ti \Gamma_n(\gamma)$ is defined in \eqref{tildeGamma}. Combining \eqref{linear_approx} and \eqref{MS_last}, we have
\begin{align}\label{ineq1}
\|\frac{\ti{\bbeta}_{n,n} - \bbeta^*}{\sqrt{\gamma}} - \ti \zeta_n  \| & \lesssim \gamma^{1/4}+(1-\lambda\gamma)^{(n-1)/4}+n\sqrt{\gamma}(1-\gamma\lambda)^{n-1} +\sqrt{\gamma}n^{1/p}\sqrt{\log n} \nonumber\\
&  = \E_n^{\mathrm{SSGD}}.
\end{align}    
Notice that $(1-\gamma\lambda)^n/ \sqrt{\gamma} \lesssim n\sqrt{\gamma}(1-\gamma\lambda)^{n-1}$ when $n\gamma \gtrsim 1$, so the former is absorbed into the latter in  \eqref{ineq1}.

Since $((\bbeta_n-\bbeta^*)/\sqrt{\gamma} - \ti \zeta_n ) \One_{\G_n} = ((\ti{\bbeta}_{n,n}-\bbeta^*)/\sqrt{\gamma} - \ti \zeta_n ) \One_{\G_n}$, we have
\begin{equation}\label{ineq2}
 \| ((\bbeta_n-\bbeta^*)/\sqrt{\gamma} - \ti \zeta_n ) \One_{\G_n} \|= \| ((\ti{\bbeta}_{n,n}-\bbeta^*)/\sqrt{\gamma} - \ti \zeta_n ) \One_{\G_n} \| \leq  \| (\ti\bbeta_{n,n}-\bbeta^*)/\sqrt{\gamma} - \ti \zeta_n  \|.    
\end{equation}
Let $\zeta_n =\Gamma_n(\gamma)^{1/2}\ti\Gamma_n(\gamma)^{-1/2} \ti\zeta_n $ which has the distribution $\N(\mathbf 0_d,\Gamma_n(\gamma))$. Notice that $\ti{\Gamma}_n(
\gamma)$ is a sum of positive semi-definite matrices, including a strictly positive definite term $\gamma S_n$, so $\Lambda_{\min}(\ti{\Gamma}_n(\gamma))\ge\gamma \Lambda_{\min}(S_n) >0$, and \(\ti{\Gamma}_n(\gamma)^{-1/2}\) is well defined. It suffices to bound the difference between \(\zeta_n\) and \(\ti{\zeta}_n\). Define $\xi_n=
\ti{\Gamma}_n(\gamma)^{-1/2}\ti{\zeta}_n$. Then \(\xi_n\sim \N(\mathbf 0_d,{\bf I}_d)\), and
\[
\ti{\zeta}_n=\ti{\Gamma}_n(\gamma)^{1/2}\xi_n,
\qquad
\zeta_n=\Gamma_n(\gamma)^{1/2}\xi_n.
\]
Therefore,
\begin{equation}\label{eq:zeta_difference_representation}
\zeta_n-\ti{\zeta}_n
=
\left(
\Gamma_n(\gamma)^{1/2}
-
\ti{\Gamma}_n(\gamma)^{1/2}
\right)\xi_n.
\end{equation}
We first claim that
\begin{equation}\label{eq:Gamma_scaled_eigenvalues}
\Lambda_{\min}\big(\Gamma_n(\gamma)\big)\gtrsim
\min\{n\gamma,1\}, \quad \Lambda_{\min}\big(\ti{\Gamma}_n(\gamma)\big)
\gtrsim
\min\{n\gamma,1\}.
\end{equation}
Indeed,
\begin{align*}
\ti{\Gamma}_n(\gamma)
&\succeq
\underline{s}\gamma
\sum_{j=0}^{n-1}
({\bf I}_d-\gamma A_n)^{2j}\\
&\succeq
\underline{s}\gamma
\sum_{j=0}^{n-1}
(1-\lambda^*\gamma)^{2j}{\bf I}_d\\
&=
\underline{s}
\frac{
1-(1-\lambda^*\gamma)^{2n}
}{
2\lambda^*-\gamma(\lambda^*)^2
}
{\bf I}_d \gtrsim \frac{1-\exp\{-2\lambda^*n \gamma \}}{2\lambda^* - \gamma(\lambda^*)^2}{\bf I}_d
\gtrsim
\min\{n\gamma,1\}{\bf I}_d.
\end{align*}
The same argument applies to \(\Gamma_n(\gamma)\). For positive-definite matrices, the square-root perturbation inequality
gives
\begin{align*}
\left|
\Gamma_n(\gamma)^{1/2}
-
\ti{\Gamma}_n(\gamma)^{1/2}
\right|
&\leq
\frac{
\left|
\Gamma_n(\gamma)-\ti{\Gamma}_n(\gamma)
\right|
}{
\sqrt{\Lambda_{\min}(\Gamma_n(\gamma))}
+
\sqrt{\Lambda_{\min}(\ti{\Gamma}_n(\gamma))}
}.
\end{align*}
Combining \eqref{eq:Gamma_scaled_eigenvalues} and
Proposition \ref{prop_Gamma}, we obtain
\begin{equation}\label{eq:Gamma_square_root_difference}
\left|
\Gamma_n(\gamma)^{1/2}
-
\ti{\Gamma}_n(\gamma)^{1/2}
\right|
\lesssim
\min\{n\gamma,1\}^{1/2}
n^{\vartheta(2-p)}.
\end{equation}
Since \(\xi_n\sim\N(\mathbf 0_d,{\bf I}_d)\),
\(\|\xi_n\|_2=\sqrt d\). It follows from
\eqref{eq:zeta_difference_representation} and
\eqref{eq:Gamma_square_root_difference} that, for fixed \(d\), 
\begin{align*}
\|\zeta_n-\ti{\zeta}_n\|
\leq
\left|
\Gamma_n(\gamma)^{1/2}
-
\ti{\Gamma}_n(\gamma)^{1/2}
\right|
\|\xi_n\|_2 \lesssim
\sqrt d\,
\min\{n\gamma,1\}^{1/2}
n^{\vartheta(2-p)} \lesssim n^{\vartheta(2-p)}.
\end{align*}
This bound holds uniformly over all \(n\geq1\) and all
\(0<\gamma<1/\lambda^*\), and implies $\|(\zeta_n-\ti{\zeta}_n ) \One_{\G_n}\| \lesssim n^{\vartheta(2-p)}$. Since $\vartheta >1/p$ and $p \ge 4$, $n^{\vartheta(2-p)} \leq  n^{2/p-1} \leq n^{-1/2}$. When $n\lambda \gamma >1/2$, $\gamma^{1/4} \gtrsim n^{-1/2}$; when $n\lambda \gamma \leq 1/2$, $(1-\lambda \gamma)^{(n-1)/4} \gtrsim n^{-1/2}$. So this error term is absorbed into the other terms. Combining this with \eqref{ineq1}-\eqref{ineq2}, we obtain
\begin{align*}
    \| (\frac{\bbeta_n-\bbeta^*}{\sqrt{\gamma}} - \zeta_n ) \One_{\G_n}\| \lesssim \E_n^{\mathrm{SSGD}}.
\end{align*}
Finally, the same argument as \eqref{proofOp} in the proof of Theorem \ref{thm_moment_sgd} implies that there exists a constant $C_2$ independent of $(n,\gamma,\delta)$ such that for every $\delta \in (0,1)$,
$$ \PP\left[\left |\frac{\bbeta_n-\bbeta^*}{\sqrt{\gamma}} - \zeta_n\right | > \frac{C_2\E_n^{\mathrm{SSGD}}}{\sqrt{\delta}} \right]  \leq \delta + M_p^pn^{1-p\vartheta}.$$

\end{proof}

\subsection{Proof of Corollary~\ref{cor_clt_last_iterate}}

\begin{proof}
Recall that $\lambda^*=M_f|\EE[\bx\bx^\top]|$. Since $0<\lambda\leq\Lambda_{\min}(A) \leq\Lambda_{\max}(A)\leq\lambda^*$, the condition $0<\gamma<1/\lambda^*$ gives
$|{\bf I}_d-\gamma A|\leq 1-\lambda\gamma<1$ and
$2\lambda-\gamma\lambda^2\geq\lambda$.
Hence, by equations (25)--(26) of \textcite{wei_gaussian_2025}, $|\Gamma_n(\gamma)-\Gamma| \lesssim \gamma+(1-\lambda\gamma)^{2n},$
which implies the convergence of $\Gamma_n(\gamma)$ in Corollary~\ref{cor_clt_last_iterate}. 

For the CLT, choose an auxiliary truncation exponent
$\vartheta\in(1/p,2/p)$. Then $n^{2/p} \gamma \log n \rightarrow 0$ implies $\gamma n^\vartheta\to0$. Since $n^{1-p\vartheta}\to0$ and
$n\gamma\to\infty$, Theorem~\ref{clt_last} applies
for all sufficiently large $n$, and the condition on $n$ and $\gamma$ ensures $\E_n^{\mathrm{SSGD}}\rightarrow0$. Taking $\delta=\mathcal E_n^{\mathrm{SSGD}}\in(0,1)$
for sufficiently large $n$ in Theorem~\ref{clt_last}, we obtain
\[
\PP\left(
\left|
\frac{\bbeta_n-\bbeta^*}{\sqrt{\gamma}}-\zeta_n
\right|
>C_2\sqrt{\mathcal E_n^{\mathrm{SSGD}}}
\right)
\leq
\mathcal E_n^{\mathrm{SSGD}}+M_p^pn^{1-p\vartheta}
\longrightarrow0.
\]
Thus the coupling error converges to zero in probability.
Since $\Gamma_n(\gamma)\to\Gamma$ implies
$\zeta_n\stackrel{D}{\to}\N(0,\Gamma)$, Slutsky's theorem
yields the asserted CLT. The argument holds for every fixed initialization $\bbeta_0$.
\end{proof}

\section{Proof of Proposition \ref{bias}}

\begin{proof}
For the setting in Proposition \ref{bias}, we have $A=\sqrt2$ and $\beta^*=\sqrt2/2$. Recall the recursion
\begin{equation}\label{eq:sgd-recursion}
    \beta_{n+1}
    =\beta_n-\gamma
      \left(\One_{\{y_{n+1} < \beta_n\}}-1/2\right).
\end{equation}
If $\beta_n<0$, then $y_{n+1}\geq0>\beta_n$ almost surely, and hence $\beta_{n+1}=\beta_n+\gamma/2$. Consequently, the iterates move upward deterministically until they first
become nonnegative, and the first nonnegative iterate belongs to
$[0,\gamma/2)$. So the interval $\left[-\gamma/2,1+\gamma/2\right]$ is invariant: below $0$ the next update is $+\gamma/2$, above $1$ the
next update is $-\gamma/2$, and an iterate in $[0,1]$ can move by only
$\gamma/2$ in either direction. Since $\beta_0 \in \left[-\gamma/2,1+\gamma/2\right]$, we have $\beta_n \in \left[-\gamma/2,1+\gamma/2\right]$ for all $n\ge 0$.

We first obtain a lower bound on the second moment. Let $F(\cdot)$ be the cumulative distribution function of $y_1$, and $\Delta_n = \beta_n-\beta^*$. Since $\left(\One_{\{y_{n+1} < \beta_n\}}-1/2\right)^2
    =1/4$,
the recursion \eqref{eq:sgd-recursion} gives
\begin{align}
    \EE\Delta_{n+1}^2
    &=\EE\Delta_n^2
      -2\gamma\EE\left[
        \Delta_n\left(F(\beta_n)-\frac12\right)
      \right]
      +\frac{\gamma^2}{4}.
      \label{eq:second-moment-recursion}
\end{align}
We claim that
\begin{equation}\label{eq:secant-bound}
    0\leq
    (s-\beta^*)\left(F(s)-\frac12\right)
    \leq
    (1+\beta^*)(s-\beta^*)^2,
    \qquad s\in\mathbb R.
\end{equation}
Indeed, for $0\leq s\leq1$ and $s \neq \beta^*$, the relevant secant slope is
\[
    \frac{F(s)-F(\beta^*)}{s-\beta^*}=s+\beta^*
    \leq1+\beta^*;
\]
for $s<0$ it is $1/[2(\beta^*-s)]\leq1/\sqrt2$, for
$s>1$ it is $1/[2(s-\beta^*)]\leq1+\beta^*$, and when $s = \beta^*$ \eqref{eq:secant-bound} trivially holds. Thereby the claim holds. It follows from
\eqref{eq:second-moment-recursion} that
\begin{equation}\label{eq:second-moment-lower-recursion}
    \EE\Delta_{n+1}^2
    \geq
    \left(1-2(1+\beta^*)\gamma\right)\EE\Delta_n^2
    +\frac{\gamma^2}{4}.
\end{equation}
Since $\beta_0=0$,
\[
    \Delta_0^2 =\frac12
    \geq\frac{\gamma}{8(1+\beta^*)},
\]
and $1-2(1+\beta^*)\gamma>0$ for $\gamma\leq0.1$. Therefore, induction using \eqref{eq:second-moment-lower-recursion} yields
\begin{equation}\label{eq:second-moment-lower-bound}
    \EE\Delta_n^2
    \geq\frac{\gamma}{8(1+\beta^*)}
    \qquad\text{for every }n\geq0.
\end{equation}

Now, as in the original decomposition, define
\[
    R_k
    =A\Delta_{k-1}
      -\left(F(\beta_{k-1})-\frac12\right).
\]
Then
\begin{equation}\label{eq:R-pieces}
    R_k=
    \begin{cases}
        A\Delta_{k-1}+\dfrac12
        =A\beta_{k-1}-\dfrac12,
        & \beta_{k-1}<0,\\[2mm]
        -\Delta_{k-1}^2,
        & 0\leq\beta_{k-1}\leq1,\\[2mm]
        A\Delta_{k-1}-\dfrac12
        =A\beta_{k-1}-\dfrac32,
        & \beta_{k-1}>1.
    \end{cases}
\end{equation}
Now we claim that, whenever
$-\gamma/2\leq\beta_{k-1}\leq1+\gamma/2$,
\begin{equation}\label{eq:R-pointwise}
    R_k\leq-\frac{\Delta_{k-1}^2}{c_R},
\end{equation}
where
\begin{equation}\label{eq:cR-constant}
    c_R
    =\frac{(21-10\sqrt2)^2}{20(30-21\sqrt2)}
    =\frac{530+287\sqrt2}{120}.
\end{equation}
Notice that $c_R>2$. Below we prove the claim. For $-\gamma/2\leq\beta_{k-1}<0$, using
$\gamma/2\leq1/20$ gives
\[
    2R_k+\Delta_{k-1}^2
    =\beta_{k-1}^2+A\beta_{k-1}-\frac12
    \leq0.
\]
Thus $R_k\leq-\Delta_{k-1}^2/2$, which implies
\eqref{eq:R-pointwise}. For $0\leq\beta_{k-1}\leq1$,
\eqref{eq:R-pointwise} follows directly from
$R_k=-\Delta_{k-1}^2$. Finally, for
$1<\beta_{k-1}\leq1+\gamma/2\leq21/20$, the function
\[
    s\longmapsto
    \frac{3/2-As}{(s-\beta^*)^2}
\]
is clearly decreasing. Its minimum on this interval is therefore attained at $s=21/20$, and
\[
    \frac{3/2-21A/20}{(21/20-\beta^*)^2}
    =\frac1{c_R}.
\]
This proves \eqref{eq:R-pointwise}. Notice that the value of $c_R$ in
\eqref{eq:cR-constant} is the smallest possible constant for \eqref{eq:R-pointwise} on the interval
$[-1/20,21/20]$. Combining \eqref{eq:R-pointwise} and
\eqref{eq:second-moment-lower-bound} gives the uniform estimate: for every $k\geq1$,
\begin{equation}\label{eq:R-expectation}
    \EE R_k
    \leq-\frac{\gamma}{8c_R(1+\beta^*)}.
\end{equation}

Taking conditional expectations in \eqref{eq:sgd-recursion} gives the
finite-sample recursion
\begin{equation}\label{eq:mean-one-step}
    \EE[\Delta_k\mid\F_{k-1}]
    =(1-A\gamma)\Delta_{k-1}+\gamma R_k.
\end{equation}
Iterating \eqref{eq:mean-one-step} and using
\eqref{eq:R-expectation}, we obtain
\begin{align}
    \EE\Delta_n
    &=(1-A\gamma)^n\Delta_0
      +\gamma\sum_{k=1}^n
       (1-A\gamma)^{n-k}\EE R_k\notag\\
    &\leq(1-A\gamma)^n\Delta_0
      -\frac{\gamma^2}{8c_R(1+\beta^*)}
       \sum_{k=1}^n(1-A\gamma)^{n-k}\notag\\
    &=(1-A\gamma)^n\Delta_0
      -\frac{\gamma}{8c_R(1+A)}
       \left(1-(1-A\gamma)^n\right),
       \label{eq:mean-iterate-bound}
\end{align}
where we used $A(1+\beta^*)=1+A$.

We now consider the standardized bias of RP averaging. Summing
\eqref{eq:mean-iterate-bound} from $i=1$ to $n$ and using
\[
    \sum_{i=1}^n(1-A\gamma)^i
    =\frac{(1-A\gamma)\left(1-(1-A\gamma)^n\right)}{A\gamma},
\]
we get
\begin{align}
    \frac1{\sqrt n}\EE\sum_{i=1}^n\Delta_i
    &\leq
    \frac{1}{\sqrt n}
    \left(\Delta_0+\frac{\gamma}{8c_R(1+A)}\right)
    \frac{(1-A\gamma)\left(1-(1-A\gamma)^n\right)}{A\gamma}
    -\frac{\gamma\sqrt n}{8c_R(1+A)}.
    \label{eq:averaged-bias-upper}
\end{align}
Since $\Delta_0 =-\beta^* =-1/A$ and $\gamma\leq1/10$,
\eqref{eq:averaged-bias-upper} implies
\begin{equation}\label{eq:key-lower}
    -\frac1{\sqrt n}\EE\sum_{i=1}^n\Delta_i
    \geq
    C_0\,
    \frac{1-(1-A\gamma)^n}{\gamma\sqrt n}
    +\frac{\gamma\sqrt n}{8c_R(1+A)},
\end{equation}
where 
\begin{equation}\label{eq:C0}
    C_0
    =\frac12
      \left(1-\frac{A}{80c_R(1+A)}\right)
      \left(1-\frac{A}{10}\right)\approx 0.42889 > 0.4.
\end{equation}

It remains only to make the last bound uniform in both $n$ and $\gamma$.
First suppose that $\gamma\sqrt n<1$. Since
$1-u\leq e^{-u}$,
\[
    (1-A\gamma)^n
    \leq e^{-A\gamma n}
    \leq e^{-A\gamma\sqrt n}.
\]
The function $x\mapsto(1-e^{-Ax})/x$ is decreasing on $(0,\infty)$.
Therefore,
\[
    \frac{1-(1-A\gamma)^n}{\gamma\sqrt n}
    \geq1-e^{-A},
\]
and \eqref{eq:key-lower} gives
\begin{equation}\label{eq:small-case}
    -\frac1{\sqrt n}\EE\sum_{i=1}^n\Delta_i
    \geq C_0(1-e^{-A})>\frac14.
\end{equation}
Now suppose that $\gamma\sqrt n\geq1$. Then
\[
    \gamma n=\frac{(\gamma\sqrt n)^2}{\gamma}
    \geq\frac1\gamma\geq10,
\]
and hence
\[
    1-(1-A\gamma)^n
    \geq1-e^{-A\gamma n}
    \geq1-e^{-10A}.
\]
Applying the AM-GM inequality in
\eqref{eq:key-lower}, we obtain
\begin{align}
    -\frac1{\sqrt n}\EE\sum_{i=1}^n\Delta_i
    &\geq
      \frac{C_0(1-e^{-10A})}{\gamma\sqrt n}
      +\frac{\gamma\sqrt n}{8c_R(1+A)}\notag\\
    &\geq
      2\sqrt{
        \frac{C_0(1-e^{-10A})}{8c_R(1+A)}
      }
      =\B\approx 0.10672>0.1.
      \label{eq:large-case}
\end{align}
Since $\B<1/4$, combining the two cases, for every $n\geq1$ and
$\gamma\in(0,0.1]$,
\[
    \sqrt n\,\EE[\beta_n^{\mathrm{RP}}-\beta^*] = \frac1{\sqrt n}\EE\sum_{i=1}^n\Delta_i\leq-\B,
\]
i.e.,
\[
    \left|\sqrt n\,\EE[\beta_n^{\mathrm{RP}}-\beta^*]\right|
    \geq\B.
\]
Taking the infimum over $n \ge 1$ and
$\gamma\in(0,0.1]$ proves the proposition.
\end{proof}

\section{Proof of Gaussian approximation for suffix averaging}\label{sec_proof_suffix_clt}
\subsection{Gaussian approximation for general suffix averaging}
\begin{lemma}\label{GA_gSGD} Consider a general recursion  
$$\bbeta_{n,k} = \bbeta_{n,k-1}-\gamma \varphi_n(\bbeta_{n,k-1},\bX_k), \quad k = 1,\ldots, n, \quad \bbeta_{n,0} = \bbeta_0 $$
with i.i.d. $\{\bX_k\}$ and fixed learning rate $\gamma>0$. Define $m_n(\bbeta) = \EE\varphi_n(\bbeta,\bX_k) $. Suppose the following conditions hold for some root $\bbeta^*$ with $m_n(\bbeta^*) = 0$, constant $p>2$, $0<\lambda \leq \lambda_* $, $0<h\leq 1$, $\ti{M}_p>0$, and a positive definite matrix $A_n$ with $\lambda_*{\bf I}_d \preceq A_n \preceq \lambda^*{\bf I}_d$:
\begin{itemize}
    \item $\| \varphi_n(\bbeta^*,\bX_k)\|_p \leq \ti{M}_p< \infty.$
    \item $\| \bbeta_{n,k} -\bbeta^* \|_4^2 \lesssim (1-\lambda \gamma)^{k}+ \gamma$.
    \item $ | m_n(\bbeta) - A_n(\bbeta-\bbeta^*)| \lesssim | \bbeta-\bbeta^*|^2. $
    \item $\|\varphi_n(\bbeta, \bX) - \varphi_n(\bbeta', \bX) \|_2 \lesssim |\bbeta-\bbeta'|^{h}.$
\end{itemize}
Define the suffix averaging $\bbeta_n^{\mathrm{SF}}  =(\sum_{k=T(n)}^n \bbeta_{n,k}) / (n-T(n)+1)$ for $1 \leq T(n)\leq n$. Let $\gamma <1/\lambda^*$. Then on a potentially enriched probability space, there exists some centered Gaussian vector $\zeta_n\sim \N(\mathbf 0_d,V_n)$ such that 
\begin{align*}
    \|\sqrt{n}(\bbeta_n^{\mathrm{SF}}-\bbeta^*)-\zeta_n \|_2 \lesssim \sqrt{n}\gamma &+ \frac{\sqrt{n}(1-\lambda\gamma)^{T(n)-1}}{(n-T(n)+1)\gamma} + \frac{\sqrt{n}(T(n)-1)(1-\lambda\gamma)^{T(n)-1}}{n-T(n)+1} \\&+ \frac{\sqrt{n}(1-\lambda \gamma)^{h(T(n)-1)/2}}{(n-T(n)+1)\sqrt{\gamma}}
    +\frac{\sqrt{n}\gamma^{(h-1)/2}}{n-T(n)+1}\\
    &+ \frac{\sqrt{n}\gamma^{h/2}}{\sqrt{n-T(n)+1}}+\frac{n^{1/p}\sqrt{\log n (n-T(n)+1+\gamma^{-1})}}{n-T(n)+1}. 
\end{align*}
where 
\begin{align*}
V_n = \frac{n}{(n-T(n)+1)^2} \Bigg[ \sum_{k=1}^{T(n)-1} &A_n^{-1} \Big( ({\bf I}_d-\gamma A_n)^{T(n)-k} - ({\bf I}_d-\gamma A_n)^{n-k+1} \Big) \\
& \times S_n \Big( ({\bf I}_d-\gamma A_n)^{T(n)-k} - ({\bf I}_d-\gamma A_n)^{n-k+1} \Big) A_n^{-1} \\
+ \sum_{k=T(n)}^{n} A_n^{-1} &\Big( {\bf I}_d - ({\bf I}_d-\gamma A_n)^{n-k+1} \Big) S_n \Big( {\bf I}_d - ({\bf I}_d-\gamma A_n)^{n-k+1} \Big) A_n^{-1} \Bigg],
\end{align*}
and $S_n = \EE [\varphi_n(\bbeta^*,\bX_n) \varphi_n(\bbeta^*,\bX_n)^{\top}]$ is a sequence of positive definite matrices.
\end{lemma}
When $n-T(n) \asymp n $ and $\gamma \gtrsim 1/n$, the convergence rate simplifies to 
\begin{align*}
    \|\sqrt{n}(\bbeta_n^{\mathrm{SF}}-\bbeta^*)-\zeta_n \|_2 \lesssim \sqrt{n}\gamma &+ \frac{(1-\lambda\gamma)^{T(n)-1}}{\sqrt{n}\gamma} + \frac{(T(n)-1)(1-\lambda\gamma)^{T(n)-1}}{\sqrt{n}} \\&+ \frac{(1-\lambda \gamma)^{h(T(n)-1)/2}}{\sqrt{n\gamma}}+ \gamma^{h/2}+n^{1/p-1/2}\sqrt{\log n}. 
\end{align*}
Here the constants implicit in $\lesssim$ and $\gtrsim$ are uniform in terms of $k$, $n$, and $\gamma$ satisfying conditions in the lemma.
\begin{remark}
Lemma \ref{GA_gSGD} applies to a broad class of stochastic approximation settings. For example, in the settings of \textcite{wei_gaussian_2025}, the fixed-learning-rate SGD for strongly convex and stochastic Lipschitz smooth problems satisfies all conditions of Lemma \ref{GA_gSGD} with $h=1$; for online quantile regression with sub-exponential covariates, Lemma \ref{Holder}, Lemma \ref{Taylorexpansion}, and Corollary 4 in \textcite{zhang_piecewise_2025} also ensure all conditions of Lemma \ref{GA_gSGD} with $h =1/2$. This result is of independent interest for uncertainty quantification in stochastic approximation.
\end{remark}
\begin{proof}
For this general SSGD recursion, define $\Delta_{n,k} = \bbeta_{n,k} - \bbeta^* $ and  $\Delta_{n,0} = \Delta_0$. We have 
$$\Delta_{n,j} = ({\bf I}_d - \gamma A_n)^j \Delta_0 + \sum_{k=1}^j \gamma ({\bf I}_d-\gamma A_n)^{j-k} (Z_{n,k} + R_{n,k} + D_{n,k}).$$
where 
\begin{align*}
    Z_{n,k} & = -\varphi_n(\bbeta^*, \bX_k),\\
    R_{n,k} & = A_n\Delta_{n,k-1}-m_n(\bbeta_{n,k-1}), \\
    D_{n,k} & =m_n(\bbeta_{n,k-1})-\varphi_n(\bbeta_{n,k-1}, \bX_k)+ \varphi_n(\bbeta^*, \bX_k).
\end{align*}
Due to the conditions, we have 
\begin{align*}
\EE Z_{n,k} & = 0, \quad \EE [D_{n,k} | \F_{k-1}] = 0, \\
\|R_{n,k}\|_2 &\lesssim \| \bbeta_{n,k-1}-\bbeta^* \|_4^2 \lesssim(1-\lambda \gamma)^{k-1} + \gamma, \\
\EE\left[|D_{n,k}|^2\mid\mathcal F_{k-1}\right]
&\lesssim |\bbeta_{n,k-1}-\bbeta^*|^{2h}, 
\end{align*}
Hence
\begin{align*}
\|D_{n,k}\|^2
&\lesssim
\| \bbeta_{n,k-1}-\bbeta^* \|_{2h}^{2h}  \\
&\leq
\| \bbeta_{n,k-1}-\bbeta^* \|_{2}^{2h}  \\
&\lesssim
\left((1-\lambda\gamma)^{k-1}+\gamma\right)^h \\
&\lesssim
(1-\lambda\gamma)^{h(k-1)}+\gamma^h,
\end{align*}
and $\{D_{n,k}\}_{1\leq k \leq n}$ forms a martingale difference sequence in terms of $k$. Denote the error of $n$-th suffix averaged SSGD as $\bar{\Delta}^{\mathrm{SF}}_n = \bbeta_n^{\mathrm{SF}} - \bbeta^*$, then
\begin{align*}
\bar{\Delta}^{\mathrm{SF}}_n &= \frac{1}{n-T(n)+1}\sum_{j=T(n)}^{n}({\bf I}_d-\gamma A_n)^j \Delta_0 \\
& \qquad + \frac{1}{n-T(n)+1}\sum_{j=T(n)}^{n}\gamma \sum_{k=1}^j ({\bf I}_d-\gamma A_n)^{j-k}(Z_{n,k}+ R_{n,k}+D_{n,k})\\
&=I_0 + \frac{\gamma}{n-T(n)+1} \Big( \sum_{k=1}^{T(n)-1} \sum_{j=T(n)}^{n} ({\bf I}_d-\gamma A_n)^{j-k}(Z_{n,k}+ R_{n,k}+D_{n,k}) \\
&+\sum_{k=T(n)}^{n} \sum_{j=k}^{n} ({\bf I}_d-\gamma A_n)^{j-k}(Z_{n,k}+ R_{n,k}+D_{n,k}) \Big)\\
&=I_0  + \frac{1}{n-T(n)+1} \sum_{k=1}^{T(n)-1} (A_n^{-1}({\bf I}_d-\gamma A_n)^{T(n)-k}-A_n^{-1}({\bf I}_d-\gamma A_n)^{n-k+1} )(Z_{n,k}+ R_{n,k}+D_{n,k})\\
&+\frac{1}{n-T(n)+1} \sum_{k=T(n)}^{n} (A_n^{-1}-A_n^{-1}({\bf I}_d-\gamma A_n)^{n-k+1} )(Z_{n,k}+ R_{n,k}+D_{n,k})\\
&=  \frac{1}{n-T(n)+1} \sum_{k=T(n)}^{n} A_n^{-1}Z_{n,k}+I_0+I_1+I_2+I_3+I_4+I_5+I_6,
\end{align*}
where 
$$I_0=\frac{1}{\gamma(n-T(n)+1)}(A_n^{-1}({\bf I}_d-\gamma A_n)^{T(n)}-A_n^{-1}({\bf I}_d-\gamma A_n)^{n+1} )\Delta_0,$$
$$ I_1 = \frac{1}{n-T(n)+1} \sum_{k=1}^{T(n)-1} (A_n^{-1}({\bf I}_d-\gamma A_n)^{T(n)-k}-A_n^{-1}({\bf I}_d-\gamma A_n)^{n-k+1} )Z_{n,k},$$
$$ I_2 = \frac{1}{n-T(n)+1} \sum_{k=1}^{T(n)-1} (A_n^{-1}({\bf I}_d-\gamma A_n)^{T(n)-k}-A_n^{-1}({\bf I}_d-\gamma A_n)^{n-k+1} )R_{n,k},$$
$$ I_3 = \frac{1}{n-T(n)+1} \sum_{k=1}^{T(n)-1} (A_n^{-1}({\bf I}_d-\gamma A_n)^{T(n)-k}-A_n^{-1}({\bf I}_d-\gamma A_n)^{n-k+1} )D_{n,k},$$
$$I_4 = - \frac{1}{n-T(n)+1} \sum_{k=T(n)}^nA_n^{-1}({\bf I}_d-\gamma A_n)^{n-k+1}Z_{n,k},$$
$$I_5 = \frac{1}{n-T(n)+1} \sum_{k=T(n)}^n(A_n^{-1}- A_n^{-1}({\bf I}_d-\gamma A_n)^{n-k+1})R_{n,k},$$
$$I_6 = \frac{1}{n-T(n)+1} \sum_{k=T(n)}^n(A_n^{-1}- A_n^{-1}({\bf I}_d-\gamma A_n)^{n-k+1})D_{n,k},$$
We have 
$$\|I_0 \|_2 \lesssim \frac{(1-\lambda \gamma)^{T(n)}}{(n-T(n)+1)\gamma}.$$
\begin{align*}
\|I_2\|_2 &\lesssim \frac{1}{n-T(n)+1}\sum_{k=1}^{T(n)-1} (1-\lambda\gamma)^{T(n)-k}\|R_{n,k}\|_2 \\
&\leq \frac{1}{n-T(n)+1}\sum_{k=1}^{T(n)-1} (1-\lambda\gamma)^{T(n)-k} ( \gamma + (1-\lambda\gamma)^{k-1} ) \\
&\lesssim  \frac{T(n)-1}{n-T(n)+1}(1-\lambda\gamma)^{T(n)-1}+\frac{1}{n-T(n)+1}.
\end{align*}
\begin{align*}
\| I_3 \|_2^2 &\lesssim \frac{1}{(n-T(n)+1)^2}\sum_{k=1}^{T(n)-1}(1-\lambda \gamma)^{2T(n)-2k} \| D_{n,k}\|_2^2 \\
&\lesssim \frac{1}{(n-T(n)+1)^2}\sum_{k=1}^{T(n)-1}(1-\lambda \gamma)^{2T(n)-2k} [  (1-\lambda\gamma)^{h(k-1)}+\gamma^h  ] \\
&\lesssim  \frac{(1-\lambda \gamma)^{h(T(n)-1)+2-h}}{\gamma(n-T(n)+1)^2}+ \frac{\gamma^{h-1}}{(n-T(n)+1)^2},
\end{align*}
where the first term is due to (let $\tilde{k} = T(n)-k $)
$$
\sum_{k=1}^{T(n)-1}(1-\lambda \gamma)^{2T(n)-2k}   (1-\lambda\gamma)^{h(k-1)} = (1-\lambda \gamma)^{h(T(n)-1)} \sum_{\tilde{k}=1}^{T(n)-1}   (1-\lambda \gamma)^{(2-h)\tilde{k}}.
$$
$$ \| I_5 \|_2  \lesssim  \frac{1}{n-T(n)+1} \sum_{k=T(n)}^n \| R_{n,k}\|_2 \lesssim \frac{1}{n-T(n)+1} \sum_{k=T(n)}^n [\gamma + (1-\lambda \gamma)^{k-1}] \lesssim \gamma+\frac{(1-\lambda\gamma)^{T(n)-1}}{(n-T(n)+1)\gamma}.$$
And
\begin{align*}
\| I_6 \|_2^2 &\lesssim \frac{1}{(n-T(n)+1)^2} \sum_{k=T(n)}^n \|D_{n,k}\|_2^2 \\
&\lesssim \frac{1}{(n-T(n)+1)^2} \sum_{k=T(n)}^n \Big[(1-\lambda \gamma)^{h(k-1)} + \gamma^h \Big] \\
&\lesssim \frac{(1-\lambda \gamma)^{h(T(n)-1)}}{(n-T(n)+1)^2 \gamma} + \frac{\gamma^h}{n-T(n)+1}.
\end{align*}
Taking square roots of $I_3$ and $I_6$, we have
$$\|I_3\|_2 \lesssim \frac{(1-\lambda \gamma)^{h(T(n)-1)/2+1-h/2}}{\sqrt{\gamma}(n-T(n)+1)}+ \frac{\gamma^{(h-1)/2}}{n-T(n)+1},$$
$$\|I_6\|_2\lesssim \frac{(1-\lambda \gamma)^{h(T(n)-1)/2}}{(n-T(n)+1)\sqrt{\gamma}}+ \frac{\gamma^{h/2}}{\sqrt{n-T(n)+1}}.$$
Define
$$L_n = \frac{1}{n-T(n)+1} \sum_{k=T(n)}^{n} A_n^{-1}Z_{n,k}+I_1+I_4.$$
By the triangle inequality, we bound $\|\bar{\Delta}_n-L_n \|_2$ using the remaining components ($I_0, I_2, I_3, I_5, I_6$). After absorbing faster decaying components, we have
\begin{align}\label{general_bound}
   \sqrt{n} \|\bar{\Delta}_n^{\mathrm{SF}}-L_n \|_2 \lesssim \underbrace{\sqrt{n}\gamma}_{J_1} &+ \underbrace{\frac{\sqrt{n}(1-\lambda\gamma)^{T(n)-1}}{(n-T(n)+1)\gamma}}_{J_2} + \underbrace{\frac{\sqrt{n}(T(n)-1)(1-\lambda\gamma)^{T(n)-1}}{n-T(n)+1}}_{J_3} \nonumber
   \\&+ \underbrace{\frac{\sqrt{n}(1-\lambda \gamma)^{h(T(n)-1)/2}}{(n-T(n)+1)\sqrt{\gamma}}}_{J_4}
    +\underbrace{\frac{\sqrt{n}\gamma^{(h-1)/2}}{n-T(n)+1}}_{J_5} + \underbrace{\frac{\sqrt{n}\gamma^{h/2}}{\sqrt{n-T(n)+1}}}_{J_6}.
\end{align}
Specifically, the terms $J_1$ through $J_6$ originate from the upper bounds of the remaining $I_i$ components:
\begin{itemize}
    \item $J_1$ is the constant bias term from the non-linear Taylor remainder in $I_5$.
    \item $J_2$ is the slower exponential decay of the Taylor remainder in $I_5$ (which also absorbs $I_0$).
    \item $J_3$ is the exponential decay coupled with linear $T(n)-1$ prefix from $I_2$.
    \item $J_4$ is the exponential decay of the Hölder noise term from $I_6$ (which absorbs the slightly faster $(1-\lambda\gamma)^{h(T(n)-1)/2+1-h/2}$ from $I_3$).
    \item $J_5$ is the constant term derived from the Hölder noise prefix in $I_3$ (which absorbs the smaller $1/(n-T(n)+1)$ constant in $I_2$).
    \item $J_6$ is the concentration rate of the Hölder noise generated in $I_6$.
\end{itemize} 
The sum $L_n$ consists of independent and centered random vectors:
\begin{align*}
L_n = \frac{1}{n-T(n)+1} &\Big[ \sum_{k=1}^{T(n)-1}(A_n^{-1}({\bf I}_d-\gamma A_n)^{T(n)-k}-A_n^{-1}({\bf I}_d-\gamma A_n)^{n-k+1} )Z_{n,k} +\\ &\sum_{k=T(n)}^{n} (A_n^{-1}-A_n^{-1}({\bf I}_d-\gamma A_n)^{n-k+1} )Z_{n,k}\Big].
\end{align*}
For $1 \leq k \leq T(n)-1$, 
$$ \| (A_n^{-1}({\bf I}_d-\gamma A_n)^{T(n)-k}-A_n^{-1}({\bf I}_d-\gamma A_n)^{n-k+1} )Z_{n,k} \|_p \leq 2 |A_n^{-1}| \ti{M}_p (1-\gamma \lambda)^{T(n)-k},$$
and for $ T(n) \leq k \leq n$, $\| (A_n^{-1}-A_n^{-1}({\bf I}_d-\gamma A_n)^{n-k+1} )Z_{n,k}\|_p \leq 2  |A_n^{-1}| \ti{M}_p$. Again by Theorem 2.1 in \textcite{mies_sequential_2023}, the bound factor is 
\begin{align*}
& \quad \ \ \frac{2 |A_n^{-1}| \ti{M}_p }{n-T(n)+1}\sqrt{\sum_{k=1}^{T(n)-1}(1-\gamma \lambda)^{2T(n)-2k}+\sum_{k=T(n)}^{n}}1\\
&\lesssim \frac{\sqrt{n-T(n)+1+\gamma^{-1}[(1-\gamma\lambda)^2-(1-\gamma\lambda)^{2T(n)}]}}{n-T(n)+1} \\
&\lesssim \frac{\sqrt{n-T(n)+1+\gamma^{-1}}}{n-T(n)+1}. 
\end{align*}
So there exists some Gaussian vector $\zeta_n \sim \N(\mathbf 0_d,V_n)$ such that
\begin{align}
 \|\sqrt{n}L_n - \zeta_n\|_2 \lesssim \frac{n^{1/p}\sqrt{\log n (n-T(n)+1+\gamma^{-1})}}{n-T(n)+1}. 
 \label{eq:MS_suffix}  
\end{align} 
Combining \eqref{general_bound} and \eqref{eq:MS_suffix} yields the desired result.
\end{proof}
\subsection{Discussion on the convergence rate in Lemma \ref{GA_gSGD}}\label{discussion_ratre}
Before establishing the optimal hyperparameter rates, we consider the basic asymptotic regime where $n-T(n) \asymp n$ and $ \gamma \gtrsim 1/n$. Under this setting, \eqref{eq:MS_suffix} can be simplified to 
\begin{align}
 \|\sqrt{n}L_n - \zeta_n\|_2 \lesssim n^{1/p-1/2}\sqrt{\log n}.  
\end{align} 
 Notice that the order of $J_5$ is no greater than $J_6$. Taking their asymptotic ratio yields:
\begin{equation}
    \frac{J_5}{J_6} \asymp \frac{\gamma^{(h-1)/2}/n}{\gamma^{h/2}/\sqrt{n}} = \frac{1}{\sqrt{n\gamma}}.
\end{equation}
Since $ \gamma \gtrsim 1/n$, $J_5$ can be absorbed into $J_6$. The simplified rate in Lemma \ref{GA_gSGD} is obtained by elementary calculations.

To achieve $\sqrt{n}\|\bar{\Delta}_n-L_n \|_2 = o(1)$  imposes the following set of sufficient conditions on $T(n)$ and $\gamma$:
\begin{itemize}
    \item For $J_1 \asymp \sqrt{n}\gamma = o(1)$, we explicitly require $\gamma = o(n^{-1/2})$. This automatically ensures $J_6 \asymp \gamma^{h/2} = o(1)$ since it implies $\gamma \to 0$.
    \item Bounding $J_2, J_3 = o(1)$ naturally yields the conditions $\exp\{-\lambda\gamma T(n)\}/\gamma = o(\sqrt{n})$ and \\
    $T(n)\exp\{-\lambda\gamma T(n)\} = o(\sqrt{n})$, using the fact $1-x \leq \exp (-x)$.
    \item It is clear that $J_4 \lesssim 1/\sqrt{n\gamma} \rightarrow 0$ when $n-T(n) \asymp n$ and $n\gamma \to \infty$.
\end{itemize}
So we conclude that when 
$$\gamma =o(n^{-1/2}), \quad \frac{\exp\{-\lambda\gamma T(n) \}}{\gamma} = o(\sqrt{n}) ,\quad T(n)\exp\{-\lambda\gamma T(n) \} = o(\sqrt{n}), $$
the approximation rate converges to $0$, i.e., $\|\sqrt{n}\bar{\Delta}_n- \sqrt{n}L_n \|_2 = o(1)$.

Now let the suffix proportion be $T(n) = \mathcal{T} n^\kappa$ and the learning rate be $\gamma = c_0 n^{-\nu} \log n$, where the constraints are $\kappa \in (0, 1]$, $1/2 < \nu \le 1$, the proportion coefficient $\mathcal{T} >0$, and the constant $c_0 > 0$. Since
\begin{equation}
    (1-\lambda\gamma)^{T(n)} \le \exp(-\lambda \gamma T(n)) = \exp\left(-\lambda c_0 \mathcal{T} n^{\kappa-\nu} \log n\right),
\end{equation}
the convergence rate can be further bounded as 
\begin{align*}
   \sqrt{n} \|\bar{\Delta}_n^{\mathrm{SF}}-L_n \|_2 \lesssim& n^{1/2-\nu} \log n + \frac{ n^{\nu-1/2} }{\log n}\exp\left(-\lambda c_0 \mathcal{T} n^{\kappa-\nu} \log n\right)+n^{\kappa-1/2} \exp\left(-\lambda c_0 \mathcal{T} n^{\kappa-\nu} \log n\right) \\
&+ \frac{n^{\nu/2-1/2}  \exp\left(-\lambda c_0 \mathcal{T} n^{\kappa-\nu} h\log n /2\right)}{\sqrt{\log n}} + n^{-\nu h/2} (\log n)^{h/2} . 
\end{align*}
Applying the AM-GM inequality to the first two terms, they are lower bounded by $ \exp(-\lambda c_0 \mathcal{T} n^{\kappa-\nu} \log n)$. Its convergence enforces that $\kappa \ge \nu$. Then the rate simplifies to 
\begin{align*}
   \sqrt{n} \|\bar{\Delta}_n^{\mathrm{SF}}-L_n \|_2 \lesssim& n^{1/2-\nu} \log n +n^{\kappa-1/2} \exp\left(-\lambda c_0 \mathcal{T} n^{\kappa-\nu} \log n\right) \\
&+ \frac{n^{\nu/2-1/2}  \exp\left(-\lambda c_0 \mathcal{T} n^{\kappa-\nu} h\log n /2\right)}{\sqrt{\log n}} + n^{-\nu h/2} (\log n)^{h/2} . 
\end{align*}

If $\kappa > \nu$, the exponential term decays faster than any polynomial rate. So the dominant term will be $n^{1/2-\nu} \log n +n^{-\nu h/2} (\log n)^{h/2}$.

If $\kappa = \nu$, the convergence rate becomes
\begin{align*}
   \sqrt{n} \|\bar{\Delta}_n^{\mathrm{SF}}-L_n \|_2 \lesssim& n^{1/2-\nu} \log n +n^{\nu-1/2 -\lambda c_0 \mathcal{T} } + \frac{n^{\nu/2-1/2-\lambda c_0 \mathcal{T}h/2 } }{\sqrt{\log n}} + n^{-\nu h/2} (\log n)^{h/2} . 
\end{align*}
It is clear that when $\nu =1$ and $ \lambda c_0\mathcal{T} \ge 1$, the right hand side is optimized as $n^{- h/2} (\log n)^{h/2}$ when $h<1$ and $n^{-1/2} \log n$ when $h=1$. Comparing both cases, the convergence rate is optimized at $\kappa = \nu = 1$ and $ \lambda c_0\mathcal{T} \ge 1$.

\subsection{Properties of the suffix averaging covariance matrix}
Define 
\begin{align}
\ti{V}_n(\gamma)
= \frac{n}{(n-T(n)+1)^2}\Bigg[
\sum_{k=1}^{T(n)-1}&A_n^{-1}
\Big(
({\bf I}_d-\gamma A_n)^{T(n)-k}
-
({\bf I}_d-\gamma A_n)^{n-k+1}
\Big) \nonumber \\
&\times S_n
\Big(
({\bf I}_d-\gamma A_n)^{T(n)-k}
-
({\bf I}_d-\gamma A_n)^{n-k+1}
\Big)A_n^{-1} \nonumber \\
+
\sum_{k=T(n)}^n&A_n^{-1}
\Big(
{\bf I}_d-({\bf I}_d-\gamma A_n)^{n-k+1}
\Big)S_n \nonumber \\
&\times
\Big(
{\bf I}_d-({\bf I}_d-\gamma A_n)^{n-k+1}
\Big)A_n^{-1}
\Bigg],
\end{align}
where $A_n$ and $S_n$ are defined in \eqref{truncatedA} and \eqref{truncatedS}. When $0 < \gamma < 1/\lambda^*$, the matrix is strictly positive definite since $\gamma ^2 S_n$ is positive definite (the $k=n$ term in the summation) and all other terms are positive semi-definite. 
\begin{proposition}\label{truncated_V}
Let \(0<\gamma<1/\lambda^*\) and
\(n-T(n)+1\asymp n\). Under Assumptions \ref{as1}-\ref{as2}:
\begin{itemize}
    \item If \(\gamma\gtrsim n^{-1}\), the smallest and largest
    eigenvalues of \(\widetilde V_n(\gamma)\) and $V_n(\gamma)$ are uniformly bounded away from zero and infinity.
    \item Uniformly over \(0<\gamma<1/\lambda^*\) and $n \ge 1$, $\left|
    \widetilde V_n(\gamma)-V_n(\gamma)
    \right|
    \lesssim n^{\vartheta(2-p)}$.
    \item If \(n/(n-T(n)+1)\to w\geq1\), then (recall that $V = wA^{-1}SA^{-1}$)
    \[\left|V_n(\gamma)-V
    \right|
    \lesssim \frac{1}{(n-T(n)+1)\gamma} +\left|
    \frac{n}{(n-T(n)+1)}-w
    \right|.
    \]
    If, in addition, \(n\gamma\to\infty\), then $V_n(\gamma)\longrightarrow V$.
\end{itemize}

\end{proposition}

\subsection{Proof of Theorem \ref{mainGA}}
\begin{proof}
We consider the truncated version of quantile regression suffix averaging SSGD 
$$\ti{\bbeta}_n^{\mathrm{SF}} = (\sum_{k=T(n)}^n \ti\bbeta_{n,k}) / (n-T(n)+1).$$ 
Under Assumptions \ref{as1}-\ref{as2}, Lemma \ref{moment_tilde}, Lemma \ref{uniform_matrix_bound}, and Lemma \ref{Holder_truncated} verify all conditions of Lemma \ref{GA_gSGD} for $\ti\bbeta_{n,k}$ with $(\lambda,\lambda^*,\lambda_*)$ defined in \eqref{lambda} and Lemma \ref{uniform_matrix_bound}, $\ti{M}_p = M_p$, and $h=1/2$. The learning-rate conditions in Theorem \ref{clt_last} ensure that $\gamma< 1/\lambda^*$. When $n-T(n) \asymp n $ and $\gamma \gtrsim 1/n$, by Lemma \ref{GA_gSGD}, on a potentially enriched probability space, there exists some centered Gaussian vector $\ti{\zeta}_n\sim \N(\mathbf 0_d,\ti{V}_n(\gamma))$ such that 
\begin{align}\label{GA_truncated}
    \| \sqrt{n}(\ti{\bbeta}_n^{\mathrm{SF}}-\bbeta^*) - \ti{\zeta}_n \| \lesssim\ & \sqrt{n}\gamma + \frac{(1-\lambda\gamma)^{T(n)-1}}{\sqrt{n}\gamma} + \frac{(T(n)-1)(1-\lambda \gamma)^{T(n)-1}}{\sqrt{n}} \nonumber \\
    &+ \frac{(1-\lambda\gamma)^{T(n)/4-1/4}}{\sqrt{n\gamma}}+ \gamma^{1/4} + n^{1/p-1/2}\sqrt{\log n} \nonumber \\
    = \E_n^{\mathrm{SF}}.
\end{align}
Notice that $(\sqrt{n}(\ti{\bbeta}_n^{\mathrm{SF}}-\bbeta^*) -\ti\zeta_n) \One_{\G_n} = (\sqrt{n}(\bbeta_n^{\mathrm{SF}}-\bbeta^*) -\ti\zeta_n) \One_{\G_n}$, we have
$$  \| [\sqrt{n}(\bbeta_n^{\mathrm{SF}}-\bbeta^*) - \ti\zeta_n]\One_{\G_n} \|= \| [\sqrt{n}(\ti{\bbeta}_n^{\mathrm{SF}}-\bbeta^*) - \ti\zeta_n]\One_{\G_n} \| \leq  \| \sqrt{n}(\ti{\bbeta}_n^{\mathrm{SF}}-\bbeta^*) - \ti\zeta_n \|.$$
Define $\zeta_n = V_n(\gamma)^{1/2}\ti{V}_n(\gamma)^{-1/2} \ti{\zeta}_n$. Following the same argument in the proof of Theorem \ref{clt_last}, we have
$$\| \zeta_n - \ti{\zeta}_n\| \lesssim | V_n(\gamma)^{1/2}-\ti{V}_n(\gamma)^{1/2} |  \leq
\frac{
\left|
V_n(\gamma)-\ti{V}_n(\gamma)
\right|
}{
\sqrt{\Lambda_{\min}(V_n(\gamma))}
+
\sqrt{\Lambda_{\min}(\ti{V}_n(\gamma))}
} \lesssim n^{\vartheta (2-p)},$$
where the last inequality is due to the first and second statements in Proposition \ref{truncated_V}. Since $\vartheta> 1/p$, $n^{\vartheta (2-p)} \leq n^{1/p-1/2}$ and can be absorbed into the previous bound. So we obtain that
\begin{align*}
    \| [\sqrt{n}(\bbeta_n^{\mathrm{SF}}-\bbeta^*) - \zeta_n ] \One_{\G_n}\| \lesssim \E_n^{\mathrm{SF}},
\end{align*}
Finally, the same argument as \eqref{proofOp} in the proof of Theorem \ref{thm_moment_sgd} implies that there exists a constant $C_3$ independent of $(n,\gamma,\delta)$ such that for every $\delta \in (0,1)$,
$$ \PP\left[\left| \sqrt{n}\big(\bbeta_n^{\mathrm{SF}}-\bbeta^*\big)  - \zeta_n\right|  > \frac{C_3\E_n^{\mathrm{SF}}}{\sqrt{\delta}} \right]  \leq \delta + M_p^pn^{1-p\vartheta},$$
The simplified rates in Theorem \ref{mainGA} are derived in \ref{discussion_ratre} with $h=1/2$.
\end{proof}

\begin{remark}\label{rmk_subexp}
 When the distribution of the covariate is sub-exponential, by Theorem 1, Theorem 4, and Corollary 4 in \textcite{zhang_piecewise_2025}, there exist some constants $\gamma_0>0$ and $\lambda>0$ such that $\| \bbeta_n -\bbeta^*\|^2_4 \lesssim (1-\lambda \gamma)^{n}+ \gamma$ for all $0 < \gamma \leq\gamma_0$. Together with Lemma \ref{Holder}-\ref{Taylorexpansion}, we have verified all conditions of Lemma \ref{GA_gSGD} for the original quantile regression SSGD with $h = 1/2$. As a result, the $\mathcal{L}^2$ Gaussian coupling theorem holds with the same convergence rate as in Theorem \ref{mainGA}, i.e., $\| \sqrt{n}\big(\bbeta_n^{\mathrm{SF}}-\bbeta^*\big)  - \zeta_n\| \lesssim \E_n^{\mathrm{SF}}$. The identical argument can prove the $\mathcal{L}^2$ analogue of Theorem \ref{clt_last}.
\end{remark}

\section{Proof of Theorem \ref{coverage}}
\begin{proof}
Recall that each of $K$ batches contains $ n = N/K$ samples. Define 
\begin{align}
 \G_{N,n} := \{\max_{1\le t\le N}|\bx_t| \le \sqrt{n}\},
\end{align}
and define the truncated sequence $\{\ti{\bbeta}_{n,i}\}_{1\leq i \leq n}$ of each batch in the same way as \eqref{truncate} with $\vartheta = 1/2$. Since $\gamma = c_0n^{-\nu}\log n$ with $\nu>1/2$, the truncation threshold here satisfies $\gamma \sqrt{n}\lesssim 1$. Hence the results of Lemma \ref{moment_tilde}, Lemma \ref{GA_gSGD}, and \eqref{GA_truncated} apply for the truncated sequence in each batch. By Theorem 2 in \textcite{zhu_high_2024} and \eqref{GA_truncated}, we have
\begin{align}\label{GAbound1}
    \sup_{s \in \RR}| \PP(\ti{t}_{\bm{v}} \ge s) - \PP(T_{K-1}\ge s)|
    \lesssim (\phi(N/K))^{1/4},
\end{align}
where $\ti{t}_{\bm{v}}$ is constructed in the same way as \eqref{eq:parallel_avg} and \eqref{eq:t_stat}, replacing all $\{\bbeta_i\}_{1\leq i \leq n}$ in each batch with $\{\ti{\bbeta}_{n,i}\}_{1\leq i \leq n}$. Here \eqref{GA_truncated} verifies Assumption 4 in \textcite{zhu_high_2024}, and their Theorem 2 applies. 

Then it suffices to bound $\sup_s|\PP(\ti{t}_{\bm{v}} \ge s)- \PP(\hat{t}_{\bm{v}} \ge s)|$.  Notice that $\PP(\ti{t}_{\bm{v}} \ge s) = \PP(\ti{t}_{\bm{v}} \ge s,\G_{N,n})+ \PP(\ti{t}_{\bm{v}} \ge s,\G^c_{N,n})$, and the same decomposition works for $\PP(\hat{t}_{\bm{v}} \ge s)$. We also have $\PP(\ti{t}_{\bm{v}} \ge s,\G_{N,n}) = \PP(\hat{t}_{\bm{v}} \ge s,\G_{N,n}) $ since $\hat{t}_{\bm{v}} =\ti{t}_{\bm{v}} $ on $\G_{N,n}$. Therefore, uniformly over $s\in\RR$,
\begin{align*}
\left|
\PP(\ti t_{\bm v}\ge s)-\PP(\hat t_{\bm v}\ge s)
\right| &=   |\PP(\ti{t}_{\bm{v}} \ge s,\G^c_{N,n})- \PP(\hat{t}_{\bm{v}} \ge s,\G^c_{N,n})|\\
&\leq  \max\{\PP(\ti{t}_{\bm{v}} \ge s,\G^c_{N,n}), \PP(\hat{t}_{\bm{v}} \ge s,\G^c_{N,n})\}\\
&\leq \PP(\G_{N,n}^c)
\leq
\sum_{t=1}^N\PP(|\bx_t|>\sqrt n)
\leq
\frac{NM_p^p}{n^{p/2}}
=
KM_p^pn^{1-p/2}.
\end{align*}
Let $C_4$ denote the implicit constant in \eqref{GAbound1}. By the triangle inequality,
\begin{align*}
& \quad \ \sup_{s\in\RR}
\left|
\PP(\hat t_{\bm v}\ge s)-\PP(T_{K-1}\ge s)
\right| \\
 & \leq
\sup_{s\in\RR}
\left|
\PP(\hat t_{\bm v}\ge s)-\PP(\ti t_{\bm v}\ge s)
\right|
+
\sup_{s\in\RR}
\left|
\PP(\ti t_{\bm v}\ge s)-\PP(T_{K-1}\ge s)
\right|\\
&\leq
KM_p^pn^{1-p/2}+C_4\phi(n)^{1/4}.
\end{align*}
This completes the proof.
\end{proof}
\begin{proof}[Proof of Corollary \ref{Corollary_cov}]
Notice that $n^{1-p/2} \lesssim \{\phi(n)\}^{1/4}$. This result directly follows from Theorem \ref{coverage}.
\end{proof}

\section{Proof of Auxiliary Lemmas}

\begin{proof}[Proof of Lemma \ref{Holder}]
\vspace{0.5em}
\textbf{Proof of Statement 1:}
By the law of total expectation:
\begin{align*}
    \EE [\nabla g(\bbeta, \bx, y)] &= \EE \left\{ \EE [-\bx(\tau - \One_{\{y < \bx^{\top}\bbeta\}}) \mid \bx] \right\} \\
    &= \EE\left\{ \bx \left( \EE[\One_{\{y < \bx^{\top}\bbeta\}}\mid \bx] - \tau \right) \right\}.
\end{align*}
Moreover,
\begin{equation*}
    \EE[\One_{\{y < \bx^{\top}\bbeta\}}\mid \bx] = \PP(y < \bx^{\top}\bbeta \mid \bx) = \PP(y-\bx^{\top}\bbeta^* \leq \bx^{\top}\bbeta-\bx^{\top}\bbeta^*  \mid \bx) = F_{\epsilon\mid\bx}(\bx^{\top}(\bbeta-\bbeta^*)).
\end{equation*}
Here we can replace $<$ and $\leq$ with each other because of the atomless conditional distribution. Substituting this back yields:
\begin{equation*}
    m(\bbeta) = \EE \left\{ [F_{\epsilon\mid\bx}(\bx^{\top} (\bbeta-\bbeta^*))-\tau]\bx \right\}.
\end{equation*}

\vspace{0.5em}
\textbf{Proof of Statement 2:}
We have
\begin{equation*}
    \nabla g(\bbeta, \bx, y) - \nabla g(\bbeta', \bx, y) = \bx \left( \One_{\{y < \bx^{\top}\bbeta\}} - \One_{\{y < \bx^{\top}\bbeta'\}} \right).
\end{equation*}
Since the difference of the two indicator functions takes values in $\{-1, 0, 1\}$, its square equals its absolute value. Thus:
\begin{equation*}
    |\nabla g(\bbeta, \bx, y) - \nabla g(\bbeta', \bx, y)|^2 = |\bx|^2 \left| \One_{\{y < \bx^{\top}\bbeta\}} - \One_{\{y < \bx^{\top}\bbeta'\}} \right|.
\end{equation*}
Taking the conditional expectation given $\bx$:
\begin{align*}
    \EE[|\nabla g(\bbeta, \bx, y) - \nabla g(\bbeta', \bx, y)|^2 \mid \bx] &= |\bx|^2 \left| \PP(y < \bx^{\top}\bbeta \mid \bx) - \PP(y < \bx^{\top}\bbeta' \mid \bx) \right| \\
    &= |\bx|^2 \left| F_{\epsilon\mid\bx}(\bx^{\top}(\bbeta-\bbeta^*)) - F_{\epsilon\mid\bx}(\bx^{\top}(\bbeta'-\bbeta^*)) \right|.
\end{align*}
From Assumption \ref{as1}, $f_{\epsilon\mid\bx}(t) \leq \M_f$ for all $t$. This implies that the cumulative distribution function $F_{\epsilon\mid\bx}$ is globally $\M_f$-Lipschitz continuous. Applying this property:
\begin{align*}
   \EE[|\nabla g(\bbeta, \bx, y) - \nabla g(\bbeta', \bx, y)|^2 \mid \bx] &\leq |\bx|^2 \cdot \M_f \left| \bx^{\top}(\bbeta - \bbeta') \right| \\
    &\leq \M_f |\bx|^3 |\bbeta - \bbeta'|.
\end{align*}
Taking the expectation over $\bx$ and applying the moment condition in Assumption \ref{as2}:
\begin{equation*}
    \EE[|\nabla g(\bbeta, \bx, y) - \nabla g(\bbeta', \bx, y)|^2] \leq \M_f \|\bx\|_3^3 |\bbeta - \bbeta'| \leq \M_f\cdot M^3_p |\bbeta - \bbeta'|,
\end{equation*}
which completes the proof by taking the square root. Finally, by Jensen's inequality,
$$ | m(\bbeta)- m(\bbeta')| \leq \|\nabla g(\bbeta, \bx, y) - \nabla g(\bbeta', \bx, y)\|\leq \sqrt{\M_f\cdot M^3_p }|\bbeta-\bbeta'|^{1/2}.$$
\end{proof}

\begin{proof}[Proof of Lemma \ref{Taylorexpansion}]
The difference between the Hessian and its linear analogue can be written as:
\begin{align*}
m(\bbeta) - A(\bbeta-\bbeta^*) 
&= \EE \left\{ \left[ F_{\epsilon\mid\bx}(\bx^{\top}(\bbeta-\bbeta^*)) - F_{\epsilon\mid\bx}(0) - f_{\epsilon\mid\bx}(0)\bx^{\top}(\bbeta-\bbeta^*) \right] \bx \right\}.
\end{align*}

Define the remainder function $R_{\bx}(t) = F_{\epsilon\mid\bx}(t) - F_{\epsilon\mid\bx}(0) - f_{\epsilon\mid\bx}(0)t$. We bound $|R_{\bx}(t)|$ by considering two cases based on the $r$-neighborhood from Assumption \ref{as1}:

\textbf{Case 1: $|t| \leq r$.} 
We express $R_{\bx}(t)$ in its integral form: $R_{\bx}(t) = \int_0^t ( f_{\epsilon\mid\bx}(u) - f_{\epsilon\mid\bx}(0) ) du$. 
Applying the local Lipschitz condition from Assumption \ref{as1}:
\begin{equation*}
|R_{\bx}(t)| \leq \left| \int_0^t |f_{\epsilon\mid\bx}(u) - f_{\epsilon\mid\bx}(0)| du \right| \leq \left| \int_0^t L|u| du \right| = \frac{L}{2}t^2.
\end{equation*}

\textbf{Case 2: $|t| > r$.} 
Since $F_{\epsilon\mid\bx}$ is a cumulative distribution function, $|F_{\epsilon\mid\bx}(t) - F_{\epsilon\mid\bx}(0)| \leq 1$. Using the bound $f_{\epsilon\mid\bx}(0) \leq \M_f$ from Assumption \ref{as1}, we have:
\begin{equation*}
|R_{\bx}(t)| \leq |F_{\epsilon\mid\bx}(t) - F_{\epsilon\mid\bx}(0)| + f_{\epsilon\mid\bx}(0)|t| \leq 1 + \M_f|t|.
\end{equation*}
Given $|t| > r$, it holds that $1 < t^2/r^2$ and $|t| < t^2/r$. Substituting these into the inequality above yields:
\begin{equation*}
|R_{\bx}(t)| \leq \frac{1}{r^2}t^2 + \frac{\M_f}{r}t^2 = \left( \frac{1}{r^2} + \frac{\M_f}{r} \right) t^2.
\end{equation*}

Combining both cases, for all $t \in \mathbb{R}$, there exists a constant $C_R = \max\left( L/2, 1/r^2 + \M_f/r \right)$ such that $|R_{\bx}(t)| \leq C_R t^2$.

Now, substituting $t = \bx^{\top}(\bbeta-\bbeta^*)$ into this global bound and applying the Cauchy-Schwarz inequality $(\bx^{\top}(\bbeta-\bbeta^*))^2 \leq |\bx |^2|\bbeta-\bbeta^*|^2$:
\begin{equation*}
|R_{\bx}(\bx^{\top}(\bbeta-\bbeta^*))| \leq C_R (\bx^{\top}(\bbeta-\bbeta^*))^2 \leq C_R |\bx|^2 |\bbeta-\bbeta^*|^2.
\end{equation*}
Finally,
\begin{align*}
| m(\bbeta) - A(\bbeta-\bbeta^*) | 
&= | \EE [ R_{\bx}(\bx^{\top}(\bbeta-\bbeta^*)) \bx ] | \\
&\leq \EE \left[ |\bx| \cdot |R_{\bx}(\bx^{\top}(\bbeta-\bbeta^*))| \right] \\
&\leq \EE \left[ |\bx| \cdot C_R |\bx|^2 |\bbeta-\bbeta^*|^2 \right] \\
&= C_R \EE[|\bx|^3] |\bbeta-\bbeta^*|^2.
\end{align*}
Since $\|\bx\|_3 < \infty$ under Assumption \ref{as2}, we conclude that:
\begin{equation*}
\| m(\bbeta) - A(\bbeta-\bbeta^*)\| \lesssim \|\bbeta-\bbeta^*\|^2,
\end{equation*}
which completes the proof.
\end{proof}

\begin{proof}[Proof of Lemma \ref{Holder_truncated}]
For the first statement, the same conditional-expectation argument as in
the proof of Statement 1 of Lemma \ref{Holder} gives
\begin{align*}
\ti{m}_n(\bbeta)
&=
\EE\left[
    \V_n(\bx)
    \left(
        \PP(y<\bx^\top\bbeta\mid\bx)-\tau
    \right)
\right] \\
&=
\EE\left\{
    \left[
        F_{\epsilon\mid\bx}\left(\bx^\top(\bbeta-\bbeta^*)\right)-\tau
    \right]\V_n(\bx)
\right\}.
\end{align*}
Since \(F_{\epsilon\mid\bx}(0)=\tau\), it follows immediately that
\(\ti{m}_n(\bbeta^*)=0\).

For the second statement, repeating the indicator-difference calculation in the proof of Statement 2
of Lemma \ref{Holder}, we obtain
\begin{align*}
\EE\left[|\nabla \ti{g}_n(\bbeta,\bx,y) -
\nabla \ti{g}_n(\bbeta',\bx,y)|^2\mid\bx\right]
&=
|\V_n(\bx)|^2
\left|
F_{\epsilon\mid\bx}\left(\bx^\top(\bbeta-\bbeta^*)\right)
-
F_{\epsilon\mid\bx}\left(\bx^\top(\bbeta'-\bbeta^*)\right)
\right| \\
&\leq
\M_f|\V_n(\bx)|^2
\left|\bx^\top(\bbeta-\bbeta')\right| \\
&\leq
\M_f|\bx|^3|\bbeta-\bbeta'|,
\end{align*}
where the last inequality follows from   $|\V_n(\bx)|\leq|\bx|$. Therefore, $\|\nabla \ti{g}_n(\bbeta,\bx,y) -
\nabla \ti{g}_n(\bbeta',\bx,y)\|_2^2
\leq
\M_f\EE|\bx|^3|\bbeta-\bbeta'|$,
and hence $\ti{C}_3=\left(\M_f\EE|\bx|^3\right)^{1/2}$.
This constant is finite under Assumption \ref{as2} and is independent of \(n\). The bound of $|\ti{m}_n(\bbeta)-\ti{m}_n(\bbeta')|$ follows from the same argument in the proof of Lemma \ref{Holder}.

For the third statement, the proof of Lemma \ref{Taylorexpansion} shows that, uniformly in \(\bx\), $
|R_{\bx}(t)|
\leq
C_Rt^2$. Consequently, we have
\begin{align*}
\ti{m}_n(\bbeta)-A_n(\bbeta-\bbeta^*)
&=
\EE\left[
    R_{\bx}(\bx^\top(\bbeta-\bbeta^*))\V_n(\bx)
\right].
\end{align*}
Using the bound of $R_{\bx}(\cdot)$ and
\(|\V_n(\bx)|\leq|\bx|\), we obtain
\begin{align*}
\left|\ti{m}_n(\bbeta)-A_n(\bbeta-\bbeta^*)\right|
&\leq
C_R\EE\left[
    |\V_n(\bx)|\,|\bx^\top(\bbeta-\bbeta^*)|^2
\right] \\
&\leq
C_R\EE|\bx|^3|\bbeta-\bbeta^*|^2.
\end{align*}
Thus
\[
\left|
    \ti{m}_n(\bbeta)-A_n(\bbeta-\bbeta^*)
\right|
\leq
\ti{C}_4|\bbeta-\bbeta^*|^2,
\qquad
\ti{C}_4=C_R\EE|\bx|^3.
\]
The constant \(\ti{C}_4\) is finite under Assumption \ref{as2} and is independent
of \(n\). This completes the proof.

\end{proof}

\subsection{Proof of Lemma \ref{verify_conditions}}
\begin{proof}
By Lemma \ref{Holder_truncated},
\begin{equation}\label{eq:truncated-population-gradient}
    \widetilde m_n(\bbeta)
    =
    \EE\left[
        \V_n(\bx)
        \left\{
            F_{\epsilon\mid\bx}\bigl(\bx^\top(\bbeta-\bbeta^*)\bigr)-\tau
        \right\}
    \right],
\end{equation}
and $\widetilde m_n(\bbeta^*)=0$. We first prove the global Lipschitz property. Since $|\V_n(\bx)|\leq|\bx|$, by Assumption \ref{as1}, it follows from
\eqref{eq:truncated-population-gradient} that
\begin{align*}
    |\widetilde m_n(\bbeta)-\widetilde m_n(\bbeta')|
    &\leq
    \EE\left[
        |\V_n(\bx)|
        \left|
            F_{\epsilon\mid\bx}\bigl(\bx^\top(\bbeta-\bbeta^*)\bigr)
            -
            F_{\epsilon\mid\bx}\bigl(\bx^\top(\bbeta'-\bbeta^*)\bigr)
        \right|
    \right]\\
    &\leq
    \M_f\EE|\bx|^2|\bbeta-\bbeta'|\\
    &\leq
    \M_fM_p^2|\bbeta-\bbeta'|.
\end{align*}
We next prove the local strong-monotonicity inequality. By the definition of $\varrho$,
\begin{align}
    \EE\left[
        |\bx|^2\One_{\{|\bx|>\varrho\}}
    \right]
    &\leq
    \frac{\EE|\bx|^p}{\varrho^{p-2}}
    \leq
    \frac12
    \Lambda_{\min}(\EE[\bx\bx^\top]).
    \label{eq:covariance-tail-bound}
\end{align}
Consequently, for every $\bm v\in\RR^d$,
\begin{align}
    \EE\left[
        (\bx^\top\bm v)^2
        \One_{\{|\bx|\leq \varrho\}}
    \right]
    &=
    \bm v^\top\EE[\bx\bx^\top]\bm v
    -\EE\left[
        (\bx^\top\bm v)^2
        \One_{\{|\bx|>\varrho\}}
    \right]   \nonumber  \\
    &\ge  \Lambda_{\min}(\EE[\bx\bx^\top])
    |\bm v|^2- |\bm v|^2 \EE [|\bx|^2\One_{\{|\bx|>\varrho\}} ]
    \geq
    \frac12
    \Lambda_{\min}(\EE[\bx\bx^\top])
    |\bm v|^2.
    \label{eq:bounded-covariance-lower}
\end{align}

Since $n^{\vartheta}\geq1$ and $\varrho\geq1$, on
$\{|\bx|\leq \varrho\}$,
\begin{equation}\label{eq:truncation-weight-lower}
    \min\left\{1,\frac{n^{\vartheta}}{|\bx|}\right\}
    \geq
    \min\left\{1,\frac1{\varrho}\right\}
    =
    \frac1{\varrho}.
\end{equation}
Let
$\bbeta,\bbeta'\in\mathcal B_{\widetilde r}(\bbeta^*)$ and set $u:=\bx^\top(\bbeta-\bbeta^*)$, $u':=\bx^\top(\bbeta'-\bbeta^*)$. If $|\bx|\leq \varrho$, then
\[
    \max\{|u|,|u'|\}
    \leq
    \varrho\widetilde r
    \leq
    \min\left\{
        r,\frac{\mu_\tau}{2L}
    \right\} \leq r.
\]
Therefore, for every point $t$ on the line segment joining $u$ and
$u'$, Assumption \ref{as1} ensures that for $P_{\bx}$-almost every $\bx$,
\begin{align*}
    f_{\epsilon\mid\bx}(t)
    &\geq
    f_{\epsilon\mid\bx}(0)-L|t|
    \geq
    \frac{\mu_\tau}{2}.
\end{align*}
It follows that, almost surely, on $\{|\bx|\leq \varrho\}$,
\begin{align*}
    &\quad (u-u')(F_{\epsilon\mid\bx}(u)-F_{\epsilon\mid\bx}(u'))=
    (u-u')^2
    \int_0^1
        f_{\epsilon\mid\bx}\bigl(u'+s(u-u')\bigr)\,ds  \geq
    \frac{\mu_\tau}{2}(u-u')^2.
\end{align*}
Since every conditional CDF is nondecreasing, we have $(u-u') (F_{\epsilon\mid\bx}(u)-F_{\epsilon\mid\bx}(u'))\geq0$ without any restriction on $\bx$. Hence, using \eqref{eq:bounded-covariance-lower} and \eqref{eq:truncation-weight-lower},
\begin{align*}
    &\quad \ (\bbeta-\bbeta')^\top
    \bigl(\widetilde m_n(\bbeta)
          -\widetilde m_n(\bbeta')\bigr)\\
    &=\EE\left[
        \min\left\{1,\frac{n^{\vartheta}}{|\bx|}\right\}
        (\bx^\top (\bbeta-\bbeta'))
        (F_{\epsilon\mid\bx}(u)-F_{\epsilon\mid\bx}(u'))
    \right]\\
    &=\EE\left[
        \min\left\{1,\frac{n^{\vartheta}}{|\bx|}\right\}
        (u - u')
        (F_{\epsilon\mid\bx}(u)-F_{\epsilon\mid\bx}(u'))
    \right]\\
    &\geq
    \frac{\mu_\tau}{2\varrho}
    \EE\left[
        (\bx^\top (\bbeta-\bbeta'))^2
        \One_{\{|\bx|\leq \varrho\}}
    \right]\\
    &\geq
    \frac{\mu_\tau}{4\varrho}
    \Lambda_{\min}(\EE[\bx\bx^\top])
    |\bbeta-\bbeta'|^2\\
    &=
    \widetilde\mu|\bbeta-\bbeta'|^2.
\end{align*}
This proves the local strong-monotonicity.

Since
$|F_{\epsilon\mid\bx}(t)-\tau|\leq1$ and
$|\V_n(\bx)|\leq|\bx|$, we have, for every $\bbeta\in\RR^d$,
\begin{align*}
    |\widetilde m_n(\bbeta)|
    \leq   \EE|\V_n(\bx)|
    \left|
        F_{\epsilon\mid\bx}\bigl(\bx^\top(\bbeta-\bbeta^*)\bigr)-\tau
    \right| \leq
    \EE|\bx|
    \leq
    M_p.
\end{align*}

It remains to prove the radial lower bound outside $\mathcal B_{\widetilde r}(\bbeta^*)$. The monotonicity of
$F_{\epsilon\mid\bx}$ implies that, for arbitrary
$\bbeta,\bbeta'\in\RR^d$,
\begin{equation}\label{eq:global-monotonicity-truncated}
    (\bbeta-\bbeta')^\top
    \bigl(\widetilde m_n(\bbeta)
          -\widetilde m_n(\bbeta')\bigr)
    \geq0.
\end{equation}
Suppose that $|\bbeta-\bbeta^*|>\widetilde r$ and define
\[
    \bbeta^{\diamond}
    :=
    \bbeta^*
    +
    \widetilde r
    \frac{\bbeta-\bbeta^*}{|\bbeta-\bbeta^*|}.
\]
Since
$\bbeta^{\diamond}\in \mathcal B_{\widetilde r}(\bbeta^*)$,
the local strong-monotonicity inequality and
$\widetilde m_n(\bbeta^*)=0$ yield $(\bbeta^{\diamond}-\bbeta^*)^\top
    \widetilde m_n(\bbeta^{\diamond})
    \geq
    \widetilde\mu\,\widetilde r^2$. Equivalently,
\begin{equation}\label{eq:boundary-radial-drift}
    \frac{(\bbeta-\bbeta^*)^\top}
         {|\bbeta-\bbeta^*|}
    \widetilde m_n(\bbeta^{\diamond})
    \geq
    \widetilde\mu\,\widetilde r.
\end{equation}
Applying \eqref{eq:global-monotonicity-truncated} to
$\bbeta$ and $\bbeta^{\diamond}$ yields
\begin{equation*}
    (\bbeta-\bbeta^{\diamond})^\top
    \bigl(\widetilde m_n(\bbeta)
          -\widetilde m_n(\bbeta^{\diamond})\bigr) = (|\bbeta-\bbeta^*| - \ti r)\frac{\bbeta-\bbeta^*}{|\bbeta-\bbeta^*|}\bigl(\widetilde m_n(\bbeta)
          -\widetilde m_n(\bbeta^{\diamond})\bigr)
    \ge 0.
\end{equation*}
then $|\bbeta-\bbeta^*| - \ti r >0$ implies
\[
    \frac{(\bbeta-\bbeta^*)^\top}
         {|\bbeta-\bbeta^*|}
    \widetilde m_n(\bbeta)
    \geq
    \frac{(\bbeta-\bbeta^*)^\top}
         {|\bbeta-\bbeta^*|}
    \widetilde m_n(\bbeta^{\diamond}).
\]
Combining this with \eqref{eq:boundary-radial-drift}, we obtain
\begin{align*}
    (\bbeta-\bbeta^*)^\top\widetilde m_n(\bbeta) \geq
    \widetilde\mu\,\widetilde r
    |\bbeta-\bbeta^*| =
    \widetilde h|\bbeta-\bbeta^*|.
\end{align*}
This completes the proof.
\end{proof}

\subsection{Proof of Lemma \ref{uniform_matrix_bound}}
\begin{proof}
Recall that
\[
A_n
=
\EE\left[
    f_{\epsilon\mid\bx}(0)\V_n(\bx)\bx^\top
\right],
\]
and
\[
S_n
=
\EE\left[
    \V_n(\bx)\V_n(\bx)^\top
    \left(
        \tau-\One_{\{y-\bx^\top\bbeta^*<0\}}
    \right)^2
\right].
\]
By the definition of $\varrho$ in \eqref{lambda},
\begin{align}
\EE\left[
    |\bx|^2\One_{\{|\bx|>\varrho\}}
\right]
&\leq
\frac{\EE|\bx|^p}{\varrho^{p-2}}
\leq
\frac12
\Lambda_{\min}(\EE[\bx\bx^\top]).
\label{eq:fixed-tail-covariance}
\end{align}
Consequently, for every $\bm v\in\RR^d$,
\begin{align*}
\bm v^\top
\EE\left[
    \bx\bx^\top\One_{\{|\bx|\leq\varrho\}}
\right]
\bm v &=
\bm v^\top\EE[\bx\bx^\top]\bm v
-
\EE\left[
    (\bx^\top\bm v)^2
    \One_{\{|\bx|>\varrho\}}
\right]\\
&\geq
\Lambda_{\min}(\EE[\bx\bx^\top])|\bm v|^2
-
|\bm v|^2
\EE\left[
    |\bx|^2\One_{\{|\bx|>\varrho\}}
\right]\\
&\geq
\frac12
\Lambda_{\min}(\EE[\bx\bx^\top])
|\bm v|^2.
\end{align*}
Therefore,
\begin{equation}\label{eq:fixed-truncated-covariance-lower}
\EE\left[
    \bx\bx^\top\One_{\{|\bx|\leq\varrho\}}
\right]
\succeq
\frac12
\Lambda_{\min}(\EE[\bx\bx^\top])
{\bf I}_d.
\end{equation}

Since $n^\vartheta\geq1$ and $\varrho\geq1$, on
$\{|\bx|\leq\varrho\}$,
\[
\min\left\{
    1,\frac{n^\vartheta}{|\bx|}
\right\}
\geq
\frac1{\varrho},
\]
where the corresponding matrix inequalities are trivial when $\bx=0$.
It follows pointwise that
\begin{align}
\V_n(\bx)\bx^\top
&\succeq
\frac1{\varrho}
\bx\bx^\top
\One_{\{|\bx|\leq\varrho\}},
\label{eq:Vx-lower}\\
\V_n(\bx)\V_n(\bx)^\top
&\succeq
\frac1{\varrho^2}
\bx\bx^\top
\One_{\{|\bx|\leq\varrho\}}.
\label{eq:VV-lower}
\end{align}
Moreover, since the truncation multiplier belongs to $[0,1]$,
\begin{align}
\V_n(\bx)\bx^\top
&\preceq
\bx\bx^\top,
\label{eq:Vx-upper}\\
\V_n(\bx)\V_n(\bx)^\top
&\preceq
\bx\bx^\top.
\label{eq:VV-upper}
\end{align}

Using Assumption~\ref{as1},
\eqref{eq:fixed-truncated-covariance-lower}, and
\eqref{eq:Vx-lower}, we obtain
\begin{align*}
A_n
&=
\EE\left[
    f_{\epsilon\mid\bx}(0)\V_n(\bx)\bx^\top
\right]\\
&\succeq
\mu_\tau
\EE\left[
    \V_n(\bx)\bx^\top
\right]\succeq
\frac{\mu_\tau}{\varrho}
\EE\left[
    \bx\bx^\top
    \One_{\{|\bx|\leq\varrho\}}
\right]\succeq
\frac{\mu_\tau}{2\varrho}
\Lambda_{\min}(\EE[\bx\bx^\top])
{\bf I}_d.
\end{align*}
On the other hand, by \eqref{eq:Vx-upper},
\begin{align*}
A_n
&\preceq
\M_f\EE\left[
    \V_n(\bx)\bx^\top
\right]\\
&\preceq
\M_f\EE[\bx\bx^\top]\preceq
\M_f\big|\EE[\bx\bx^\top]\big|
{\bf I}_d\preceq
\M_fM_p^2{\bf I}_d.
\end{align*}
Thus, we have shown that
\begin{equation}
\lambda_*{\bf I}_d
\preceq
A_n
\preceq
\lambda^*{\bf I}_d,
\qquad n\geq1.
\end{equation}
We next consider $S_n$. Since
\[
\PP(y-\bx^\top\bbeta^*<0\mid\bx)
=
F_{\epsilon\mid\bx}(0)
=
\tau,
\]
we have
\begin{align*}
\EE\left[
    \left.
    \left(
        \tau-\One_{\{y-\bx^\top\bbeta^*<0\}}
    \right)^2
    \right|\bx
\right]=
\tau^2(1-\tau)+(1-\tau)^2\tau
=
\tau(1-\tau).
\end{align*}
Consequently, $S_n
=
\tau(1-\tau)
\EE\left[
    \V_n(\bx)\V_n(\bx)^\top
\right]$.
Using \eqref{eq:fixed-truncated-covariance-lower} and
\eqref{eq:VV-lower},
\begin{align*}
S_n
&\succeq
\frac{\tau(1-\tau)}{\varrho^2}
\EE\left[
    \bx\bx^\top
    \One_{\{|\bx|\leq\varrho\}}
\right]\succeq
\frac{\tau(1-\tau)}{2\varrho^2}
\Lambda_{\min}(\EE[\bx\bx^\top])
{\bf I}_d.
\end{align*}
Similarly, \eqref{eq:VV-upper} gives
\begin{align*}
S_n
\preceq
\tau(1-\tau)\EE[\bx\bx^\top]\preceq
\tau(1-\tau)
\big|\EE[\bx\bx^\top]\big|
{\bf I}_d.
\end{align*}
Hence, we have shown that
\begin{equation}\label{eq:Sn-uniform-all-n}
\underline{s}{\bf I}_d
\preceq
S_n
\preceq
\overline{s}{\bf I}_d,
\qquad n\geq1.
\end{equation}
\end{proof}

\subsection{Proof of Proposition \ref{prop_Gamma}}
\begin{proof}
For the first statement, by the definition of \(\V_n(\bx)\), for every \(\bx\),
\[
|\bx-\V_n(\bx)|\,|\bx|
\leq
|\bx|^2\One_{\{|\bx|>n^\vartheta\}},
\]
and
\[
\left|
\bx\bx^\top-\V_n(\bx)\V_n(\bx)^\top
\right|
\leq
|\bx|^2\One_{\{|\bx|>n^\vartheta\}}.
\]
It follows from Assumption \ref{as1} that
\begin{align*}
|A_n-A|=
\left|
\EE\left[
f_{\epsilon\mid\bx}(0)
\big(\V_n(\bx)-\bx\big)\bx^\top
\right]
\right|\leq
\EE\left[
f_{\epsilon\mid\bx}(0)
|\V_n(\bx)-\bx|\,|\bx|
\right]\leq
\M_f\EE\left[
|\bx|^2\One_{\{|\bx|>n^{\vartheta}\}}
\right].
\end{align*}
Similarly,
\begin{align*}
|S_n-S|
&=
\tau(1-\tau)
\left|
\EE\left[
\V_n(\bx)\V_n(\bx)^\top-\bx\bx^\top
\right]
\right|\\
&\leq
\tau(1-\tau)
\EE\left[
\left|
\V_n(\bx)\V_n(\bx)^\top-\bx\bx^\top
\right|
\right]\leq \tau(1-\tau)
\EE\left[
|\bx|^2\One_{\{|\bx|>n^{\vartheta}\}}
\right].
\end{align*}
Since $p>2$, on the event $\{|\bx|>n^{\vartheta}\}$, $|\bx|^2 \leq n^{\vartheta(2-p)} |\bx|^p$. Consequently,
\begin{align*}
\EE\left[
|\bx|^2\One_{\{|\bx|>n^{\vartheta}\}}
\right] \leq
n^{\vartheta(2-p)}\EE|\bx|^p \leq
M_p^p n^{\vartheta(2-p)}.
\end{align*}
So we obtain that
\begin{equation}
|A_n-A|+|S_n-S|
\lesssim
n^{\vartheta(2-p)}.
\end{equation}
For the second claim, since \(0<\gamma<1/\lambda^*\), the spectral bounds on \(A_n\) and \(A\)
give
\[
\left|({\bf I}_d-\gamma A_n)^j\right|
\leq(1-\lambda\gamma)^j,
\quad
\left|({\bf I}_d-\gamma A)^j\right|
\leq(1-\lambda\gamma)^j.
\]
For every \(j\geq1\), the difference-of-powers identity gives
\begin{align*}
&({\bf I}_d-\gamma A_n)^j
-
({\bf I}_d-\gamma A)^j=
-\gamma\sum_{\ell=0}^{j-1}
({\bf I}_d-\gamma A_n)^{j-1-\ell}
(A_n-A)
({\bf I}_d-\gamma A)^\ell.
\end{align*}
Consequently, 
\begin{equation}\label{eq:matrix_power_An_A_difference}
    \left| ({\bf I}_d-\gamma A_n)^j -({\bf I}_d-\gamma A)^j\right|\leq
j\gamma(1-\lambda\gamma)^{j-1}|A_n-A|.
\end{equation}
Recall the definition in \eqref{Gamma} and \eqref{tildeGamma}. Adding and subtracting the corresponding terms involving \(S\), we obtain
\begin{align*}
\ti{\Gamma}_n(\gamma)-\Gamma_n(\gamma)
={}&
\sum_{j=0}^{n-1}
\gamma
({\bf I}_d-\gamma A_n)^j
(S_n-S)
({\bf I}_d-\gamma A_n)^j\\
&+
\sum_{j=0}^{n-1}
\gamma
\left[
({\bf I}_d-\gamma A_n)^j
-
({\bf I}_d-\gamma A)^j
\right]
S
({\bf I}_d-\gamma A_n)^j\\
&+
\sum_{j=0}^{n-1}
\gamma
({\bf I}_d-\gamma A)^j
S
\left[
({\bf I}_d-\gamma A_n)^j
-
({\bf I}_d-\gamma A)^j
\right].
\end{align*}
It follows from \eqref{eq:matrix_power_An_A_difference} and
\(|S|\leq\overline{s}\) that
\begin{align}
\left|
\ti{\Gamma}_n(\gamma)-\Gamma_n(\gamma)
\right|
\leq
\gamma|S_n-S|
\sum_{j=0}^{n-1}(1-\lambda\gamma)^{2j}+
2\overline{s}\gamma^2|A_n-A|
\sum_{j=1}^{n-1}
j(1-\lambda\gamma)^{2j-1}.
\label{eq:Gamma_difference_initial}
\end{align}
For the first sum, we have $\gamma\sum_{j=0}^{n-1}(1-\lambda\gamma)^{2j} \lesssim n\gamma$ and
$$\gamma\sum_{j=0}^{n-1}(1-\lambda\gamma)^{2j}\leq
\frac{\gamma}
{1-(1-\lambda\gamma)^2}=
\frac{1}{2\lambda-\lambda^2\gamma}
\leq
\frac{1}{\lambda}.$$
Therefore,
\begin{equation}\label{eq:Gamma_geometric_first}
\gamma\sum_{j=0}^{n-1}(1-\lambda\gamma)^{2j}
\lesssim \min\{n\gamma,1\}.
\end{equation}
For the second sum, similarly $\gamma^2\sum_{j=1}^{n-1}
j(1-\lambda\gamma)^{2j-1} \lesssim n^2\gamma^2$ and
\begin{align*}
\gamma^2\sum_{j=1}^{n-1}
j(1-\lambda\gamma)^{2j-1} \leq
\gamma^2\sum_{j=1}^{\infty}
j(1-\lambda\gamma)^{2j-1}=\frac{\gamma^2(1-\lambda\gamma)}
{\left[1-(1-\lambda\gamma)^2\right]^2}\leq
\frac{1}{\lambda^2}.
\end{align*}
Hence,
\begin{equation}\label{eq:Gamma_geometric_second}
\gamma^2\sum_{j=1}^{n-1}
j(1-\lambda\gamma)^{2j-1}
\lesssim \min\{ n^2\gamma^2,1\}.
\end{equation}
Since $\gamma$ is bounded, combining \eqref{eq:Gamma_difference_initial},
\eqref{eq:Gamma_geometric_first}, and
\eqref{eq:Gamma_geometric_second}, we obtain
\begin{align}
\left|
\ti{\Gamma}_n(\gamma)-\Gamma_n(\gamma)
\right|
\lesssim
\min\{ n\gamma,1\}|S_n-S|+\min\{ n^2\gamma^2,1\}|A_n-A| \lesssim \min \{n\gamma ,1 \}n^{\vartheta(2-p)} .
\end{align}
\end{proof}

\subsection{Proof of Proposition \ref{truncated_V}}

\begin{proof}
 Since $0<\gamma<1/\lambda^*$, the spectral bounds on $A_n$ give
$$ {\bf I}_d-\gamma A_n
\preceq (1-\lambda_*\gamma){\bf I}_d
\preceq (1-\lambda\gamma){\bf I}_d, \quad 0< \Lambda_{\min}({\bf I}_d-\gamma A_n).$$
Consequently, $\left|({\bf I}_d-\gamma A_n)^j\right|
\leq (1-\lambda\gamma)^j.$ The same bound applies when replacing $A_n$ with $A$.
Moreover,
\begin{equation}\label{eq:geometric_contraction_An}
\sum_{j=1}^{\infty}(1-\lambda\gamma)^{2j}
=
\frac{(1-\lambda\gamma)^2}
     {1-(1-\lambda\gamma)^2}
\leq \frac{1}{\lambda\gamma}.
\end{equation}
 Define
\begin{align*}
\Theta_n
&:=
\sum_{i=1}^{n-T(n)+1}({\bf I}_d-\gamma A_n)^i\\
&=
(\gamma A_n)^{-1}({\bf I}_d-\gamma A_n)
\Big[
{\bf I}_d-({\bf I}_d-\gamma A_n)^{n-T(n)+1}
\Big],
\\
\Phi_n
&:=
\frac{n}{(n-T(n)+1)^2}A_n^{-1}
\Big[
{\bf I}_d-({\bf I}_d-\gamma A_n)^{n-T(n)+1}
\Big]\\
&\quad\times
\left\{
\sum_{j=1}^{T(n)-1}
({\bf I}_d-\gamma A_n)^j
S_n
({\bf I}_d-\gamma A_n)^j
\right\}\\
&\quad\times
\Big[
{\bf I}_d-({\bf I}_d-\gamma A_n)^{n-T(n)+1}
\Big]A_n^{-1},
\\
\Psi_n
&:=
\frac{n}{(n-T(n)+1)^2}A_n^{-1}
\left\{
\sum_{i=1}^{n-T(n)+1}
({\bf I}_d-\gamma A_n)^i
S_n
({\bf I}_d-\gamma A_n)^i
\right\}A_n^{-1}.
\end{align*}
Changing variables $j=T(n)-k$ in the first sum and $i=n-k+1$
in the second sum gives
\begin{align*}
\ti{V}_n(\gamma)
=\Phi_n
+\frac{n}{n-T(n)+1}A_n^{-1}S_nA_n^{-1}-\frac{n}{(n-T(n)+1)^2}A_n^{-1}
\big(\Theta_nS_n+S_n\Theta_n\big)A_n^{-1}
+\Psi_n.
\end{align*}
In particular,
\begin{equation}\label{eq:Vn_main_remainder}
\ti{V}_n(\gamma)
=
\frac{n}{n-T(n)+1}A_n^{-1}S_nA_n^{-1}
+\R_n,
\end{equation}
where
\begin{equation}\label{eq:covariance_remainder_definition}
\R_n
:=
\Phi_n
-\frac{n}{(n-T(n)+1)^2}A_n^{-1}
\big(\Theta_nS_n+S_n\Theta_n\big)A_n^{-1}
+\Psi_n.
\end{equation}
Due to $|
{\bf I}_d-({\bf I}_d-\gamma A_n)^{n-T(n)+1}| \leq 1$, $|A_n^{-1}|\leq 1/\lambda_*$, and \eqref{eq:geometric_contraction_An}, elementary calculations show that
\begin{align*}
|\Theta_n|
&\leq
\sum_{i=1}^{n-T(n)+1}(1-\lambda\gamma)^i
\leq\frac{1}{\lambda\gamma},
\\
|\Phi_n|
&\leq
\frac{n\overline{s}}
{\lambda\lambda_*^2(n-T(n)+1)^2\gamma},
\\
|\Psi_n|
&\leq
\frac{n\overline{s}}
{\lambda\lambda_*^2(n-T(n)+1)^2\gamma}.
\end{align*}
It follows from \eqref{eq:covariance_remainder_definition} that
\begin{equation}\label{eq:covariance_remainder_bound}
|\R_n|
\leq
\frac{4\overline{s}}{\lambda\lambda_*^2}
\frac{n}{(n-T(n)+1)^2\gamma}.
\end{equation}
Thus, \eqref{eq:Vn_main_remainder} yields 
\begin{equation}\label{eq:Vn_principal_approximation}
\left|
\ti{V}_n(\gamma)
-\frac{n}{n-T(n)+1}A_n^{-1}S_nA_n^{-1}
\right|
\leq
\frac{4\overline{s}}{\lambda\lambda_*^2}
\frac{n}{(n-T(n)+1)^2\gamma}.
\end{equation}
Under the conditions $n-T(n)\asymp n$ and $\gamma\gtrsim1/n$, it follows from \eqref{eq:Vn_main_remainder} and
\eqref{eq:covariance_remainder_bound} that
\begin{align*}
|\ti{V}_n(\gamma)|
&\leq
\frac{n}{n-T(n)+1}|A_n^{-1}|^2|S_n|
+|\R_n|\\
&\leq
\frac{n}{n-T(n)+1}
\frac{\overline{s}}{\lambda_*^2}
+
\frac{4\overline{s}}{\lambda\lambda_*^2}
\frac{n}{(n-T(n)+1)^2\gamma}\\
&=
\frac{n}{n-T(n)+1}
\frac{\overline{s}}{\lambda_*^2}
\left[
1+
\frac{4}{\lambda(n-T(n)+1)\gamma}
\right] \lesssim1.
\end{align*}
We next establish the uniform lower bound. Let 
$\bm{v} \in\RR^d$ satisfy $|\bm{v}|=1$. Reindexing the second sum in $\ti{V}_n(\gamma)$ by $i=n-k+1$
and retaining only the indices
\[
i=
\left\lceil\frac{n-T(n)+1}{2}\right\rceil,
\ldots,n-T(n)+1
\]
gives
\begin{align*}
\bm{v}^\top \ti{V}_n(\gamma)\bm{v}
\geq
\frac{n\underline{s}}{(n-T(n)+1)^2}
\sum_{i=\left\lceil(n-T(n)+1)/2\right\rceil}^{n-T(n)+1}
\left|
\Big(
{\bf I}_d-({\bf I}_d-\gamma A_n)^i
\Big)A_n^{-1}\bm{v}
\right|^2.
\end{align*}
For every retained index,
\begin{align*}
({\bf I}_d-\gamma A_n)^i
&\preceq
(1-\lambda\gamma)^i{\bf I}_d \preceq
\exp(-\lambda\gamma i){\bf I}_d\\
&\preceq
\exp\left(
-\frac{\lambda\gamma(n-T(n)+1)}{2}
\right){\bf I}_d.
\end{align*}
Therefore,
\[
{\bf I}_d-({\bf I}_d-\gamma A_n)^i
\succeq
\left[
1-\exp\left(
-\frac{\lambda\gamma(n-T(n)+1)}{2}
\right)
\right]{\bf I}_d.
\]
Since
\[
|A_n^{-1}\bm{v}|
\geq
\frac{1}{\big|A_n|}
\geq
\frac{1}{\lambda^*},
\]
we obtain
\[
\left|
\Big(
{\bf I}_d-({\bf I}_d-\gamma A_n)^i
\Big)A_n^{-1}\bm{v}
\right|^2
\geq
\frac{1}{(\lambda^*)^2}
\left[
1-\exp\left(
-\frac{\lambda\gamma(n-T(n)+1)}{2}
\right)
\right]^2.
\]
The retained sum contains at least $(n-T(n)+1)/2$ terms. Hence,
\begin{align*}
\bm{v}^\top \ti{V}_n(\gamma)\bm{v}
\geq
\frac{n}{n-T(n)+1} \cdot
\frac{\underline{s}}{2(\lambda^*)^2}
\left[
1-\exp\left(
-\frac{\lambda\gamma(n-T(n)+1)}{2}
\right)
\right]^2 \gtrsim1.
\end{align*}
Here we use the fact that
\[
1-\exp\left(
-\frac{\lambda\gamma(n-T(n)+1)}{2}
\right)
\gtrsim1.
\]
Since the lower bound holds uniformly over all $\bm{v}\in\RR^d$ with
$|\bm{v}|=1$, it follows that $\Lambda_{\min}(\ti{V}_n(\gamma))
\gtrsim1$.
Therefore, the smallest and largest eigenvalues of $\ti{V}_n(\gamma)$
are uniformly bounded away from zero and infinity. The same proof for $V_n(\gamma)$ applies verbatim with $A_n$ and $S_n$ replaced by $A$ and $S$, since $A$ and $S$ satisfy the same eigenvalue bounds as $A_n$ and $S_n$.

For the second statement, define a matrix function
\[
\H_{n,k}(A)
=
\begin{cases}
\displaystyle
A^{-1}
\left[
({\bf I}_d-\gamma A)^{T(n)-k}
-
({\bf I}_d-\gamma A)^{n-k+1}
\right],
&1\leq k<T(n),\\[1.2em]
\displaystyle
A^{-1}
\left[
{\bf I}_d-({\bf I}_d-\gamma A)^{n-k+1}
\right],
&T(n)\leq k\leq n,
\end{cases}
\]
which can be rewritten as
\[
\H_{n,k}(A)
=
\begin{cases}
\displaystyle
\gamma\sum_{j=T(n)-k}^{n-k}
({\bf I}_d-\gamma A)^j,
&1\leq k<T(n),\\[1.2em]
\displaystyle
\gamma\sum_{j=0}^{n-k}
({\bf I}_d-\gamma A)^j,
&T(n)\leq k\leq n.
\end{cases}
\]
Then 
\[
V_n(\gamma)
=
\frac{n}{(n-T(n)+1)^2}
\sum_{k=1}^n
\H_{n,k}(A)S\H_{n,k}(A),
\]
and 
\[
\ti{V}_n(\gamma)
=
\frac{n}{(n-T(n)+1)^2}
\sum_{k=1}^n
\H_{n,k}(A_n)S_n\H_{n,k}(A_n).
\]
Elementary calculations show that, uniformly in \(n\), \(k\), and \(\gamma\),
\[
|\H_{n,k}(A_n)|
\leq
\gamma\sum_{j=0}^{\infty}(1-\lambda\gamma)^j
=\frac1\lambda,
\]
and similarly, $|\H_{n,k}(A)|\leq 1/\lambda$. Moreover, the power series identity gives, for every \(j\geq1\),
\[
\left|
({\bf I}_d-\gamma A_n)^j
-
({\bf I}_d-\gamma A)^j
\right|
\leq
j\gamma(1-\lambda\gamma)^{j-1}|A_n-A|.
\]
Therefore,
\begin{align*}
|\H_{n,k}(A_n)-\H_{n,k}(A)|
\leq
\gamma^2|A_n-A|
\sum_{j=1}^{\infty}
j(1-\lambda\gamma)^{j-1}=
\frac{|A_n-A|}{\lambda^2}.
\end{align*}

Adding and subtracting the corresponding terms, we have
\begin{align*}
&\H_{n,k}(A_n)S_n\H_{n,k}(A_n)
-\H_{n,k}(A)S\H_{n,k}(A)\\
={}&
\big(
\H_{n,k}(A_n)-\H_{n,k}(A)
\big)
S_n\H_{n,k}(A_n)+
\H_{n,k}(A)(S_n-S)\H_{n,k}(A_n)\\
&+
\H_{n,k}(A)S
\big(
\H_{n,k}(A_n)-\H_{n,k}(A)
\big).
\end{align*}
Thus,
\[
\left|
\H_{n,k}(A_n)S_n\H_{n,k}(A_n)
-\H_{n,k}(A)S\H_{n,k}(A)
\right|
\leq
\frac{2\overline{s}}{\lambda^3}|A_n-A|
+
\frac{1}{\lambda^2}|S_n-S|.
\]
By Proposition \ref{prop_Gamma},
\begin{align*}
\left|
\ti{V}_n(\gamma)-V_n(\gamma)
\right|
&\leq
\frac{n^2}{(n-T(n)+1)^2}
\left(
\frac{2\overline{s}}{\lambda^3}|A_n-A|
+
\frac{1}{\lambda^2}|S_n-S|
\right)\\
&\lesssim
|A_n-A|+|S_n-S|\lesssim
n^{\vartheta(2-p)}
\end{align*}
uniformly over all \(n\geq1\) and
\(0<\gamma<1/\lambda^*\).

Finally, we prove the last statement. Under the regime
\[
\frac{n}{n-T(n)+1}\longrightarrow w\ge1,
\]
the sequence $n/(n-T(n)+1)$ is uniformly bounded. Therefore, \eqref{eq:covariance_remainder_bound} gives
\begin{align}
|\R_n|
&\leq
\frac{4\overline{s}}{\lambda\lambda_*^2}
\frac{n}{(n-T(n)+1)^2\gamma}\nonumber\\
&=
\frac{4\overline{s}}{\lambda\lambda_*^2}
\frac{n}{n-T(n)+1}
\frac{1}{(n-T(n)+1)\gamma}\lesssim
\frac{1}{(n-T(n)+1)\gamma}.
\label{eq:covariance_remainder_convergence}
\end{align}
Now we redefine $\Theta_n$, $\Phi_n$, $\Psi_n$, and $\R_n$ by replacing $A_n,S_n$ with $A,S$. It is  clear that
$$V_n(\gamma)
=
\frac{n}{n-T(n)+1}A^{-1}SA^{-1}
+\R_n = \frac{n}{n-T(n)+1}V+\R_n,$$
and \eqref{eq:covariance_remainder_convergence} still holds for the redefined $\R_n$, since the eigenvalues' bound for $A_n$ and $S_n$ identically applies to $A$ and $S$. Consequently,
\begin{align}\label{Vdecompose}
V_n(\gamma)-V
=\left(
\frac{n}{n-T(n)+1}-w
\right)V
+\R_n.
\end{align}
Combining \eqref{eq:covariance_remainder_convergence} and \eqref{Vdecompose}, we obtain 
\[\left|V_n(\gamma)-V
    \right|
    \lesssim \frac{1}{(n-T(n)+1)\gamma} +\left|
    \frac{n}{(n-T(n)+1)}-w
    \right|.
\]
When $n\gamma \rightarrow \infty$, both terms in the upper bound go to $0$. Hence, $V_n(\gamma)\rightarrow V$.
\end{proof}

\section{Algorithm and Additional Simulation Results}\label{sec_supp_sim}
We first provide the pseudo code for our proposed online inference algorithm below.

\begin{algorithm}[ht]
\caption{Efficient Online Inference for Large-Scale Quantile Regression}
\label{alg:parallel_qr}

\KwInput{A sequence of i.i.d.\ observations
$\{(\bm{x}_{\ell},y_{\ell})\}_{\ell\geq1}$; $K\geq2$; unit vector $\bm v\in\RR^d$; $\alpha\in(0,1)$.}

\KwResult{$100\%(1-\alpha)$ confidence interval
$\widehat{\mathrm{CI}}_{\bm v,\alpha}$ for $\bm v^{\top}\bbeta^*$.}

\Begin{
    \For{$i=1,2,\ldots$}{
        \For{$k=1,\ldots,K$ \textnormal{ in parallel}}{
            Set $\ell\leftarrow k+K(i-1)$ as the index of the
            $i$-th observation processed by the $k$-th chain\;

            Update $\bbeta_{i,(k)}$ according to
            \eqref{onlineqr} using $(\bm{x}_{\ell},y_{\ell})$\;

            Update $\bbeta^{\mathrm{SF}}_{i,(k)}$
            according to \eqref{suffix}\;
        }

        \When{\textnormal{an output is requested}}{
            Compute
            \begin{align*}
                \bar{\bbeta}_{i,K}
            \leftarrow
            \frac{1}{K}\sum_{k=1}^K
            \bbeta^{\mathrm{SF}}_{i,(k)},\quad 
            \hat{\sigma}^2_{\bm v}
            \leftarrow
            \frac{1}{K-1}\sum_{k=1}^K
            \left(
            \bm v^{\top}\bbeta^{\mathrm{SF}}_{i,(k)}
            -\bm v^{\top}\bar{\bbeta}_{i,K}
            \right)^2.
            \end{align*}

            Construct
            \[
            \widehat{\mathrm{CI}}_{\bm v,\alpha}
            \leftarrow
            \left[
            \bm v^{\top}\bar{\bbeta}_{i,K}
            \;\pm\;
            t_{1-\alpha/2,K-1}
            \frac{\hat{\sigma}_{\bm v}}{\sqrt K}
            \right].
            \]
        }
    }
}
\end{algorithm}

We also report additional simulation results in Table~\ref{tb_inf_a005_d500}-\ref{tb_inf_a001_d1500}. Here we define the relative error as 
\begin{equation}\label{ratiodiff}
    \mathcal{D}_{\alpha} = \left| \frac{\text{MisCov}}{\alpha}-1\right|,
\end{equation}
where $\text{MisCov}$ equals one minus the empirical coverage rate. Table \ref{tab:inference_speed} records the time of different procedures as shown in Figure \ref{fig:estimation_speed}.

\begin{table}[htbp!]
    \vskip -0.1in
    \begin{center}
        \begin{small}
                \begin{tabular}{lccc}
                    \midrule
                    $\alpha = 0.05,\,d = 500$ & & {\textbf{Coverage (Relative error)}} & {\textbf{Length (std)}} \\
                    \midrule
                    \multirow{2}{*}{\textbf{$k = 7.5\times 10^{5}$}} & Parallel Inference & 0.9490 (0.020) & 0.011 (0.00333) \\
                     & Random Scaling & 0.9850 (0.700) & 0.023 (0.01515) \\
                    \midrule
                    \multirow{2}{*}{\textbf{$k = 1.5\times 10^{6}$}} & Parallel Inference & 0.9470 (0.060) & 0.006 (0.00208) \\
                     & Random Scaling & 0.9840 (0.680) & 0.012 (0.00766) \\
                    \midrule
                    \multirow{2}{*}{\textbf{$k = 3\times 10^{6}$}} & Parallel Inference & 0.9470 (0.060) & 0.004 (0.00137) \\
                     & Random Scaling & 0.9740 (0.480) & 0.007 (0.00388) \\
                    \midrule
                \end{tabular}
        \end{small}
    \end{center}
    \caption{Comparison of Parallel Inference and Random Scaling for $\alpha=0.05$, $d=500$, and $n = 3 \times 10^6$. The value of $k$ represents the $k$-th iteration or $k$-th sample.  Coverage relative error is defined by $\mathcal{D}_{\alpha}$ in \eqref{ratiodiff}. CI lengths are reported as mean with standard deviation in parentheses.}
    \label{tb_inf_a005_d500}
\end{table}

\begin{table}[htbp!]
    \vskip -0.1in
    \begin{center}
        \begin{small}
                \begin{tabular}{lccc}
                    \midrule
                    $\alpha = 0.01,\,d = 500$ & & {\textbf{Coverage (Relative error)}} & {\textbf{Length (std)}} \\
                    \midrule
                    \multirow{2}{*}{\textbf{$k = 7.5\times 10^{5}$}} & Parallel Inference & 0.9894 (0.060) & 0.017 (0.00535) \\
                     & Random Scaling & 0.9975 (0.750) & 0.036 (0.02287) \\
                    \midrule
                    \multirow{2}{*}{\textbf{$k = 1.5\times 10^{6}$}} & Parallel Inference & 0.9885 (0.150) & 0.010 (0.00333) \\
                     & Random Scaling & 0.9968 (0.680) & 0.019 (0.01145) \\
                    \midrule
                    \multirow{2}{*}{\textbf{$k = 3\times 10^{6}$}} & Parallel Inference & 0.9904 (0.040) & 0.007 (0.00222) \\
                     & Random Scaling & 0.9955 (0.550) & 0.010 (0.00574) \\
                    \midrule
                \end{tabular}
        \end{small}
    \end{center}
    \caption{Comparison of Parallel Inference and Random Scaling for $\alpha=0.01$, $d=500$, and $n = 3 \times 10^6$. The value of $k$ represents the $k$-th iteration or $k$-th sample. Coverage relative error is defined by $\mathcal{D}_{\alpha}$ in \eqref{ratiodiff}. CI lengths are reported as mean with standard deviation in parentheses.}
    \label{tb_inf_a001_d500}
\end{table}

\begin{table}[htbp!]
    \vskip -0.1in
    \begin{center}
        \begin{small}
                \begin{tabular}{lccc}
                    \midrule
                    $\alpha = 0.05,\,d = 1500$ & & {\textbf{Coverage (Relative error)}} & {\textbf{Length (std)}} \\
                    \midrule
                    \multirow{2}{*}{\textbf{$k = 2.5\times 10^{6}$}} & Parallel Inference & 0.9420 (0.160) & 0.006 (0.00209) \\
                     & Random Scaling & 0.9890 (0.780) & 0.021 (0.01486) \\
                    \midrule
                    \multirow{2}{*}{\textbf{$k = 5\times 10^{6}$}} & Parallel Inference & 0.9350 (0.300) & 0.004 (0.00126) \\
                     & Random Scaling & 0.9880 (0.760) & 0.011 (0.00736) \\
                    \midrule
                    \multirow{2}{*}{\textbf{$k = 10^{7}$}} & Parallel Inference & 0.9510 (0.020) & 0.003 (0.00086) \\
                     & Random Scaling & 0.9850 (0.700) & 0.006 (0.00369) \\
                    \midrule
                \end{tabular}
        \end{small}
    \end{center}
    \caption{Comparison of Parallel Inference and Random Scaling for $\alpha=0.05$, $d=1500$, and $n = 10^7$. The value of $k$ represents the $k$-th iteration or $k$-th sample. Coverage relative error is defined by $\mathcal{D}_{\alpha}$ in \eqref{ratiodiff}. CI lengths are reported as mean with standard deviation in parentheses.}
    \label{tb_inf_a005_d1500}
\end{table}

\begin{table}[htbp!]
    \vskip -0.1in
    \begin{center}
        \begin{small}
                \begin{tabular}{lccc}
                    \midrule
                    $\alpha = 0.01,\,d = 1500$ & & {\textbf{Coverage (Relative error)}} & {\textbf{Length (std)}} \\
                    \midrule
                    \multirow{2}{*}{\textbf{$k = 2.5\times 10^{6}$}} & Parallel Inference & 0.9884 (0.160) & 0.010 (0.00319) \\
                     & Random Scaling & 0.9988 (0.880) & 0.033 (0.02242) \\
                    \midrule
                    \multirow{2}{*}{\textbf{$k = 5\times 10^{6}$}} & Parallel Inference & 0.9886 (0.140) & 0.006 (0.00196) \\
                     & Random Scaling & 0.9991 (0.910) & 0.017 (0.01118) \\
                    \midrule
                    \multirow{2}{*}{\textbf{$k = 10^{7}$}} & Parallel Inference & 0.9902 (0.020) & 0.004 (0.00131) \\
                     & Random Scaling & 0.9971 (0.710) & 0.009 (0.00554) \\
                    \midrule
                \end{tabular}
        \end{small}
    \end{center}
    \caption{Comparison of Parallel Inference and Random Scaling for $\alpha=0.01$, $d=1500$, and $n = 10^7$. The value of $k$ represents the $k$-th iteration or $k$-th sample. Coverage relative error is defined by $\mathcal{D}_{\alpha}$ in \eqref{ratiodiff}. CI lengths are reported as mean with standard deviation in parentheses.}
    \label{tb_inf_a001_d1500}
\end{table}

\begin{table}[htbp!]
	\centering
	\begin{small}
		\begin{tabular}{lcccc}
			\toprule
			\textbf{Method}
			& $\boldsymbol{n=10^4}$
			& $\boldsymbol{n=3\times10^4}$
			& $\boldsymbol{n=10^5}$
			& $\boldsymbol{n=5\times10^5}$ \\
			\midrule
			Parallel Inference
			& 0.0009 & 0.0022 & 0.0071 & 0.0339 \\
			Conquer Bootstrap
			& 0.0474 & 0.1116 & 0.3303 & 1.4667 \\
			Standard Quantile Regression
			& 0.3874 & 0.7623 & 1.5802 & 4.6028 \\
			\bottomrule
		\end{tabular}
	\end{small}
	\caption{Average inference time in seconds for different sample sizes ($d = 20$).}
	\label{tab:inference_speed}
\end{table}
As described in Section~\ref{sec_data}, we evaluate the overall lack-of-fit for the real data study. In each of 200 replications, we draw a quantile level uniformly from $(0,1)$, estimate the corresponding quantile-regression coefficients using the same suffix-averaged SSGD procedure, randomly sample a covariate vector from the observed data, and generate a synthetic log hourly wage from the fitted conditional quantile. Figure \ref{fig:realdata_qqplot} shows the QQ plot of the simulated outcomes versus the observed ones, which is close to the 45-degree line and indicates a reasonable fit. Details of this procedure are also introduced in \textcite{wang_quantile_2012}.
\begin{figure}[htbp!]
    \centering
    \includegraphics[width=0.7\linewidth]{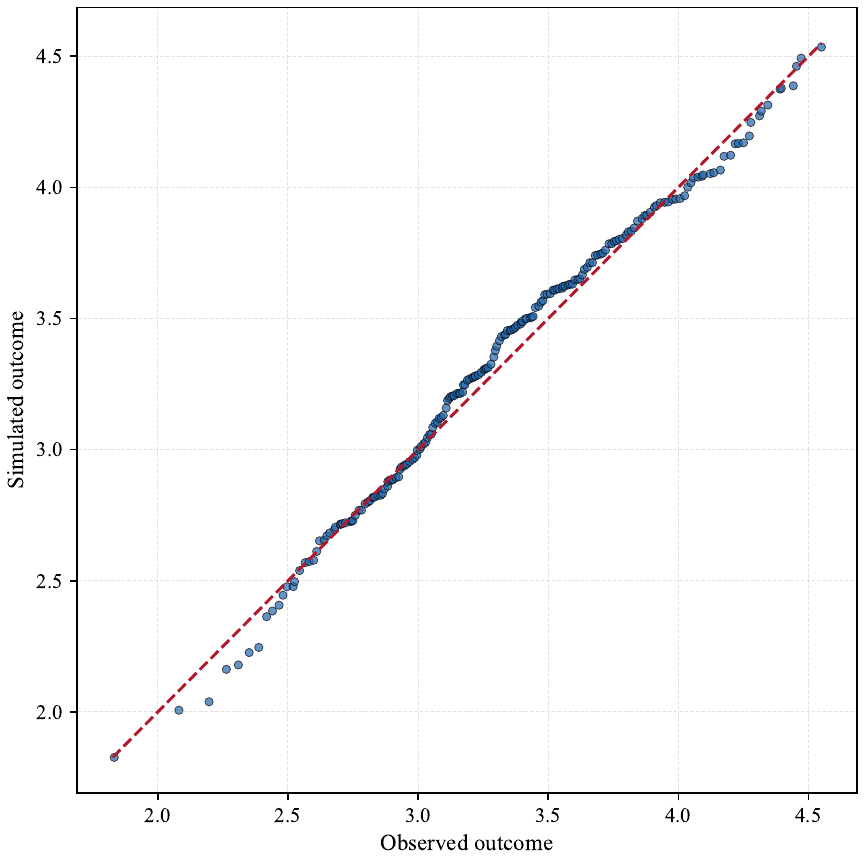}
    \caption{Lack-of-fit diagnosis QQ plot.}
    \label{fig:realdata_qqplot}
\end{figure}

\end{appendix}

\end{document}